\documentclass{article}

\PassOptionsToPackage{square,numbers}{natbib}

\usepackage[final]{neurips_2020}

\usepackage[utf8]{inputenc} 
\usepackage[T1]{fontenc}    
\usepackage{hyperref}       
\usepackage{url}            
\usepackage{booktabs}       
\usepackage{amsfonts}       
\usepackage{nicefrac}       
\usepackage{microtype}      
\usepackage{amsthm}
\usepackage{wrapfig}
\usepackage{multirow}
\usepackage{adjustbox}
\usepackage{bibunits}

\newcommand{\bigb}{\textsc{BigBird}\xspace}

\usepackage{amsmath,amsfonts,amsthm}
\usepackage{thmtools,thm-restate} 
\usepackage{xcolor}
\usepackage{times}
\usepackage{amssymb}
\usepackage{bm}
\usepackage{xspace}
\usepackage{mathtools}
\usepackage{wrapfig}

\usepackage{nicefrac}

\usepackage{cleveref}
\usepackage{caption}
\usepackage{graphicx}
\usepackage{subcaption}
\usepackage{enumitem}
\usepackage[ruled,vlined]{algorithm2e}
\usepackage{stmaryrd}

\usepackage{mathtools, nccmath}

\Crefname{equation}{Eq.}{Eqs.}
\Crefname{figure}{Fig.}{Fig.}
\Crefname{tabular}{Tab.}{Tabs.}
\Crefname{table}{Tab.}{Tabs.}
\Crefname{section}{Sec.}{Sec.}
\Crefname{appendix}{App.}{App.}

\newtheorem{theorem}{Theorem}
\newtheorem{lemma}{Lemma}
\newtheorem{proposition}{Proposition}

\newtheorem{definition}{Definition}

\newtheorem{task}{Task}
\newtheorem{conjecture}{Conjecture}

\DeclareMathOperator*{\argmax}{arg\,max}
\DeclareMathOperator*{\argmin}{arg\,min}

\def\N{\mathbb{N}}
\def\Q{\mathbb{Q}}
\def\R{\mathbb{R}}

\def\mK{{\bm{K}}}

\def\mV{{\bm{V}}}
\def\mX{{\bm{X}}}
\def\mY{{\bm{Y}}}
\def\mZ{{\bm{Z}}}

\def\vzero{{\bm{0}}}

\def\va{{\bm{a}}}

\def\vp{{\bm{p}}}
\def\vq{{\bm{q}}}

\def\vs{{\bm{s}}}

\def\vv{{\bm{v}}}

\def\vx{{\bm{x}}}
\def\vy{{\bm{y}}}
\def\vz{{\bm{z}}}

\newcommand{\lu}{\underline{\phantom{A}}}

\newcommand{\Gd}{\ensuremath{\mathbb{G}_{\delta}}}

\newcommand{\relu}{\operatorname{ReLU}}

\newcommand{\Attn}{\ensuremath{\textsc{Attn}\xspace}}
\newcommand{\CrossAttn}{\ensuremath{\textsc{CrossAttn}\xspace}}
\newcommand{\penc}{\operatorname{pos}}
\newcommand{\Enc}{\operatorname{Enc}}

\newcommand{\Dec}{\operatorname{Dec}}

\newcommand{\sz}{0,\ldots,0}

\newcommand{\vzs}{\vzero_s}
\newcommand{\oh}[1]{\llbracket\ #1\ \rrbracket}

\DeclarePairedDelimiter{\norm}{\|}{\|}

\DeclarePairedDelimiter{\set}{ \{ }{ \} }
\newcommand{\bkt}[2]{\left \langle #1, #2 \right \rangle}

\definecolor{violet}{rgb}{0.50, 0.16, 0.88}

\title{Big Bird: Transformers for Longer Sequences}

\author{%
  Manzil Zaheer, $\quad$ $\quad$ $\quad$ Guru Guruganesh, $\quad$ $\quad$ $\quad$Avinava Dubey, \\
  {\bf  Joshua Ainslie, $\quad$ Chris Alberti, $\quad$ Santiago Ontanon, $\quad$ Philip Pham,}  \\
  {\bf Anirudh Ravula, $\quad$ Qifan Wang, $\quad$ Li Yang, $\quad$ Amr Ahmed } \\
  Google Research\\
  \texttt{\{manzilz, gurug, avinavadubey\}@google.com}
}

\begin{document}

\maketitle
\begin{abstract}

Transformers-based models, such as BERT, have been one of the most successful deep learning models for NLP. Unfortunately, one of their core limitations is the quadratic dependency (mainly in terms of memory) on the sequence length due to their full attention mechanism.  To remedy this, we propose, \bigb, a sparse attention mechanism that reduces this quadratic dependency to linear.  We show that \bigb is a universal approximator of sequence functions and is Turing complete, thereby preserving these  properties of the quadratic, full attention model. Along the way, our theoretical analysis reveals  some of the benefits of having $O(1)$ global tokens (such as CLS), that attend to the entire sequence  as part of the sparse attention mechanism. The proposed sparse attention can handle sequences of length up to 8x of  what was previously possible using  similar hardware.  As a consequence of the capability to handle longer context, \bigb  drastically improves performance  on various NLP tasks  such as question answering and summarization. We also propose novel applications to genomics data.
\end{abstract}

\section{Introduction}
Models based on Transformers \citep{vaswani2017attention}, such as BERT \citep{devlin2018bert,liu2019roberta}, 
are wildly successful for a wide variety of Natural Language Processing (NLP) tasks and consequently are mainstay of modern NLP research. 
Their versatility and robustness are the primary drivers behind the wide-scale adoption of Transformers. 
The model is easily adapted for a diverse range of sequence based tasks -- as a seq2seq model for translation~\citep{vaswani2017attention}, summarization~\citep{miller2019leveraging}, generation~\citep{chen2019distilling}, etc.~or as a standalone encoders for sentiment analysis~\citep{sun2019utilizing}, POS tagging~\citep{martin2019camembert}, machine reading comprehension~\citep{wang2019multi}, etc. -- and it is known to vastly outperform previous 
sequence models like LSTM~\citep{hochreiter1997long}.
The key innovation in Transformers is the introduction of a self-attention mechanism, which can be evaluated in parallel for each token of the input sequence, eliminating the sequential dependency in recurrent neural networks, like LSTM.
This parallelism enables Transformers to leverage the full power of modern SIMD hardware accelerators like GPUs/TPUs, thereby facilitating training of NLP models on datasets of unprecedented size.
This ability to train on large scale data has led to surfacing of models like BERT~\citep{devlin2018bert} and T5~\citep{raffel2019exploring}, which pretrain transformers on large general purpose corpora and transfer the knowledge to down-stream task. The pretraining has led to significant improvement in low data regime downstream tasks~\citep{kumar2020data} as well as tasks with sufficient data~\citep{yang2019xlnet} and thus have been a major force behind the ubiquity of transformers in contemporary NLP.

The self-attention mechanism  overcomes  constraints of RNNs (namely the sequential nature of RNN)
by allowing each token in the input sequence to attend independently to every other token 
in the sequence.    This design choice has several interesting repercussions.
In particular, the full self-attention have computational and memory requirement that 
is quadratic in the sequence length. We note that while the corpus can be large, the sequence length, which provides the context in many applications is very limited. 
Using commonly available current hardware and model sizes, this 
requirement translates to roughly being able to handle input sequences of length 512 tokens. 
This reduces its direct applicability to tasks that require larger context, 
like QA~\citep{lin2003makes}, document classification, etc.

However, while we know that self-attention and Transformers are useful, our theoretical
understanding is rudimentary. What aspects of the self-attention model are necessary for its performance?
What can we say about the expressivity of Transformers and similar models? Apriori, it was not 
even clear from the  design if the proposed self-attention mechanism was as effective as RNNs. 
For example, the self-attention does not even obey sequence order as it is permutation equivariant. 
This concern has been partially resolved, as \citet{Yun19} showed that transformers are expressive 
enough to capture all continuous sequence to sequence functions with a compact domain. 
Meanwhile, \citet{Perez19} showed that the full  transformer is Turing Complete 
(i.e.~can simulate a full Turing machine). Two natural questions arise: Can we achieve the empirical 
benefits of a fully quadratic self-attention scheme using fewer inner-products? Do these 
sparse attention mechanisms preserve the expressivity and flexibility of the original network?

In this paper, we address both the above questions and produce a sparse attention mechanism  that
improves performance on a multitude of tasks that require long contexts.
We systematically develop \bigb, an attention mechanism whose complexity 
is linear in the number of tokens (\Cref{sec:arch}). We take inspiration from graph 
sparsification methods and understand  where the proof for expressiveness of Transformers 
breaks down when full-attention is relaxed to form the proposed attention pattern. This understanding
helped us develop \bigb, which is theoretically as expressive and also empirically useful. 
In particular, our \bigb consists of three main part:

\begin{itemize}[leftmargin=6mm, itemsep=0mm, partopsep=0pt,parsep=0pt]
    \item A set of $g$ global tokens attending on all parts of the sequence.
    \item All tokens attending to a set of $w$ local neighboring tokens.
    \item All tokens attending to a set of $r$ random tokens.
\end{itemize}
This leads to a high performing attention mechanism scaling to much longer sequence lengths (8x).

To summarize, our main \textbf{contributions} are:
\begin{enumerate}[leftmargin=6mm, itemsep=2mm, partopsep=0pt,parsep=0pt]
    \item 
    \bigb satisfies all the known theoretical properties of full transformer (\Cref{sec:theory}). In particular, we show that adding extra tokens allows one to express all continuous sequence to sequence functions with only $O(n)$-inner products. Furthermore, we show that under standard assumptions regarding precision, \bigb is Turing complete.
    
    \item 
    Empirically, we show that the extended context modelled by \bigb benefits variety of NLP tasks. 
    We achieve \emph{state of the art} results for  question answering and document summarization on a number of different datasets. 
    Summary of these results are presented in \Cref{sec:expt-nlp}.
    
    \item
    Lastly, we introduce a novel application of attention based models where long contexts 
    are beneficial: extracting contextual representations of genomics sequences like DNA. 
    With longer masked LM pretraining, \bigb improves performance on downstream tasks such as promoter-region and chromatin profile prediction (\Cref{sec:expt-bio}). 
\end{enumerate}

\subsection{Related Work}
There have been a number of interesting attempts, that were aimed at alleviating the quadratic dependency of Transformers, which can broadly categorized into two directions.
First line of work embraces the length limitation and develops method around it.
Simplest methods in this category just employ sliding window~\citep{wang2019multi}, but in general most work fits in the following general paradigm:
using some other mechanism select a smaller subset of relevant contexts to feed in the transformer and optionally iterate, i.e. call transformer block multiple time with different contexts each time. 
Most prominently, SpanBERT~\citep{joshi2020spanbert}, ORQA~\citep{lee2019latent}, REALM~\citep{guu2020realm}, RAG~\citep{lewis2020retrieval} have achieved strong performance for different tasks. However, it is worth noting that these methods often require significant engineering efforts (like back prop through large scale nearest neighbor search) and are hard to train.

Second line of work questions if full attention is essential and have tried to come up with approaches that do not require full attention, thereby reducing the memory and computation requirements.
Prominently,~\citet{dai2019transformer, sukhbaatar2019adaptive, rae2019compressive} have proposed auto-regresive models that work well for left-to-right language modeling but suffer in tasks which require bidirectional context.
\citet{child2019generating} proposed a sparse model that reduces the complexity to $O(n\sqrt{n})$, \citet{kitaev2019reformer} further reduced the complexity to $O(n\log(n))$ by using LSH to compute  nearest neighbors. 
\citet{ye2019bp} proposed binary partitions of the data where as \citet{qiu2019blockwise} reduced complexity by using block sparsity. 
Recently, Longformer \cite{beltagy2020longformer} introduced a localized sliding 
window based mask with few global mask to reduce computation and extended BERT to longer sequence based tasks. 
Finally, our work is closely related to and built on the work of Extended Transformers Construction~\citep{ainslie2020etc}. 
This work was designed to encode structure in text for transformers. The idea of global tokens was used extensively by them to achieve their goals. Our theoretical work can be seen as providing a  justification for the success of these models as well.
It is important to note that most of the aforementioned methods are heuristic based and empirically are not as versatile and robust as the original transformer, i.e. the same architecture do not attain SoTA on multiple standard benchmarks. (There is one exception of Longformer which we include in all our comparisons, see~\Cref{sec:app-related-work} for a more detailed comparison). 
Moreover, these approximations do not come with theoretical guarantees.

\begin{figure}
    \vspace{-3mm}
    \centering
    \begin{subfigure}{.22\textwidth}
        \includegraphics[width=\linewidth]{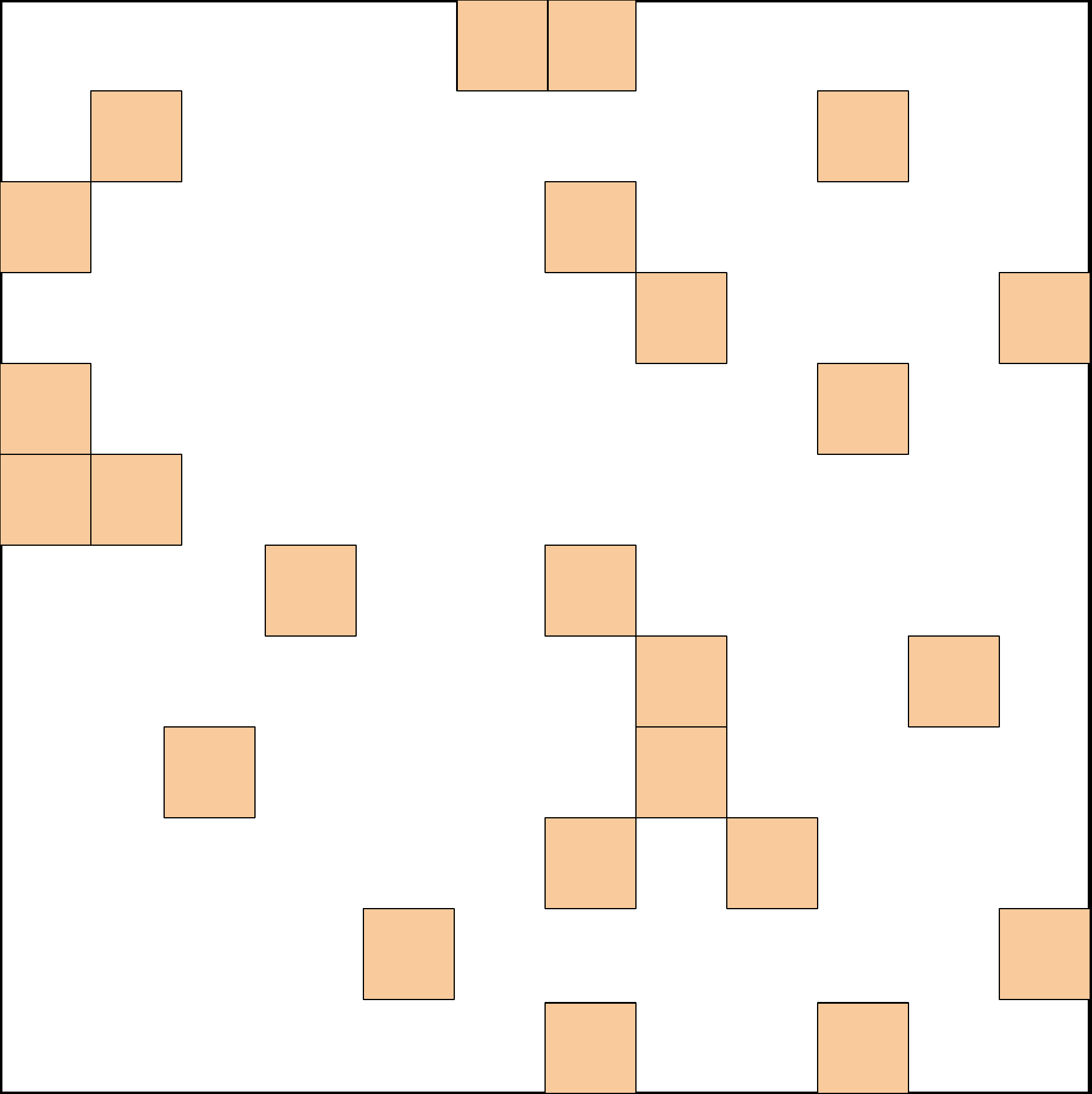}
        \caption{Random attention}
        \label{fig:rnd_atn}
    \end{subfigure}\hfill
    \begin{subfigure}{.22\textwidth}
        \includegraphics[width=\linewidth]{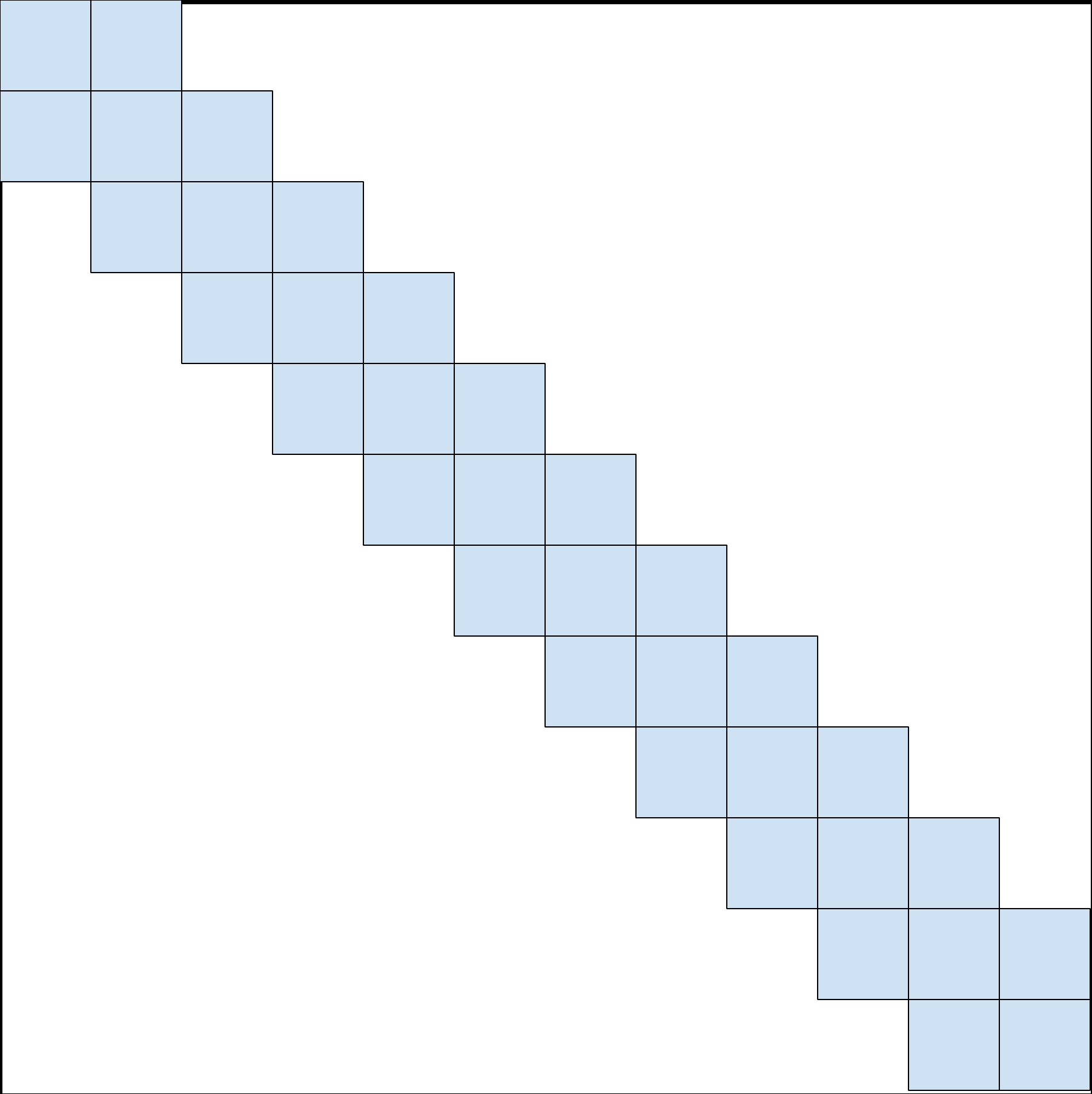}
        \caption{Window attention}
        \label{fig:wnd:atn}
    \end{subfigure}\hfill
    \begin{subfigure}{.22\textwidth}
        \includegraphics[width=\linewidth]{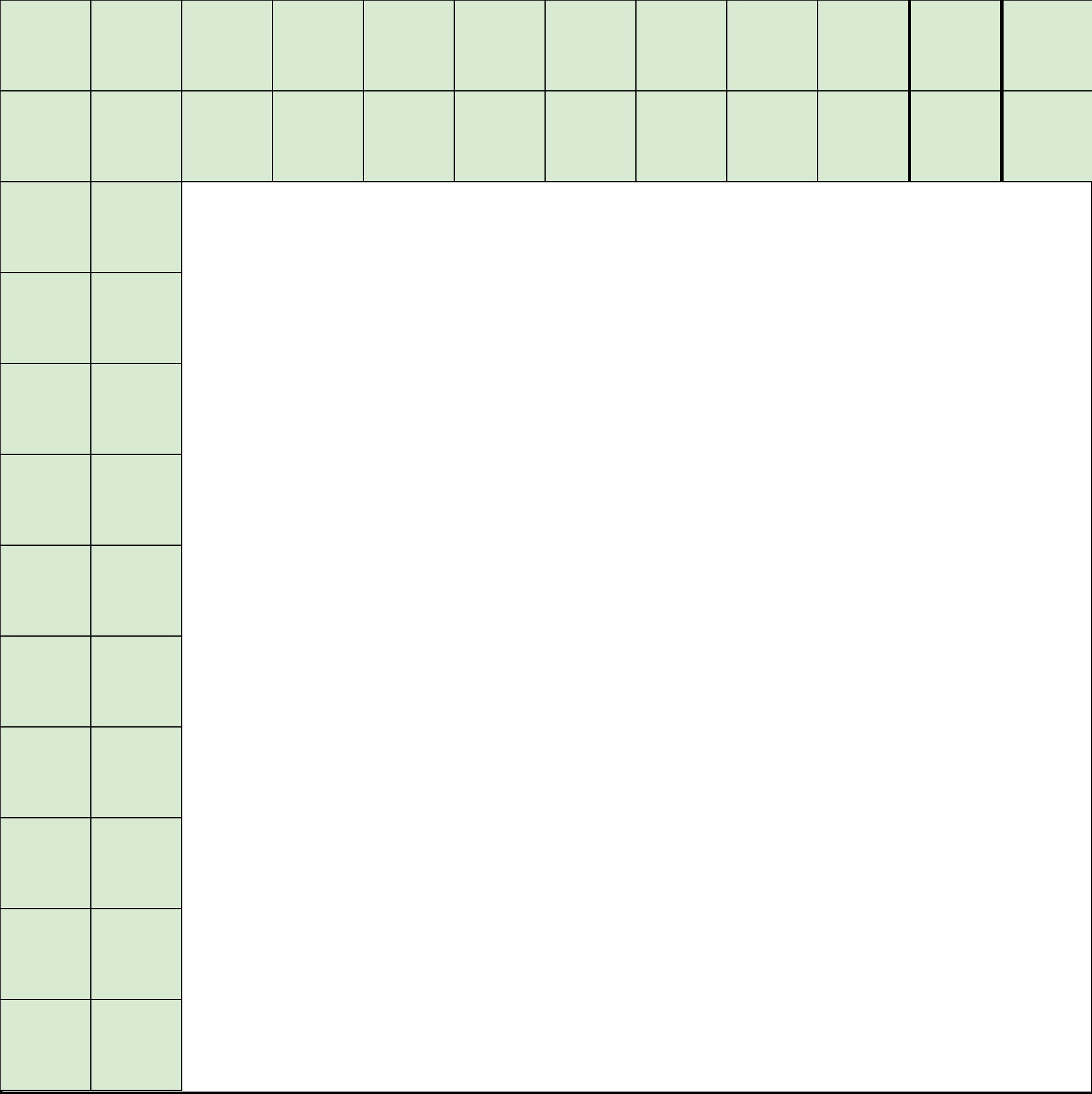}
        \caption{Global Attention}
        \label{fig:gbl_atn}
    \end{subfigure}\hfill
    \begin{subfigure}{.22\textwidth}
        \includegraphics[width=\linewidth]{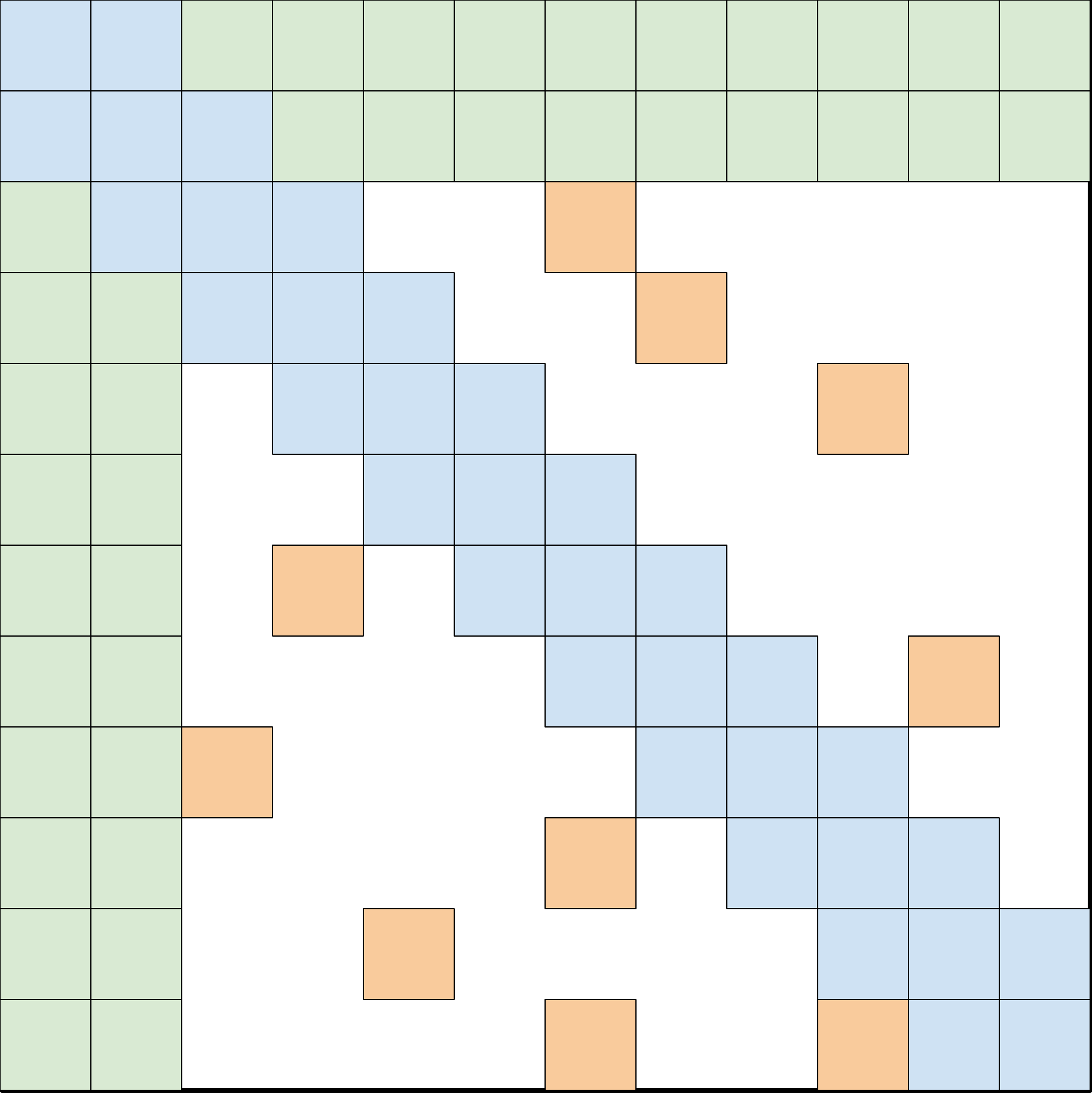}
        \caption{\bigb}
        \label{fig:bigb_atn}
    \end{subfigure}
    \hfill
    \caption{Building blocks of the attention mechanism used in \bigb. White color indicates absence of attention. (a) random attention with $r=2$, (b) sliding window attention with $w=3$ (c) global attention with $g=2$. (d) the combined \bigb model.}
    \label{fig:my_label}
\end{figure}

\section{\bigb Architecture}
\label{sec:arch}
In this section, we describe the \bigb model using the \emph{generalised attention mechanism} that is used in each layer of transformer operating on an input sequence $\mX = (\vx_1, ..., \vx_n) \in \R^{n\times d}$. 
The \emph{generalized attention mechanism} is described by a directed graph $D$ whose vertex set is $[n] = \set{1,\dots,n}$. 
The set of arcs (directed edges) represent the set of inner products that the attention mechanism will consider. 
Let $N(i)$ denote the out-neighbors set of node $i$ in $D$, then the $i^\text{th}$ output vector of the generalized attention mechanism is defined as \useshortskip
\vspace{-1mm}
\begin{equation}\vspace{-2mm}
\small
\Attn_D(\mX)_i = \vx_i +  \sum_{h=1}^H \sigma \left(Q_h(\vx_i) K_h(\mX_{N(i)})^T \right) \cdot V_h(\mX_{N(i)}) 
\label{AT} \tag{AT}
\end{equation}
where $Q_h, K_h:\R^d \to \R^m$ are query and key functions respectively, $V_h:\R^d \to \R^d$ is a value function, $\sigma$ is a scoring function (e.g. softmax or hardmax) and $H$ denotes the number of heads. 
Also note $X_{N(i)}$ corresponds to the matrix formed by only stacking $\{\vx_j : j\in N(i) \}$ and not all the inputs. 
If $D$ is the complete digraph, we recover the full quadratic attention mechanism of \citet{vaswani2017attention}. 
To simplify our exposition, we will operate on the adjacency matrix $A$ of the graph $D$  even though the underlying graph maybe sparse.
To elaborate, $A \in [0, 1]^{n \times n}$ with $A(i,j)=1$ if query $i$ attends to key $j$ and is zero otherwise.
For example, when $A$ is the ones matrix (as in BERT), it leads to quadratic complexity, since all tokens attend on every other token.
This view of self-attention as a fully connected graph allows us to exploit existing graph theory to help reduce its complexity. 
The problem of reducing the quadratic complexity of self-attention can now be seen as a \emph{graph sparsification problem}. 
It is well-known that random graphs are expanders and can approximate complete graphs in a number of different contexts including in their spectral properties~\citep{spielman2011spectral,hoory2006expander}.
We believe sparse random graph for attention mechanism should have two desiderata: small average path length between nodes and a notion of locality, each of which we discuss below.

Let us consider the simplest random graph construction, known as  Erd\H{o}s-R\'enyi model, where each edge is independently chosen with a fixed probability.
In such a random graph with just $\tilde{\Theta}(n)$ edges, 
the shortest path between any two nodes is logarithmic in the number of nodes~\citep{chung2002average, katzav2018distribution}.
As a consequence, such a random graph approximates the complete graph spectrally and its second eigenvalue (of the adjacency matrix) is quite far from the first eigenvalue ~\citep{benaych2019largest,benaych2020spectral,alt2019extremal}. 
This property leads to a rapid mixing time for random walks in the grpah, which informally suggests that information can flow fast between any pair of nodes.
Thus, we propose a sparse attention where each query attends over $r$ random number of keys i.e. $A(i, \cdot ) = 1$ for $r$ randomly chosen keys (see~\Cref{fig:rnd_atn}).

The second viewpoint which inspired the creation of \bigb is that  most contexts within NLP and computational biology have data which displays a great deal of \emph{locality of reference}. 
In this phenomenon, a great deal of information about a token can be derived from its neighboring tokens. 
Most pertinently, \citet{clark2019does} investigated self-attention models in NLP tasks and concluded that that neighboring inner-products are extremely important. 
The concept of locality, proximity of tokens in linguistic structure, also forms the basis of various linguistic theories such as transformational-generative grammar.
In the terminology of graph theory, clustering coefficient is a measure of locality of connectivity, and is high when the graph contains many cliques or near-cliques (subgraphs that are almost fully interconnected).
Simple Erd\H{o}s-R\'enyi random graphs do not have a high clustering coefficient~\citep{sussman2017clusteringcoeff}, but
a class of random graphs, known as small world graphs, exhibit high clustering coefficient~\citep{watts1998collective}.
A particular model introduced by~\citet{watts1998collective} is of high relevance to us as it achieves a good balance between average shortest path and the notion of locality.
The generative process of their model is as follows:
Construct a regular ring lattice, a graph with $n$ nodes each connected to $w$ neighbors, $w$/2 on each side.

\begin{wraptable}{r}{64mm}
    \vspace{-3mm}
    \centering
    \small
    \begin{tabular}{@{}lrrr@{}}
    \toprule
    Model &  MLM & SQuAD & MNLI \\
    \midrule
    BERT-base      &  64.2 & 88.5 & 83.4 \\
    Random (R)     &  60.1 & 83.0 & 80.2 \\
    Window (W)     &  58.3 & 76.4 & 73.1 \\
    R + W          &  62.7 & 85.1 & 80.5 \\
    \bottomrule
    \end{tabular}
    \caption{Building block comparison @512}
    \label{tab:init}
    \vspace{-2mm}
\end{wraptable}
In other words we begin with a sliding window on the nodes.
Then a random subset ($k$\%) of all connections is replaced with a random connection. 
The other (100 - $k$)\% local connections are retained.
However, deleting such random edges might be inefficient on modern hardware, so we retain it, which will not affect its properties.
In summary, to capture these local structures in the context, in \bigb, we define a sliding window attention, so that during self attention of width $w$, query at location $i$ attends from $i-w/2$ to $i+w/2$ keys.
In our notation, $A(i, i-w/2:i+w/2) = 1$  (see \Cref{fig:wnd:atn}). 
As an initial sanity check, we performed basic experiments to test whether these intuitions are sufficient in getting performance close to BERT like models, while keeping attention linear in the number of tokens.
We found that random blocks and local window were insufficient in capturing all the context necessary to 
compete with the performance of BERT.

The final piece of \bigb is inspired from our theoretical analysis  (\Cref{sec:theory}), which is critical for empirical performance. 
More specifically, our theory utilizes the importance of ``global tokens'' (tokens that attend to all tokens in the sequence and to whom all tokens attend to (see \Cref{fig:gbl_atn}). 
These global tokens can be defined in two ways:
\begin{itemize}[leftmargin=6mm, itemsep=2mm, partopsep=0pt,parsep=0pt]
    \item \bigb-\textsc{itc}: In internal transformer construction (\textsc{itc}), we make some existing tokens ``global'', which attend over the entire sequence. Concretely, we choose a subset $G$ of indices 
    (with $g:=|G|$), such that $A(i, :) = 1$ and $A(:,i) =1$ for all $i \in G$. 
    \item \bigb-\textsc{etc}: In extended transformer construction (\textsc{etc}), we include additional ``global'' tokens such as CLS. Concretely, we add $g$ global tokens that attend to all existing tokens. In our notation, this corresponds to creating a new matrix $B \in [0, 1]^{(N+g)\times (N+g)}$ by adding $g$ rows to matrix $A$, such that $B(i,:) = 1$, and $B(:, i) =1$ for all $i \in \{1,2, \ldots g\}$, and $B(g+i, g+j) = A(i, j) \forall\ i,j \in \{1, \ldots, N \}$. This adds extra location to store context and as we will see in the experiments improves performance.
\end{itemize}

The final attention mechanism for \bigb (\Cref{fig:bigb_atn}) has all three of these properties: queries attend to $r$ random keys, each query attends to $w/2$ tokens to the left of its location and $w/2$ to the right of its location and they contain $g$ global tokens (The global tokens can be from existing tokens or extra added tokens).
We provide implementation details in~\Cref{sec:apndx-impl}.

\section{Theoretical Results about Sparse Attention Mechanism}
\label{sec:theory}

In this section, we will show that that sparse attention mechanisms are as 
powerful and expressive as full-attention mechanisms in two respects.
First, we show that when sparse attention mechanisms are used in a standalone encoder 
(such as BERT), they are Universal Approximators of sequence 
to sequence functions in the style of \citet{Yun19}. 
We note that this property was also explored theoretically in contemporary work~\citet{yun2020on}.
Second, unlike~\citep{yun2020on}, we further show that sparse encoder-decoder transformers are Turing Complete 
(assuming the same conditions defined in~\citep{Perez19}).
Complementing the above positive results, we also show that moving to a sparse-attention mechanism 
incurs a cost, i.e.~there is no free lunch. In~\Cref{sec:limit}, we show lower bounds 
by exhibiting a natural task where any  sufficiently sparse  mechanism will 
require polynomially more layers. 

\subsection{Notation}
The complete Transformer {\em encoder} stack is nothing but the repeated application of a single-layer encoder (with independent parameters). 
We denote class of such Transformer encoders stack, defined using  generalized encoder (\Cref{sec:arch}),  by $\mathcal{T}_{D}^{H,m,q}$ which consists of $H$-heads with head size $m$ and $q$ is the hidden layer size of the output network, and the attention layer is defined by the directed graph $D$. 

The key difference between our proposed attention mechanism to that of \citet{vaswani2017attention,Yun19} is that we add a special token at the beginning of each sequence and assign it a special vector. 
We will refer to this as $\vx_0$. 
Therefore our graph $D$ will have vertex set $\set{0} \cup [n] = \set{0,1,2,\dots,n}$.
We will assume that this extra node and its respective vector will be dropped at the final output layer of transformer. 
To avoid cumbersome notation, we will still treat transformer as mapping sequences 
$\mX \in \R^{n \times d}$ to $\R^{n \times d}$. We will also allow the 
transformer to append position embeddings  $E \in \R^{d\times n}$ to 
matrix $X$ in the input layer. 

Finally, we need to define the function class and distance measure for proving universal approximation property. 
Let $\mathcal{F}_{CD}$ denote the set of continuous functions $f: [0,1]^{n \times d} \to \R^{n \times d} $ 
which are continuous with respect to the topology defined by  $\ell_p$ norm. 
Recall for any  $p \geq 1$, the $\ell_p$ distance is $d_p(f_1,f_2)  = \left(\int \norm{f_1(X) - f_2(X)}_p^p dX \right)^{1/p}$.

\subsection{Universal Approximators}
\begin{definition}
The star-graph $S$ centered at $0$ is the graph defined on $\set{0,\dots, n}$. 
The neighborhood of all vertices $i$ is $N(i) = \{0,i\} $  for $i \in \{1\dots n\}$ and 
$N(0) = \{1,\dots n\}$. 
\end{definition} 
Our main theorem is that the sparse attention mechanism defined by any graph containing $S$
is a universal approximator:
\begin{theorem}
\label{thm:universal}
    Given $1 < p < \infty$ and $\epsilon > 0 $, for any $f \in \mathcal{F}_{CD}$, there exists a 
    transformer with sparse-attention, $g \in \mathcal{T}_D^{H,m,q}$  such that 
    $d_p(f,g)\leq \epsilon$ where $D$ is any graph containing star graph $S$. 
\end{theorem}
To prove the theorem, we will follow the standard proof structure outlined in~\citep{Yun19}.

\textbf{Step 1: Approximate $\mathcal{F}_{CD}$ by piece-wise constant functions.} 
    Since $f$ is a continuous function with bounded domain $[0,1)^{n \times d}$, we will 
    approximate it with a  suitable piece-wise constant function. This is accomplished by a suitable partition of the  region $[0,1)$ into a grid of granularity $\delta$ to get a discrete set $\Gd$. Therefore, we can assume that we are dealing with a function 
    $\bar{f}: \Gd \to \R^{n \times d}$, where $d_p(f,\bar{f}) \leq \frac{\epsilon}{3}$. 
    
\textbf{Step 2:  Approximate piece-wise constant functions by modified transformers.} 
    This is the key step  of the proof where  the self-attention mechanism is used to generate a \emph{contextual-mapping} of the input. Informally, a contextual mapping is a unique code
    for the pair consisting of a matrix $(\mX,\vx_{i})$ and a column. Its uniqueness allows
    the Feed forward layers to use each code to map it to a unique output column. 
    
    The main technical challenge is computing the contextual mapping using only sparse attention mechanism. 
    This was done in \cite{Yun19} using a ``selective'' shift operator which shift up entries that are in a 
    specific interval. Key to their proof was the fact that the shift, was exactly the range of the largest 
    entry to the smallest entry.  
    
    Creating a contextual mapping with a sparse attention mechanism is quite a challenge.
    In particular, because each query only attends to a few keys, it is not at all clear
    that sufficient information can be corralled to make a contextual embedding of the 
    entire matrix. To get around this, we develop a sparse shift operator which shifts the entries of the matrices if they lie in a certain range. The exact amount of the shift is controlled by the directed sparse attention graphg $D$. The second key ingredient is the use of additional global token.  By carefully applying the operator
    to a set of chosen ranges, we will show that each column will contain a unique mapping of the
    full mapping.   Therefore, we can   
    augment the loss of inner-products in the self attention mechanism by using multiple 
    layers and an auxiliary global token. 
    
\textbf{Step 3: Approximate modified transformers by original Transformers}: The final 
step is to approximate the modified transformers by the original transformer which uses ReLU and 
softmax. 

We provide the full details in~\Cref{sec:apndx-universal}.

\subsection{Turing Completeness}
Transformers are a very general class. In the original paper of \citet{vaswani2017attention}, they were used 
in both an encoder and a decoder. While the previous section outlined how powerful just the encoders were,  another natural question is to ask what the
additional power of both a decoder along with an encoder is? \citet{Perez19} showed that the 
full  transformer based on a quadratic attention mechanism is Turing Complete. This result 
makes one  unrealistic assumption, which is that the model works on arbitrary precision 
model. Of course,  this is necessary as otherwise, Transformers are bounded finite 
state machines and cannot be Turing Complete. 

It is natural to ask if the full attention mechanism is necessary. Or can a 
sparse attention mechanism also be used to simulate any Turing Machine? 
We show that this is indeed the case: we can use a sparse encoder and sparse decoder
to simulate any Turing Machine. 

To use the sparse attention mechanism in the transformer architecture, we need to 
define a suitable modification where each token only reacts to previous tokens. 
Unlike the case for BERT, where the entire attention mechanism is applied once, in  full 
transformers, the sparse attention mechanism at decoder side is used token by token. 
Secondly the work of \citet{Perez19}, uses each token as a representation of the tape 
history and  uses the full attention to move and retrieve the correct tape symbol.
Most of the construction of \citet{Perez19} goes through for sparse attentions, except for their 
addressing scheme to point back in history (Lemma B.4 in \citep{Perez19}).
We show how to simulate this using a sparse attention mechanism and defer the 
details to~\Cref{sec:apndx-turing}.

\subsection{Limitations}
\label{sec:limit}
We demonstrate a natural task which can be solved by the full attention mechanism in $O(1)$-layers.
However, under standard complexity theoretic assumptions, this problem requires 
$\tilde{\Omega}(n)$-layers for any sparse attention layers with $\tilde{O}(n)$ edges (not just \bigb). (Here $\tilde{O}$ hides poly-logarthmic  factors). 
Consider the simple problem of finding the corresponding furthest vector 
for each vector in the given sequence of length $n$. Formally,

\textbf{Task 1.} \ Given $n$ unit vectors $\{u_1,\dots,u_n\}$, find $f(u_1,\dots,u_n) \to (u_{1^*}, \dots, u_{n^*})$  where for a fixed $j \in [n]$, we define $ j^* = \argmax_{k} \|u_k - u_j\|_2^2$.

Finding vectors that are furthest apart boils down to minimize 
inner product search in case of unit vectors. For a full-attention mechanism 
with appropriate query 
and keys, this task is very easy as we can evaluate all pair-wise inner products. 

The impossibility for sparse-attention follows from hardness results 
stemming from Orthogonal Vector Conjecture(OVC) \citep{abboud2014consequences,abboud2015tight,backurs2015edit,williams2005new}. The OVC  is a widely used assumption
in fine-grained complexity. Informally, it states that one cannot determine if the minimum inner product among 
$n$ boolean vectors is $0$ in subquadratic time. In~\Cref{sec:apndx-limit}, we show a 
reduction using  OVC to show that if a transformer $g \in \mathcal{T}_D^{H=1,m=2d,q=0}$ for 
any sparse directed graph $D$ can evaluate the Task $1$, it can solve the orthogonal vector problem.
\begin{proposition}
    There exists a single layer full self-attention $g\in\mathcal{T}^{H=1,m=2d,q=0}$ that can 
    evaluate Task 1,  i.e. $g(u_1,...,u_n) = [u_{1^*},\dots, u_{n^*}]$, but for any sparse-attention 
    graph $D$ with $\tilde{O}(n)$ edges (i.e.~inner product evaluations), would require $\tilde{\Omega}(n^{1-o(1)})$ layers.
\end{proposition}
\vspace{-3mm}
We give a formal proof of this fact in~\Cref{sec:apndx-limit}.

\section{Experiments: Natural Language Processing} \label{sec:expt-nlp}

In this section our goal is to showcase benefits of modeling longer input sequence for NLP tasks, for which we select three representative tasks.
We begin with basic masked language modeling (MLM; \citealt{devlin2018bert}) to check if better contextual representations can be learnt by utilizing longer contiguous sequences.
Next, we consider QA with supporting evidence, for which capability to handle longer sequence would allow us to retrieve more evidence using crude systems like TF-IDF/BM25.
Finally, we tackle long document classification where discriminating information may not be located in first 512 tokens.
Below we summarize the results for \bigb using sequence length 4096\footnote{code available at \url{http://goo.gle/bigbird-transformer}}, while we defer all other setup details including computational resources, batch size, step size, to~\Cref{sec:apndx-expt-nlp}.

\paragraph{Pretraining and MLM}
We follow \citep{devlin2018bert, liu2019roberta} to create base and large versions of \bigb and pretrain it using MLM objective. 
This task involves predicting a random subset of tokens which have been masked out. 
We use four standard data-sets for pretraining (listed in \Cref{sec:app-expt-nlp:mlm}, \Cref{tab:mlm_data}), warm-starting from the public RoBERTa checkpoint\footnote{\url{https://github.com/pytorch/fairseq/tree/master/examples/roberta}}. 
We compare performance in predicting the masked out tokens in terms of bits per character, following~\citep{beltagy2020longformer}. 
As seen in~\Cref{sec:app-expt-nlp:mlm}, ~\Cref{tab:mlm_bpc},  both \bigb and Longformer perform better than limited length RoBERTa, with \bigb-\textsc{etc} performing the best.
We note that we trained our models on a 
reasonable $16GB$ memory/chip with batch size of 32-64.
Our memory efficiency is due to efficient blocking and sparsity structure of 
the sparse attention mechanism described in~\Cref{sec:arch}. 
\begin{table}
    \centering
    \small
    \begin{tabular}{@{}l c c c ccc c cc c c@{}}
    \toprule
    \multirow{2}{*}{Model} & & 
    \multicolumn{3}{c}{HotpotQA} & &
    \multicolumn{2}{c}{NaturalQ} & &
    TriviaQA & & 
    WikiHop\\
    \cmidrule{3-5} \cmidrule{7-8} \cmidrule{10-10}  \cmidrule{12-12} 
               && Ans & Sup & Joint  &&  LA & SA && Full && MCQ  \\
    \midrule
    RoBERTa     && 73.5 & 83.4 & 63.5 &&  -   &  -  && 74.3 && 72.4  \\
    Longformer  && 74.3 & 84.4 & 64.4 &&  -   &  -  && 75.2 && 75.0  \\
    \bigb-\textsc{itc}  && \textbf{75.7} & 86.8 & 67.7 && 70.8 & 53.3 && \textbf{79.5} && \textbf{75.9} \\
    \bigb-\textsc{etc}  && 75.5 & \textbf{87.1} & \textbf{67.8} && \textbf{73.9} & \textbf{54.9} && 78.7 && \textbf{75.9} \\
    \bottomrule
    \end{tabular}
    \vspace{2mm}
    \caption{QA Dev results using Base size models. We report accuracy for WikiHop and F1 for HotpotQA, Natural Questions, and TriviaQA.}
    \label{tab:QADev}
\end{table}

\begin{table}
    \centering
    \small
    \begin{tabular}{@{}l c c ccc c cc c cc@{}}
    \toprule
    \multirow{2}{*}{Model} & 
    \multicolumn{3}{c}{HotpotQA} & &
    \multicolumn{2}{c}{NaturalQ} & &
    \multicolumn{2}{c}{TriviaQA} & &
    WikiHop\\
    \cmidrule{2-4} \cmidrule{6-7} \cmidrule{9-10} \cmidrule{12-12} 
    & Ans & Sup & Joint  &&  LA & SA && Full & Verified && MCQ  \\
    \midrule
    HGN \citep{fang2019hierarchical}         & \textbf{82.2} & 88.5 & \textbf{74.2} &&  -   &  -   &&  -   &  -   &&  -   \\
    GSAN                                     & 81.6 & 88.7 & 73.9 &&  -   &  -   &&  -   &  -   &&  -   \\
    ReflectionNet \citep{gong2020reflection} &  -   &  -   &  -   && 77.1 & \textbf{64.1} &&  -   &  -   &&  -   \\
    RikiNet-v2 \citep{liu2020rikinet}           &  -   &  -   &  -   && 76.1 & 61.3 &&  -   &  -   &&  -   \\
    Fusion-in-Decoder \citep{izacard2020fid} &  -   &  -   &  -   &&  -   &  -   && 84.4 & 90.3 &&  -   \\
    SpanBERT \citep{joshi2020spanbert}       &  -   &  -   &  -   &&  -   &  -   && 79.1 & 86.6 &&  -   \\ 
    MRC-GCN \citep{tang2020multi}            &  -   &  -   &  -   &&  -   &  -   &&  -   &  -   && 78.3 \\
    MultiHop \citep{chen2019multi}           &  -   &  -   &  -   &&  -   &  -   &&  -   &  -   && 76.5 \\
    Longformer \citep{beltagy2020longformer} & 81.2 & 88.3 & 73.2 &&  -   &  -   && 77.3 & 85.3 && 81.9 \\
    \midrule
    \bigb-\textsc{etc} & 81.2 & \textbf{89.1} & 73.6 && \textbf{77.8} & 57.9 && \textbf{84.5} & \textbf{92.4} && \textbf{82.3} \\
    \bottomrule
    \end{tabular}
    \vspace{2mm}
    \caption{Fine-tuning results on \textbf{Test} set for QA tasks. The Test results (F1 for HotpotQA, Natural Questions, TriviaQA, and Accuracy for WikiHop) have been picked from their respective leaderboard. For each task the top-3 leaders were picked not including \bigb-etc. \textbf{For Natural Questions Long Answer (LA), TriviaQA, and WikiHop, \bigb-ETC is the new state-of-the-art}. On HotpotQA we are third in the leaderboard by F1 and second by Exact Match (EM).}
    \label{tab:QATest}
\end{table}

\paragraph{Question Answering (QA)} 
We considered following four challenging datasets:
\vspace{-1mm}
\begin{enumerate}[leftmargin=6mm, itemsep=2mm, partopsep=0pt,parsep=0pt]
\item  Natural Questions~\citep{kwiatkowski2019natural}: 
For the given question, find a short span of answer (SA) from the given evidences as well highlight the paragraph from the given evidences 
containing information about the correct answer (LA).
\item HotpotQA-distractor~\citep{yang2018hotpotqa}: Similar to natural questions, it requires finding the answer (Ans) as well as the supporting facts (Sup) over different documents needed for multi-hop reasoning from the given evidences. 
\item TriviaQA-wiki~\citep{JoshiTriviaQA2017}: We need to provide an answer for the given question using provided Wikipedia evidence, however, the answer might not be present in the given evidence. On a smaller \emph{verified} subset of question, the given evidence is guaranteed to contain the answer. Nevertheless, we model the answer as span selection problem in this case as well.
\item  WikiHop~\citep{welbl2018constructing}: Chose correct option from multiple-choice questions (MCQ), by aggregating information spread across multiple documents given in the evidences.
\end{enumerate}
As these tasks are very competitive, multiple highly engineered systems have been designed specific each dataset confirming to respective output formats.
For a fair comparison, we had to use some additional regularization for training \bigb, details of which are provided in~\Cref{sec:app-expt-nlp:qa} along with exact architecture description.
We experiment using the base sized model and select the best configuration on the development set for each dataset (as reported in~\Cref{tab:QADev}).
We can see that \bigb-\textsc{etc}, with expanded global tokens consistently outperforms all other models.
Thus, we chose this configuration to train a large sized model to be used for evaluation on the hidden test set.

In~\Cref{tab:QATest}, we compare \bigb-\textsc{etc} model to top-3 entries from the leaderboard excluding \bigb. 
One can clearly see the importance of using longer context as both Longformer and \bigb outperform models with smaller contexts.
Also, it is worth noting that \bigb submission is a single model, whereas the other top-3 entries for Natural Questions are  ensembles, which might explain the slightly lower accuracy in exact answer phrase selection.

\paragraph{Classification} 
We experiment on datasets of different lengths and contents, specifically various document classification and GLUE tasks.
Following BERT, we used one layer with cross entropy loss on top of the first [CLS] token.
We see that gains of using \bigb are more significant when we have longer documents and fewer training examples.
For instance, using base sized model,
\bigb improves state-of-the-art for Arxiv dataset by about $\bm{5\%}$ \textbf{points}.
On Patents dataset, there is improvement over using simple BERT/RoBERTa, but given the large size of training data the improvement over SoTA (which is not BERT based) is not significant.
Note that this performance gain is not seen for much
smaller IMDb dataset.
Along with experimental setup detail, we present detailed results in~\Cref{sec:app-expt-nlp:cls} which show competitive performance.

\begin{table}[b]
    \centering
    \small
    \begin{tabular}{@{}p{1mm}l @{}p{3mm}@{} rrr @{}p{5mm}@{} rrr @{}p{5mm}@{} rrr@{}}
    \toprule
     \multicolumn{2}{l}{\multirow[b]{2}{*}{\hspace{-2mm}\normalsize{Model}}} & &
     \multicolumn{3}{c}{Arxiv} & & \multicolumn{3}{c}{PubMed} & & \multicolumn{3}{c}{BigPatent}\\
    \cmidrule{4-6} \cmidrule{8-10} \cmidrule{12-14}
    & & & R-1 & R-2 & R-L & & R-1 & R-2 & R-L & & R-1 & R-2 & R-L \\
    \midrule
    \multirow{10}{*}{\rotatebox[origin=c]{90}{Prior Art}}
    & SumBasic~\citep{nenkova2005impact}          & & 29.47 &  6.95 & 26.30 & & 37.15 & 11.36 & 33.43 & & 27.44 & 7.08 & 23.66\\
    & LexRank~\citep{erkan2004lexrank}          & & 33.85 & 10.73 & 28.99 & & 39.19 & 13.89 & 34.59 & & 35.57 & 10.47 & 29.03 \\
    & LSA~\citep{wiseman2017challenges}               & & 29.91 &  7.42 & 25.67 & & 33.89 &  9.93 & 29.70 & & - & - & - \\
    & Attn-Seq2Seq~\citep{sutskever2014sequence}    & & 29.30 &  6.00 & 25.56 & & 31.55 &  8.52 & 27.38 & & 28.74 & 7.87 & 24.66 \\
    & Pntr-Gen-Seq2Seq~\citep{see2017get} & & 32.06 &  9.04 & 25.16 & & 35.86 & 10.22 & 29.69 & &  33.14 & 11.63 & 28.55 \\
    & Long-Doc-Seq2Seq~\citep{cohan2018discourse} & & 35.80 & 11.05 & 31.80 & & 38.93 & 15.37 & 35.21 & & - & - & - \\
    & Sent-CLF~\citep{subramanian2019extractive}  & & 34.01 &  8.71 & 30.41 & & 45.01 & 19.91 & 41.16 & & 36.20 & 10.99 & 31.83 \\
    & Sent-PTR~\citep{subramanian2019extractive}  & & 42.32 & 15.63 & 38.06 & & 43.30 & 17.92 & 39.47 & & 34.21 & 10.78 & 30.07 \\
    & Extr-Abst-TLM~\citep{subramanian2019extractive} & & 41.62 & 14.69 & 38.03 & & 42.13 & 16.27 & 39.21 & & 38.65 & 12.31 & 34.09 \\
    & Dancer~\citep{gidiotis2020divide}  & & 42.70 & 16.54 & 38.44 & & 44.09 & 17.69 & 40.27 & & - & - & - \\
    \midrule
    \multirow{4}{*}{\rotatebox[origin=c]{90}{Base}}
    & Transformer & & 28.52 &  6.70 & 25.58 & & 31.71 &  8.32 & 29.42 & & 39.66 & 20.94 & 31.20 \\
    & \; + RoBERTa~\citep{rothe2019leveraging} & & 31.98 &  8.13 & 29.53  & & 35.77 & 13.85 & 33.32 & & 41.11 & 22.10 & 32.58 \\
    & \; + Pegasus~\citep{zhang2019pegasus}  & & 34.81 & 10.16 & 30.14 & & 39.98 & 15.15 & 35.89 & & 43.55 & 20.43 & 31.80 \\
    & \bigb-RoBERTa & & \underline{41.22} & \underline{16.43} & \underline{36.96} & & \underline{43.70} & \underline{19.32} & \underline{39.99} & & \underline{55.69} & \underline{37.27} & \underline{45.56} \\
    \midrule
    \multirow{3}{*}{\rotatebox[origin=c]{90}{Large}}
    & Pegasus (Reported)~\citep{zhang2019pegasus} & & 44.21 & 16.95 & 38.83 & & 45.97 & 20.15 & 41.34 & & 52.29 & 33.08 & 41.75 \\
    & Pegasus (Re-eval)                          & & 43.85 & 16.83 & 39.17 & & 44.53 & 19.30 & 40.70 & & 52.25 & 33.04 & 41.80 \\
    & \bigb-Pegasus & & \textbf{46.63} & \textbf{19.02} & \textbf{41.77}   & & \textbf{46.32} & \textbf{20.65} & \textbf{42.33}   & & \textbf{60.64} & \textbf{42.46} & \textbf{50.01}  \\
    \bottomrule
    \end{tabular}
    \vspace{2mm}
    \caption{Summarization ROUGE score for long documents.}
    \label{tab:long_sum_res}
\end{table}

\subsection{Encoder-Decoder Tasks}
\label{sec:seq2seq}
For an encoder-decoder setup, one can easily see that both suffer from quadratic complexity due to the full self attention.
We focus on introducing the sparse attention mechanism of \bigb only at the encoder side.
This is because, in practical generative applications, the length of output sequence is typically small as compared to the input. 
For example for text summarization, we see in realistic scenarios (c.f.~\Cref{sec:appn_summarization}~\Cref{tab:long_sum_data}) that the median output sequence length is $\sim 200$ where as the input sequence's median length is $>3000$.
For such applications, it is more efficient to use sparse attention mechanism for the encoder and full self-attention for the decoder.

\paragraph{Summarization}
Document summarization is a task of creating a short and accurate summary of a text document. 
We used three long document datasets for testing our model details of which are mention in~\Cref{tab:long_sum_data}.
In this paper we focus on abstractive summarization of long documents where using a longer contextual encoder should improve performance.
The reasons are two fold:
First, the salient content can be evenly distributed in the long document, not just in first 512 tokens, and this is by design in the BigPatents dataset~\citep{sharma2019bigpatent}.
Second, longer documents exhibit a richer discourse structure and summaries are considerably more abstractive, thereby observing more context helps.
As has been pointed out recently ~\citep{rothe2019leveraging,zhang2019pegasus}, pretraining helps in generative tasks, we warm start from our general purpose MLM pretraining on base-sized models as well as utilizing state-of-the-art summarization specific pretraining from Pegasus~\citep{zhang2019pegasus} on large-sized models.
The results of training \bigb sparse encoder along with full decoder on these long document datasets are presented in~\Cref{tab:long_sum_res}.
We can clearly see modeling longer context brings significant improvement. 
Along with hyperparameters, we also present results on shorter but more widespread datasets in~\Cref{sec:appn_summarization}, which show that using sparse attention does not hamper performance either.

\section{Experiments: Genomics} 
\label{sec:expt-bio}

There has been a recent upsurge in using deep learning for genomics data \citep{tampuu2019viraminer, zhang2019ncnet, busia2019deep}, which has resulted in improved performance on several biologically-significant tasks such as
promoter site prediction \citep{oubounyt2019deepromoter}, methylation analysis \citep{levy2020methylnet}, 
predicting functional effects of non-coding variant \citep{zhou2015predicting}, etc.
These approaches consume DNA sequence fragments as inputs, and therefore we believe longer input sequence handling capability of \bigb would be beneficial as many functional effects in DNA are highly non-local \citep{buldyrev1995long}.
Furthermore, taking inspiration from NLP, we learn powerful contextual representations for DNA fragments utilizing abundant unlabeled data (e.g. human reference genome, Saccharomyces Genome Database) via MLM pretraining.
Next, we showcase that our long input \bigb along with the proposed pretraining significantly improves performances in two downstream tasks.
Detailed experimental setup for the two tasks are provided in~\Cref{sec:apndx-expt-bio}.

\begin{wraptable}{r}{39mm}
    \vspace{-4mm}
    \centering
    \small
    \begin{tabular}{@{}lr@{}}
    \toprule
    Model &  BPC \\
    \midrule
    SRILM \cite{liang2012segmenting}  & 1.57  \\
    BERT (sqln. 512)  & 1.23 \\
    \midrule
    \bigb (sqln. 4096)   & \textbf{1.12} \\
     \bottomrule
    \end{tabular}
    \caption{MLM BPC}
    \label{tab:gml}
    \vspace{-3mm}
\end{wraptable}
\paragraph{Pre-training and MLM}
As explored in \citet{liang2012segmenting}, instead of operating on base pairs, we propose to first segment DNA into tokens so as to further increase the context length (\Cref{sec:apndx-expt-bio}, \Cref{fig:apndx_mlm_data}).
In particular, we build a byte-pair encoding~\citep{kudo2018sentencepiece} table for the DNA sequence of size 32K, with each token representing 8.78 base pairs on average.
We learn contextual representation of these token on the human reference genome (GRCh37)\footnote{\url{https://www.ncbi.nlm.nih.gov/assembly/GCF_000001405.13/}} using MLM objective.
We then report the bits per character (BPC) 
on a held-out set in \Cref{tab:gml}. 
We find that attention based contextual representation of DNA does improve BPC, which is further improved by using longer context.

 \begin{wraptable}{r}{35mm}
    \vspace{-4mm}
    \centering
    \small
    \begin{tabular}{@{}lr@{}}
    \toprule
    Model &  F1 \\
    \midrule
    CNNProm~\citep{umarov2017recognition}  & 69.7  \\
    DeePromoter~\citep{oubounyt2019deepromoter}  & 95.6 \\
    \midrule
    \bigb   & \textbf{99.9} \\
     \bottomrule
    \end{tabular}
    \caption{Comparison.}
    \label{tab:gpp}
    \vspace{-3mm}
\end{wraptable}
\paragraph{Promoter Region Prediction}
Promoter is a DNA region typically located upstream of the gene, which is the site of transcription initiation.
Multiple methods have been proposed to identify the promoter regions in a given DNA sequence~\citep{yang2017exploiting, lin2017identifying, bharanikumar2018promoterpredict, xiao2019ipsw, oubounyt2019deepromoter}, as it is an important first step in understanding gene regulation.
The corresponding machine learning task is to classify a given DNA fragment as promoter or non-promoter sequence. We use the dataset compiled by \citet{oubounyt2019deepromoter} which was built from Eukaryotic Promoter Database (EPDnew) \citep{dreos2013epd} 
\footnote{ \url{https://epd.epfl.ch/human/human_database.php?db=human}}. 
We finetuned the pretrained \bigb model from above, using the training data and report F1 on test dataset. 
We compare our results to the previously reported best method in \Cref{tab:gpp}.
We see that  
\bigb achieve nearly perfect accuracy with a $5\%$ jump from the previous best reported accuracy.

\begin{wraptable}{r}{57mm}
    \centering
    \small
    \begin{tabular}{@{}lrrr@{}}
    \toprule
    Model &  TF & HM & DHS \\
    \midrule
    gkm-SVM~\citep{ghandi2014enhanced}  & 89.6 & - & -  \\
    DeepSea~\citep{zhou2015predicting}  & 95.8 & 85.6 & \textbf{92.3} \\
    \midrule
    \bigb   & \textbf{96.1} & \textbf{88.7} & 92.1 \\
     \bottomrule
    \end{tabular}
    \caption{Chromatin-Profile Prediction}
    \label{tab:gnve} 
    \vspace{-3mm}
\end{wraptable}
\paragraph{Chromatin-Profile Prediction} Non-coding regions of DNA do not code for proteins.  
Majority of diseases and other trait associated single-nucleotide polymorphism are correlated to non-coding genomic variations \citep{zhou2015predicting, khurana2016role}.
Thus, understanding the functional effects of non-coding regions of DNA is a very important task. An important step in this process, as defined by \citet{zhou2015predicting}, is to predict large-scale chromatin-profiling from non-coding genomic sequence. 
To this effect, DeepSea \citep{zhou2015predicting}, compiled 919 chromatin-profile of 2.4M non-coding variants from Encyclopedia of DNA Elements (ENCODE)\footnote{\url{https://www.encodeproject.org/}} and Roadmap Epigenomics projects\footnote{\url{http://www.roadmapepigenomics.org/}}. 
The corresponding ML task is to predict, for a given non-coding region of DNA, these 919 chromatin-profile including $690$ transcription factors (TF) binding profiles for $160$ different TFs, $125$ DNase I sensitivity (DHS) profiles and $104$ histone-mark (HM) profiles. 
We jointly learn 919 binary classifiers to predict these functional effects from sequence of DNA fragments.
On held-out chromosomes, we compare AUC with the baselines in~\Cref{tab:gnve} and see that we significantly improve on  performance on the harder task HM, which is known to have longer-range correlations~\citep{gates2017histone} than others.

\section{Conclusion}
We propose \bigb: a sparse attention mechanism that is linear in the 
number of tokens. \bigb satisfies a number of theoretical results: 
 it is a universal approximator of sequence to sequence functions and is also
Turing complete.  Theoretically,  we use the power of extra global tokens
preserve the expressive powers of the model. 
We complement these results by showing that moving to sparse attention mechanism
do incur a cost. 
Empirically, \bigb  gives  \emph{state-of-the-art} performance on a number of NLP tasks such as 
question answering and long document classification. We  further introduce attention based 
contextual language model for DNA and fine-tune it for down 
stream tasks such as  promoter region prediction and predicting effects of non-coding variants.

\bibliographystyle{abbrvnat}
\bibliography{ref}

\newpage
\appendix
\begin{center}{
\Large
\textbf{Big Bird: Transformers for Longer Sequences -- Appendix}
}
\end{center}

\newcommand{\Gdpos}{\ensuremath{\mathbf{G}^{E}_{\delta}}}
\newcommand{\FCDpos}{\ensuremath{\mathcal{F}^{E}_{CD}}}
\newcommand{\tx}{\ensuremath{\tilde{x}}}
\newcommand{\tl}{\ensuremath{\tilde{f}}}

\section{Universal Approximators} \label{sec:apndx-universal}

\subsection{Notation}
\label{sec:apndx-enc-notation}
We begin by setting up some notations following \citet{Perez19} to formally describe the complete architecture of Transformers. 
A single layer of Transformer encoder is a parametric function $\Enc$ receiving a sequence $\mX = (\vx_1, ..., \vx_n)$ of vectors in $\R^d$ and returning a sequence $\mZ = (\vz_1, ..., \vz_n)$ of the same length.
Each $\vz_i$ is a $d$ dimensional vector as well.
We interchangeably treat the sequence $\mX$ as a matrix in $\R^{n\times d}$.
$\Enc$ has two components: 
\begin{enumerate}[leftmargin=6mm, itemsep=2mm, partopsep=0pt,parsep=0pt]
    \item An attention mechanism $\Attn$ that takes in the sequence $\mX$ and returns sequence $(\va_1, ..., \va_n)$ of the same length and dimensionality; and 
    \item A two layer fully connected network $O$ that takes in a vector in $\R^d$ and returns a vector in $\R^d$. 
\end{enumerate}
Then $i$-th output vector of $\Enc(\mX)$ is computed as follows:
\begin{align}
    \vz_i = O(\va_i) + \va_i \qquad\text{where}\qquad \va_i = \Attn(\mX)_i + \vx_i
\end{align}
Now it remains to define $\Attn$ and $O$ which we do next.

As described in~\Cref{sec:arch}, an attention mechanism is parameterized by three functions: $Q,K,V: \R^{d} \to \R^{m}$. 
In this paper, we assume that they are simply matrix products: $Q(\vx) = \vx W_Q $, 
$K(\vx) = \vx W_K $, and $V(\vx) = \vx W_V $, where $W_Q, W_K, W_V \in \R^{d \times m}$ and 
$W_V\in \R^{d \times d}$.
In reality a multi-headed attention is used, i.e. we have not only one,
but $H$-sets of Query/Key/Value weight matrices, $W_Q^h, W_V^h, W_K^h \text{ for }h=1,...,H$.
Thus, for a directed graph $D$ over $[n]$, the $i^\text{th}$ output vector of the generalized attention mechanism would be
\begin{align}
\Attn_D(\mX)_i &=  \sum_{h=1}^H \sigma \left((\vx_i W_Q^h) (\mX_{N(i)} W_K^h)^T  \right) \cdot (\mX_{N(i)} W_V^h ) \label{AT_app} \tag{AT}
\end{align}
where $N(i)$ denote the out-neighbors set of node $i$ in $D$.
In other words, the set of arcs (directed edges) in $D$ represents the set of inner products that our attention mechanism will consider.
Also recall that $\sigma$ is a scoring function such as softmax or hardmax.

Lastly, we define the output fully connected network as follows:
\begin{align*}
    O(\va_i) &= \relu \left(\va_i W_1 + b_1 \right)  W_2\cdot 
                    + b_2 \label{FF} \tag{FF}
\end{align*}
Here $W_1 \in \R^{ d\times q }$, $W_2 \in \R^{q \times d}$, $b_1\in\R^p$, and $b_2\in\R^d$ are parameters of 
output network $O$. 

\textbf{Additional Notation} We introduce a few pieces of additional notation that will be useful. 
Let $[a,b)_{\delta} = \{ a, a+\delta, \dots, a + \lfloor \frac{b-a}{\delta} \rfloor \cdot \delta \} $. Therefore,
$[0,1)_{\delta} = \{ 0, \delta, 2\delta, \dots, (1-\delta)\}$. 
We use $\mathbf{1} [ \mathcal{E}]$ to denote the indicator variable; it is $1$ if the event $\mathcal{E}$ occurs and $0$ otherwise. 

\subsection{Proof}
In this section, we will present the full proof of~\cref{thm:universal}. 
The proof will contain three parts. 
The first and the third part will largely follow standard techniques. The main innovation lies is in the second part.

\subsubsection{Approximate \texorpdfstring{$\mathcal{F}_{CD}$}{Fcd} by piece-wise constant functions}
    First, we consider a suitable partition of the region $(0,1)$ into a 
    grid of granularity $\delta$, which we denote by $G_\delta$.  We do this using Lemma~8 from \citet{Yun19}, which we restate for completeness:
    \begin{lemma}[Lemma 8~\citep{Yun19}] \label{lem:piecewise}
       For any given $f \in \mathcal{F}_{CD}$ and $1\leq p \leq \infty$, there exists a $\delta >0$ such that
        there exists a piece-wise constant function $\bar{f}$ with $d_{p}(f,\bar{f}) \leq \frac{\epsilon}{3}$. 
         Concretely, $\bar{f}$ is defined as 
    \[ \bar{f}(X) = \sum_{P \in \Gd} f(P) \cdot \mathbf{1} \left[ \norm{\relu(X-P)}_{\infty} \leq  \delta  \right]  \] 
    \end{lemma}

    Since transformers can learn a positional embedding $E$, without any loss of generality, 
    we can consider the translated function. In particular, define
    \[ E = \begin{bmatrix}
    0 & 0 & 0 & \dots & 0  \\
    \delta^{-d} & \delta^{-d} & \delta^{-d} & \dots & \delta^{-d}  \\
    \delta^{-2d} & \delta^{-2d} & \delta^{-2d} & \dots & \delta^{-2d}  \\
    \vdots \\
    \delta^{-(n-1)d}  & \delta^{-(n-1)d}  & \delta^{-(n-1)d}  & \dots & \delta^{-(n-1)d}  \\
    \end{bmatrix} \] 
    
    We will try to approximate $g(X) = f(X - E)$ where $g$ is defined on the domain 
    $[0,1]^d \times [\delta^{-d}, \delta^{-d}+1]^d \times \dots \times [\delta^{-(n-1)d}, \delta^{-(n-1)d}+1]^d  $. To do so, we will apply a suitable modification of ~\Cref{lem:piecewise},
    which will consider the discretized grid 
    \[ \Gdpos := [0,1]_{\delta}^d \times [\delta^{-d}, \delta^{-d}+1]_{\delta}^d \times \dots \times [\delta^{-(n-1)d}, \delta^{-(n-1)d}+1]_{\delta}^d. \] 
    
    Therefore, it suffices to approximate a function $\bar{f}: \Gdpos \to \R^{n \times d}$ 
   defined as  
    \[ \bar{f}(X) = \sum_{P \in \Gdpos} f(P-E) \cdot \mathbf{1} \left[ \norm{\relu(X-P)}_{\infty} \leq  \delta  \right].  \] 

\subsubsection{Contextual Mappings and Sparse Attention Mechanisms}

Throughout this section, we will assume that we are given a function that has an extra global token 
at index $0$ and all vectors have an extra dimension appended to them. The latter assumption is  
without loss of generality as we can use the Feed-Forward Network to append sparse dimensions. 
In particular, we will associate $X \in \R^{(n+1) \times (d+1)}$ where we write $X = (x_0,x_1,\dots,x_n)$. 
Although our function is only defined for $\Gdpos \subset \R^{n \times d}$, 
we can amend the function in a natural way by making it ignore the first column. 
To avoid excessive clutter, we will assume that the function value is evaluated on the last $n$ columns.

The main idea in this section is the use of contextual mapping to enable Transformers 
to compute any discretized function. A contextual mapping is an unique encoding of each 
tuple  $(X,x_i) $ where  $X \in \Gdpos$, and  each column $x_i \in [\delta^{-(i-1)d}, \delta^{-(i-1)d}+1)^d_{\delta}$ for all $i \in [n]$. 
We restate the definition adapted to our setting below
\begin{definition}[Defn 3.1~\cite{Yun19}] (Contextual Mapping)
\label{defn:contextual-mapping}
A contextual mapping is a function mapping $q: \Gdpos \to \R^{n}$  if it satisfies the following:
    \begin{enumerate}
        \item For any $P \in \Gdpos$,  $q(P)$  contains distinct entries.
        \item For any two $P, P' \in \Gdpos$ with $P \neq P'$, all entries of $q(P)$ and $q(P')$ 
            are distinct. 
    \end{enumerate}
\end{definition}

The key technical novelty of the proof is computing a contextual mapping using only the 
sparse attention mechanism.  We create a 
``selective shift'' operator which only shifts entries of a vector that 
lie in  a certain range.  We will use this shift operator strategically to ensure that 
we attain a contextual mapping at the end of the process. 
The lemma below, which is based on parts of the proof of Lemma 6 of \cite{Yun19}, 
states that we can implement a suitable ``selective'' shift operator using a 
sparse attention mechanism. 
\begin{lemma}
    Given a function $\psi : \R^{(n+1)\times (d+1)} \times \R^2 \to \R^{(n+1) \times 1}$ and a vector $u \in \R^{d+1}$ and a sparse attention mechanism
    based on the directed graph $D$, we can implement a selective shift operator that receives as input
    a matrix $X \in \R^{(n+1) \times (d+1) }$  and outputs $X +  \rho \cdot \psi_u(X,b_1,b_2)$ where
     \[ 
    \psi_u(Z; b_1, b_2)_{i} =   \begin{cases}
        (\max_{j \in N(i)} u^T Z_{j} - \min_{j \in N(i)} u^T Z_{j})e_1  & \text{ if } b_1\leq  u^T Z_{j} \leq b_2\\
        0 & \text{ else. } 
    \end{cases}
    \]
    Note that $e_1 \in R^{d+1}$ denotes $(1,0,\dots,0)$. 
\end{lemma}
\begin{proof}
    Consider the function , which can be implemented by a sparse attention mechanism :
    \[ \tilde{\psi}(X,b)_i = \sigma_H \Big[ (u^T \cdot X_i)^T \cdot (u^T X_{N(i)} - b1_{N(i)}^T) e^{(1)} (u^T X_{N(i)})   \Big] \] 
    This is because the Key, Query and Value functions are simply affine transformations of $X$. 
    
    Given any graph $D$, the above function will evaluate to the following:
    \[ 
   \tilde{\psi}(Z; b)_{i} =  \begin{cases}
        (\max_{j \in N(i)} u^T Z_{j}) e_1 & \text{ if } u^T Z_{j} > b\\
        (\min_{j \in N(i)} u^T Z_{j}) e_1 & \text{ if } u^T Z_{j} < b \\
    \end{cases}
    \] 
    
    Therefore we can say that $\tilde{\psi}(Z; b_{Q}) - \tilde{\psi}(Z; b_{Q'})$ satisfies 
    \[ 
    \psi(Z; b_1, b_2)_{i} =   \begin{cases}
        (\max_{j \in N(i)} u^T Z_{j} - \min_{j \in N(i)} u^T Z_{j}) e_1  & \text{ if } b_1\leq  u^T Z_{j} \leq b_2\\
        0 & \text{ else } 
    \end{cases}
    \] 
\end{proof}

The following lemma, which is the heart of the proof, uses the above selective shift operators
to construct contextual mappings. 

\begin{lemma}
    \label{lem:contextual}
    There  exists  a function $g_c: \R^{(n+1) \times (d+1)} \to \R^{(n+1)} $ and 
    a unique vector $u$, such that for all $P \in \Gdpos$ $g_c(P) := \bkt{u}{g(P)}$ 
    satisfies the property that $g_c$ is a  contextual mapping of $P$. 
    Furthermore, $g_c \in \mathcal{T}_D^{2,1,1}$ using a composition of sparse attention layers
    as long as $D$ contains the star graph. 
\end{lemma}
\begin{proof}
    Define $u \in \R^{d+1} = [1,\delta^{-1},\delta^{-2},\dots, \delta^{-d+1},\delta^{-nd}]$ and let 
    $X_{0} = (0,\dots,0,1)$.  We will assume that $\bkt{x_i}{x_0}=0$, by assuming that all the 
    columns $x_1,\dots, x_n$ are appended by $0$. 
    
    To successfully encode the entire context in each token, we will interleave the shift operator
    to target the original columns $1,\dots,n$ and to target the global column $0$. After a column $i$ is targeted, its inner product with $u$ will encode the entire context of the first $i$ columns. 
    Next, we will shift the global token to take this context into account. This can be subsequently used by the  remaining columns.    
    
    For $i \in \{0,1,\dots,n\} $, we will use  $l_i$ to denote the innerproducts $\bkt{u}{x_i}$ at the beginning. For 
    $f_i = \bkt{u}{x_i}$ after the $i^{th}$ column has changed for $i \in \{1,\dots,n\}$ and we will use
    $f_0^k $ to denote $\bkt{u}{x_0}$ after the $k^{th}$ phase. We need to distinguish the global token  
    further as it's inner product will change in each phase. 
    Initially, given  $X \in \Gdpos$, the following are true:
    \begin{align*}
     \delta^{-(i-1)d} &\leq \bkt{u}{X_{i}} \leq \delta^{-id}-\delta \qquad \text{ for all } i \in [n]  \\
     \delta^{-(n+1)d} &= \bkt{u}{X_{0}} 
    \end{align*}
   Note that all $l_i$ ordered in distinct buckets $l_1 < l_2 < \dots < l_n <l_0$. 
    
We do this in phases indexed from $i  \in \{1,\dots, n\}$.  Each phase consists of two distinct parts:
    \newline
    \textbf{ The low shift operation:} These operation will be  of the form
        \[ X \leftarrow X +  \delta^{-d}\psi \left(X,v -\delta/2, v + \delta/2 \right) \]  for values
        $v \in [\delta^{-id}), \delta^{-(i+1)d})_{\delta}$. 
        The range is chosen so that only $l_i$ will be in the range and no other $l_j$ $j\neq i$
        is in the range. 
        This will shift exactly the $i^{th}$ column $x_i$ so that the new inner product
        $f_i = \bkt{u}{x_i} $  is substantially larger than $l_i$. Furthermore, no
        other column of $X$ will be affected. 
    \newline
    \textbf{ The high shift operation: }
    These operation will be  of the form
        \[ X \leftarrow X +  \delta^{-nd} \cdot \psi \left(X,v -\delta/2, v + \delta/2 \right)\] 
        for values $v \in [S_i, T_i)_{\delta}$. The range $[S_{i}, T_i)_{\delta}$ is chosen to 
        only affect the column $x_0$ (corresponding to the global token) and no other column. In particular, this 
        will shift the global token by a further $\delta^{-nd}$. Let $\tl_0^i$ denote the value of 
        $\tl^i_0 = \bkt{u}{x_0}$ at the end of $i^{th}$ high operation. 
    
    Each phase interleaves a shift operation to column $i$ and updates the global token. 
    After each phase, the updated $i^{th}$ column $f_i = \bkt{u}{x_i} $ will contain a unique token
    encoding the values of all the $l_1,\dots,l_i$. After the high update, $\tl_{0}^i = \bkt{u}{x_0}$
    will contain information about the first $i$ tokens. 
    
    Finally, we define the following constants for all $k \in \{0,1,\dots,n\}$. 
     \begin{align}
    T_k &= (\delta^{-(n+1)d} + 1)^{k} \cdot \delta^{-nd} -  \sum_{t=2}^k  (\delta^{-(n+1)d} + 1)^{k-t}( 2\delta^{-nd-d}  +  \delta^{-nd} +1) \delta^{-td} \notag \\
         & \qquad -   (\delta^{-(n+1)d} + 1)^{k-1}(\delta^{-nd-d}  +  \delta^{-nd} )\delta^{-d}  - \delta^{-(k+1)d} \label{eqn:upper-bound} \tag{UP}
    \end{align}
    \begin{align}
    S_k &= (\delta^{-(n+1)d} + 1)^{k} \cdot \delta^{-nd} -  \sum_{t=2}^k  (\delta^{-(n+1)d} + 1)^{k-t}( 2\delta^{-nd-d}  +  \delta^{-nd} +1) \delta^{-(t-1)d} \notag \\
         & \qquad -   (\delta^{-(n+1)d} + 1)^{k-1}(\delta^{-nd-d}  +  \delta^{-nd} ) - \delta^{-kd} \label{eqn:lower-bound} \tag{LP}
    \end{align}
    
    After each $k$ phases, we will maintain the following invariants:
    \begin{enumerate}
    \item  $S_k < \tl^k_0 < T_k$ for all $k \in \{ 0, 1, \dots, n\}$. 
    \item  $T_{k-1} \leq  f_k < S_k $
    \item  The order of the inner products after $k^{th}$ phase is 
     \[l_{k+1} < l_{k+2} \dots < l_n < f_1 < f_2< \dots <  f_{k} < \tl_0^k .\] 
    \end{enumerate}

    \paragraph{Base case}    
    The case $k=0$, is trivial as we simply set  $S_0 = \delta^{-(n+1)d}$, $T_0 = \delta^{-(n+1)\cdot d}+\delta$. 
    
    The first nontrivial case is $k=1$. 

   \paragraph{ Inductive Step }
   First, in the low shift operation is performed in the range $[\delta^{-(k-1)d}, \delta^{-kd})_{\delta}$
   Due to the invariant, we know that there exists only one column $x_k$ that is affected by this shift. 
   In particular, for column $k$, we will have $\max_{j \in N(k)} \bkt{u}{x_j} = \bkt{u}{x_{0}} = \tl^{k-1}_0$. The minimum is $l_k$. Thus the update will be
   $f_k = \delta^{-d} ( \tl_0^{k-1} - l_k ) + l_k $.  Observe that for small enough $\delta$,
   $f_k \geq \tl_0^{k-1}$.  Hence the total ordering, after this operation is
    \begin{align}
        l_k+1 < l_{k+2} \dots < l_n < f_1 < f_2< \dots <  \tl_0^{k-1} < f_{k}  \label{eqn:intermediate}
    \end{align}
    Now when we operate a higher selective shift operator in the range $[S_{k-1},T_{k-1})_{\delta}$.  
    Since only global token's innerproduct  $\tl_0^{k-1}$ is in this range, 
    it will be the only column affected by the shift operator. The global token operates  over the entire range, we know from~\Cref{eqn:intermediate} that,  $f_k = \max_{i \in [n]} \bkt{u}{x_i}$ and $l_{k+1} = \min_{i \in [n]} \bkt{u}{x_i}$. 
    The new value $\tl_0^k = \delta^{-nd} \cdot (f_k - l_{k+1} ) + \tl_0^{k-1}$. 
    Expanding and simplifying we get,  
    \begin{align*}
        \tl_0^k &= \delta^{-nd} \cdot (f_k - l_{k+1} ) + \tl_0^{k-1} \\
        &= \delta^{-nd} \cdot (   \delta^{-d} ( \tl_0^{k-1} - l_k ) + l_k - l_{k+1} ) + \tl_0^{k-1} \\
        &= \delta^{-(n+1)d} \cdot (    \tl_0^{k-1} - l_k )  +  \delta^{-nd}(l_k - l_{k+1}) + \tl_0^{k-1} \\
        &= (\delta^{-(n+1)d} + 1) \tl_0^{k-1} - (\delta^{-nd-d}  +  \delta^{-nd}) l_k - l_{k+1} \\
        \intertext{Expanding the above recursively, we get}
        &= (\delta^{-(n+1)d} + 1)^{k} \cdot \tl_0^{0} -  \sum_{t=2}^k  (\delta^{-(n+1)d} + 1)^{k-t}( 2\delta^{-nd-d}  +  \delta^{-nd} +1) l_t \\
         & \qquad -   (\delta^{-(n+1)d} + 1)^{k-1}(\delta^{-nd-d}  +  \delta^{-nd} )l_1  - l_{k+1}
         \end{align*}
         
    Since we know that $\tl_0^0 = \delta^{-nd}$ and each $l_i < \delta^{-id}$, we can substitute this to get~\Cref{eqn:upper-bound}
    and we can get an lower-bound~\Cref{eqn:lower-bound} by using $l_i \geq \delta^{-(i-1)d} $. 
    
    By construction, we know that $S_k \leq \tl_0^k < T_k$.  For sufficiently small $\delta$, 
    observe that $S_k \leq \tl_0^k < T_k$ all are essentially the dominant term $ \approx O(\delta^{-n(k+1)d - kd})$ and all the lower order terms do not matter.  As a result it is 
    immediate to see that that $f_k > \delta^{-d} (\tl_0^{k-1} - l_k) > T_{k-1}$ and hence we
    can see that the invariant 2 is also satisfied. Since only column $k$ and the global token are affected,
    we can see that invariant 3 is also satisfied. 
    
    After $n$ iterations, $\tl^n_0$ contains a unique encoding for any $P \in \Gdpos$. 
    To ensure that all tokens are distinct, we will add an additional layer 
    $X = X + \delta^{-n^2d} \psi(X,v -\delta/2, v+\delta/2)$ for all $v \in [S_1, T_n)_{\delta}$. 
    This ensures that for all $P, P' \in \Gdpos$, each entry of $q(P) $ and $q(P')$ are distinct. 
\end{proof}

The previous lemma shows that we can compute a contextual mapping using only sparse transforms. 
We now use the following lemma to show that we can use a contextual mapping and feed-forward layers
to accurately map to the desired output of the function $\bar{f}$. 

\begin{lemma}[Lemma 7~\cite{Yun19}]
Let $g_c$ be the function in~\Cref{lem:contextual}, we can construct
a function $g_v: \R^{(n+1) \times (d+1)} \to \R^{(n+1) \times d} $ composed of 
$O(n \delta^{-nd})$ feed-forward layers (with hidden dimension $q=1$) 
with activations in $\Phi$ such that 
$g_v$ is defined as  $g_v(Z) = [g_v^{tkn}(Z_{1}), \dots, g^{tkn}_v(Z_{n})]$, 
where  for all $j \in \{1,\dots, n\}$, 
        \[ g_v^{tkn}(g_c(L)_{j}) =  f(L)_{j}  \]
\end{lemma}

\subsubsection{Approximating modified Transformers by Transformers}

The previous section assumed we used Transformers that used hardmax operator $\sigma_H$ and 
activations functions belonging to the set $\Phi$. This is without loss of generality as
following lemma shows.

\begin{lemma}[Lemma 9 \cite{Yun19}]
For each $g \in \bar{\mathcal{T}}^{2,1,1}$ and $1 \leq p \leq \infty$, $\exists g \in \mathcal{T}^{2,1,4}$ such that
$d_p(g,\bar{g}) \leq \epsilon/3$
\end{lemma}
    
Combining the above lemma with the \Cref{lem:contextual}, we get our main result:
\begin{theorem}
    Let $1\leq p \leq \infty$ and $\epsilon > 0$, there exists a transformer network 
    $g \in \mathcal{T}_D^{2,1,4}$
    which achieves a ratio of $d_{p}(f,g) \leq \epsilon$ where $D$ is the sparse graph. 
\end{theorem}

Since the sparsity graph associated with \bigb contains a star network, we know that it 
can express any continuous function from a compact domain.

\paragraph{Contemporary work on Universal Approximability of Sparse Transformers}
We would like to note that, contemporary work done by \citet{yun2020on}, also parallelly explored the ability of sparse transformers with linear connections to capture sequence-to-sequence functions on the compact domain.

\newpage
\section{Turing Completeness}
\label{sec:apndx-turing}

In this section, we will extend our results to the setting of \citet{Perez19}. Our 
exposition will largely use their proof structure but we will make a few changes. 
We  repeat some of the lemmas with the amendments to make the exposition 
self-contained.

\subsection{Notation}
\paragraph{Transformer Decoder}
We need both an encoder and a decoder in the transformer for simulating a Turing machine.
We utilize the same notation used in~\Cref{sec:apndx-enc-notation} for encoders. 
The decoder is similar to an encoder but with additional attention to an external pair of key-value vectors $(\mK^{\textbf{e}}\in\R^{n\times m},\mV^{\textbf{e}}\in\R^{n\times d})$, which usually come from the encoder stack.
A single layer of Transformer decoder is a parametric function $\Dec$ receiving a sequence $\mY_j=(\vy_1,\ldots, \vy_j)$ of vectors in $\R^d$ plus the external $(\mK^{\textbf{e}}, \mV^{\textbf{e}})$ and returning a sequence of vectors $\mZ_j=(\vz_1,\ldots,\vz_j)$ of the same length. Each $\vz_i$ is a $d$ dimensional vector as well.  $\Dec$ has three components, one more than $\Enc$:
\begin{enumerate}[leftmargin=6mm, itemsep=2mm, partopsep=0pt,parsep=0pt]
    \item An attention mechanism $\Attn$ that takes in the sequence $\mY_j$ and returns sequence $(\vp_1, ..., \vp_j)$ of the same length and dimensionality;  
    \item A cross-attention mechanism $\CrossAttn$ that takes in the sequence $(\vp_1, ..., \vp_j)$ plus the external $(\mK^{\textbf{e}}, \mV^{\textbf{e}})$ and returns sequence $(\va_1, ..., \va_j)$, with each $\va_i\in\R^d$; and
    \item A two layer fully connected network $O$ that takes in a vector in $\R^d$ and returns a vector in $\R^d$. 
\end{enumerate}
\vspace{-2mm}
Then $i$-th output vector of $\Dec(\mY_j; \mK^{\textbf{e}}, \mV^{\textbf{e}})$ is computed as follows:
\begin{flalign}
    && \vz_i &= O(\va_i) + \va_i & \label{eq:dec-ff} \\
    &\text{where} & \va_i &= \CrossAttn(\vp_i, \mK^{\textbf{e}}, \mV^{\textbf{e}}) + \vp_i & \label{eq:dec-ext} \\
    &\text{and} & \vp_i &= \Attn_D(\mY_j)_i + \vy_i \label{eq:dec-self}&
\end{flalign}
$\Attn_D$ and $O$ are as defined in~\Cref{sec:apndx-enc-notation} and it remains to define $\CrossAttn$.
The $i^\textrm{th}$ output vector of multi-head cross-attention attention is given by
\begin{align}
\CrossAttn(\mY_j)_i &=  \sum_{h=1}^H \sigma \left((\vy_i W_Q^h) (\mK^{(e)} W_K^h)^T  \right) \cdot (\mV^{(e)} W_V^h ) 
\end{align}
where $W_Q^h, W_K^h, W_V^h \in \R^{d \times m}$, $W_V^h\in \R^{d \times d}$, for all $h = 1, \ldots H$ heads.

\paragraph{Turning Machine}
We will use the same setup of Turning Machine that was used by \citet{Perez19} (see section B.4). 
Given a Turing Machine $M = (Q,\Sigma, \delta, q_{init},F)$, we use the following notation
\begin{align*}
    q^{(j)} &: \text{ state of Turing machine } M \text{ at time }j.  \\
    s^{(j)} &: \text{ symbol under the head of } M \text{ at time }j.  \\
    v^{(j)} &: \text{ symbol written by } M \text{ at time }j.  \\
    m^{(j)} &: \text{ head direction in the transition of } M \text{ at time }j. 
\end{align*}

\paragraph{Vector representations}
For a symbol $s\in \Sigma$, $\oh{s}$ denotes its one-hot vector representation in $\Q^{|\Sigma|}$.
All the transformer intermediate vectors used in our simulations have dimension $d=2|Q|+4|\Sigma|+16$.
Note that we use five extra dimension as compared to \citet{Perez19}.
We follow the convention used in \citet{Perez19} and write a a vector $\vv\in\Q^d$ arranged in four groups of values
as follows
\[
\begin{array}{rcllr}
\vv & = & [ 
& \vq_1,\vs_1,x_1, \\
&&& \vq_2,\vs_2,x_2,x_3,x_4,x_5,x_6, \\
&&& \vs_3,x_7,\vs_4, \\
&&& x_8,x_9,x_{10},x_{11},x_{12},x_{13},x_{14},x_{15},x_{16} & ] 
\end{array}
\]
where $\vq_i\in \Q^{|Q|}$, $\vs_i\in \Q^{|\Sigma|}$, and $x_i\in\Q$.

\subsection{Details of the Simulation}
\label{sec:details}

In this section, we give more details on the architecture of the encoder and decoder needed to implement our simulation strategy.

\begin{figure}[b]
    \vspace{-5mm}
    \centering
    \includegraphics[width=\linewidth]{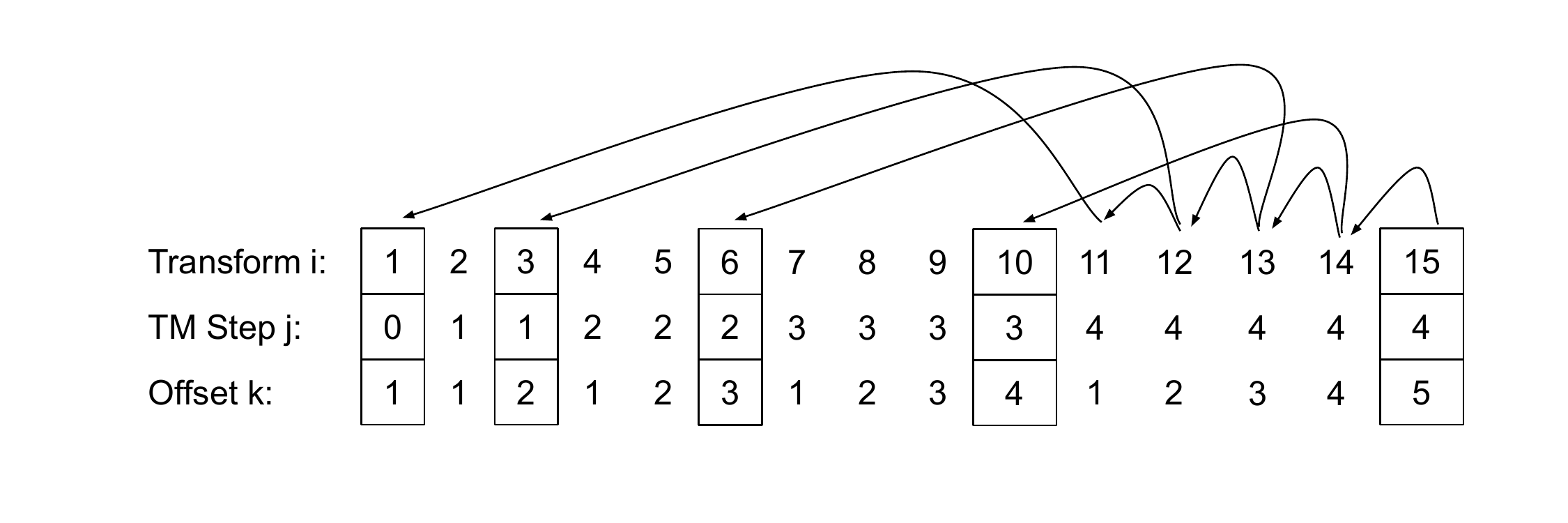}
    \vspace{-10mm}
    \caption{Mapping between transformer step and original Turing machine step.}
    \label{fig:apndx_turing}
\end{figure}

\paragraph{High Level Overview:}
Given the Turing machine $M$, we will show that a transformer with an appropriate encoder and decoder $\mathcal{T}_D$ can simulate each step of $M$'s execution.
Our simulation strategy will mostly follow \citet{Perez19}, except we will use a sparse attention mechanism.
The main idea is to maintain the current Turing machine state $q^{(j)}$ and symbol under the head $s^{(j)}$ as part of the decoder sequence $\mY$ for all time step $j$ so that we can always simulate the corresponding Turing machine transition $\delta(q^{(j)}, s^{(j)}) = (q^{(j)}, v^{(j)}, m^{(j)})$.
The key difference will rise in Lemma~B.4 of \citet{Perez19}, where full attention is used to select the appropriate symbol from tape history in one step.
To accomplish the same task with sparse attention, we will exploit the associative property of max and break down the symbol selection over multiple steps.
Thus, unlike \citet{Perez19} one decoding step of our sparse transformer $\mathcal{T}_D$ does not correspond to one step of the Turing machine $M$.
In particular, we will have two type of steps: compute step corresponding to update of $M$'s state and intermediate steps corresponding to aggregating the max (which in turn is used for symbol selection).
Let $i$ denote the step of $\mathcal{T}_D$ and $g(i)$ denote the step of $M$ being simulated at step $i$ of the decoder.
At each decoding step we want to maintain the current Turing machine state $q^{g(i)}$ and symbol under the $s^{g(i)}$ in $\vy_i$.
For roughly $O(\sqrt{i})$ intermediate steps the state will remain the same, while we aggregate information about relevant past output symbols through sparse attention.
To maintain the same state for intermediate steps, we introduce an extra switching layer (\Cref{sec:appndx-new-layer}).
Finally, at the next compute step we will make the transition to new state $q^{g(i)+1}$, new head movement $m^{g(i)}$, and new output symbol $v^{g(i)}$ to be written. 
Thereby we are able to completely simulate the given Turing machine $M$.
As a result, we can prove the following main theorem:
\begin{theorem}
There exists a sparse attention mechanism using $O(n)$ inner products such that the resulting class of Transformer Networks using this sparse attention mechanism is Turing Complete. 
\end{theorem}

\subsubsection*{Encoder}
As \citep{Perez19}, we use the same trivial single layer encoder where resulting $\mK^{(e)}$ contains position embedding and $\mV^{(e)}$ contains one-hot symbol representation.

\subsubsection*{Decoder}
\paragraph{Sparse Self-Attention mechanism for Decoder}
In this section, we will consider a particular instance of the sparse graph $D$ at decoder.
We define its edges to be given by the following relations:
$\forall j\in\N_{+}, 1\leq k\leq j+1$,  
\begin{align*}
&\left(\frac{j(j+1)}{2} + k, \frac{k(k+1)}{2} \right) \text{ and }\\ \\
&\left(\frac{j(j+1)}{2} + k, \frac{j(j+1)}{2} + k \right)  \text{ if } k>1 \text{ else } \left(\frac{j(j+1)}{2} + 1, \frac{j(j+1)}{2} \right).
\end{align*}

This graph can be seen as a special case of \bigb where first type 
of edges are realizations of random and second type of edges correspond to locality.
Also note that this graph satisfies the left-to-right constraint of 
decoder, i.e.~no node attends to a node in the future.

\paragraph{Embeddings and positional encodings}
Our construction needs a different positional encoding $\penc_{\Dec} :\N\to\Q^{d}$ for decoder: 
\begin{equation*}
\begin{array}{rcllr}
\penc_{\Dec}(i) & = & [ 
& \sz, \\
&&& \sz, \\
&&& \sz, \\
&&& 1,g(i)+1,\frac{1}{g(i)+1},\frac{1}{(g(i)+1)^2},h(i),0,0,0,0 & ] 
\end{array}
\end{equation*}
where $g(i) = \left\lfloor \frac{-1+\sqrt{1+8i}}{2} \right\rfloor$ and $h(i) = g(i+1) - g(i)$. Note that $h(i)$ reduces to a binary indicator variable $
\mathbf{1}\left\lbrace\frac{-1+\sqrt{1+8i}}{2}=\left\lfloor \frac{-1+\sqrt{1+8i}}{2} \right\rfloor\right\rbrace$.

\subsubsection*{Induction Setup}
We next show how to construct the decoder layers to produce the sequence of outputs $\vy_1,\vy_2,\ldots$,
where $\vy_i$ is given by:
\begin{equation*}
\begin{array}{rcllr}
{\vy}_i & = & [ 
& \oh{q^{g(i)}},\oh{s^{g(i)}},c^{g(i)}, \\
&&& \sz, \\
&&& \vzs,0,\oh{w^{(i)}}, \\
&&& 0,0,0,0,0,u_1^{(i)},u_2^{(i)},u_3^{(i)},u_4^{(i)} & ] 
\end{array}
\end{equation*}
That is, at step $i$ of our sparse decoder $\vy_i$, it will contain the information about the state of the turing machine $M$ at time $g(i)$, the symbol under the head of $M$ at time $g(i)$, and the current location of head of $M$ at time $g(i)$.
We also have a placeholder symbol $w$ and placeholder scalars $u_1,u_2, u_3$, whose role will be clear from our construction.

We consider as the starting vector for the decoder the vector
\begin{equation*}
\begin{array}{rcllr}
{\vy}_1 & = & [ 
& \oh{q_{\text{init}}},\oh{\#},0, \\
&&& \sz, \\
&&& \sz, \\
&&& \sz & ] 
\end{array}
\end{equation*}
We assume that the start head is at $c^{(0)}=0$, the initial state is $q^{(0)}=q_{\text{init}}$, and $s^{(0)}=\#$ as we initialize from clean tape.
We show the correctness of our construction by an inductive argument:
we describe the architecture piece by piece and at the same time will show for every $r\geq 0$ 
, our architecture constructs $\vy_{r+1}$ from the previous vectors 
$(\vy_0,\ldots,\vy_r)$.

Thus, assume that $\vy_1,\ldots,\vy_r$ satisfy the properties stated above. 
Since we are using positional encodings, 
the actual input for the first layer of the decoder is the sequence
\[
\vy_1+\penc_{\Dec}(1),\ \vy_2+\penc_{\Dec}(2),\ \ldots,\ \vy_{r}+\penc_{\Dec}(r).
\]
We denote by $\overline{\vy}_i$ the vector $\vy_i$ plus its positional encoding.
Thus we have $\forall \ 1 \leq i \leq r$ that
\begin{equation*}
\begin{array}{rcllr}
\overline{\vy}_i & = & [ 
& \oh{q^{g(i)}},\oh{s^{g(i)}},c^{g(i)}, \\
&&& \sz, \\
&&& \vzs,0,\oh{w^{(i)}}, \\
&&& 1,g(i)+1,\frac{1}{g(i)+1},\frac{1}{(g(i)+1)^2},h(i),u_1^{(i)},u_2^{(i)},u_3^{(i)},u_4^{(i)} &] 
\end{array}
\end{equation*}

\subsubsection{Layer 1: Simulate Transition Function}

In this layer, we use the cross-attention between encoder and decoder to access the 
input string and  a feed-forward network to simulate the transition function of $M$.
The first self attention in~\Cref{eq:dec-self} is not used in this layer and 
we just produce the identity.
This identity function is achieved by setting all queries, keys, values to be 0 
everywhere plus the residual connection.
Thus, we have  ${\vp}^1_i=\overline{\vy}_i$.

Since $\vp^1_i$ is of the form $[\lu,\ldots,\lu,1,g(i)+1,\lu,\ldots,\lu]$, we know
by Lemma B.1 of \citet{Perez19} that if we use $\vp^1_i$ to attend over the encoder we obtain
\begin{equation*}\label{eq:first-layer2}
\begin{array}{rcllr}
\CrossAttn(\vp^1_i,\mK^{\textbf{e}},\mV^{\textbf{e}}) & = & [ 
& \sz, \\
&&& \sz, \\
&&& \oh{\alpha^{g(i)+1}},\beta^{g(i)+1},\vzs, \\
&&& \sz & ] 
\end{array}
\end{equation*}
where $\alpha$ and $\beta$ are as defined in Eq. (21) of \citep{Perez19}.
Thus in~\Cref{eq:dec-ext} we finally produce the vector $\va^1_i$ given by
\begin{equation}\label{eq:ai_1}
\begin{array}{rcllr}
\va^1_{i} & = &&
\CrossAttn(\vp^1_i,\mK^{\textbf{e}},\mV^{\textbf{e}}) + \vp^1_i \\ 
& = & [ 
& \oh{q^{g(i)}},\oh{s^{g(i)}},c^{g(i)}, \\
&&& \sz, \\
&&& \oh{\alpha^{g(i)+1}},\beta^{g(i)+1},\oh{w^{(i)}}, \\
&&& 1,g(i)+1,\frac{1}{g(i)+1},\frac{1}{(g(i)+1)^2},h(i),u_1^{(i)},u_2^{(i)},u_3^{(i)},u_4^{(i)} & ] 
\end{array}
\end{equation}

As the final piece of the first decoder layer 
we use a function $O_1(\cdot)$ (\Cref{eq:dec-ff}) that satisfies the following lemma.
\begin{lemma}[Lemma B.2 \citep{Perez19}]\label{lem:M}
There exists a two-layer feed-forward network $O_1:\Q^d\to\Q^d$ such that with input vector $\va^1_{i}$ (\Cref{eq:ai_1}) produces as output
\begin{equation*}
\begin{array}{rcllr}
O_1(\va^1_i) & = & [ 
& \sz, \\
&&& \oh{q^{g(i)+1}},\oh{v^{g(i)}},m^{g(i)},0,0,0,0 \\
&&& \sz, \\
&&& \sz & ] 
\end{array}
\end{equation*}
\end{lemma}
That is, function $O_1(\cdot)$ simulates transition $\delta(q^{g(i)},s^{g(i)})$ 
to construct $\oh{q^{g(i)+1}}$, $\oh{v^{g(i)}}$, and $m^{g(i)}$
besides some other linear transformations.

Thus, finally the output of the first decoder layer is
\begin{equation*}
\begin{array}{rcllr}
\vz^1_i = O_1(\va^1_i) + \va^1_i & = & [ 
& \oh{q^{g(i)}},\oh{s^{g(i)}},c^{g(i)}, \\
&&& \oh{q^{g(i)+1}},\oh{v^{g(i)}},m^{g(i)},0,0,0,0, \\
&&& \oh{\alpha^{g(i)+1}},\beta^{g(i)+1},\oh{w^{(i)}}, \\
&&& 1,g(i)+1,\frac{1}{g(i)+1},\frac{1}{(g(i)+1)^2},h(i),u_1^{(i)},u_2^{(i)},u_3^{(i)},u_4^{(i)} & ] 
\end{array}
\end{equation*}

\subsubsection{Layer 2: Finding Head Node}
In this layer, we only use the feed-forward network to evaluate the next location of the head.
The self-attention and cross-attention are set to be the identity function, so $\va_i^2=\vp_i^2=\vz_i^1$.
Recall that $c^{g(i)}$ is the cell to which $M$ is pointing to at time $g(i)$, and that it satisfies the following recursion $c^{g(i)+1}=c^{g(i)}+m^{g(i)}$, which can be expanded to see that
that $c^{g(i)+1}=m^{(0)}+m^{(1)}+\cdots + m^{g(i)}$.
Its not difficult to see that a two layer network with non-linearity can compute $c^{g(i)+1}/(g(i)+1)$ and $c^{g(i)}/(g(i)+1)$ from $c^{g(i)}$, $m^{g(i)}$, and $1/(g(i)+1)$ using the relation $c^{g(i)+1}=c^{g(i)}+m^{g(i)}$.
At the  end of layer 2, we obtain 
\begin{equation*}
\begin{array}{rcllr}
\vz^2_i  \ = \ O_2(\va^2_i) + \va^2_i & = & [ 
& \oh{q^{g(i)}},\oh{s^{g(i)}},c^{g(i)}, \\
&&& \oh{q^{g(i)+1}},\oh{v^{g(i)}},c^{g(i)+1},\frac{1}{g(i)+1},\frac{1}{(g(i)+1)^2},\frac{c^{g(i)+1}}{g(i)+1},\frac{c^{g(i)}}{g(i)+1}, \\
&&& \oh{\alpha^{g(i)+1}},\beta^{g(i)+1},\oh{w^{(i)}}, \\
&&& 1,g(i)+1,\frac{1}{g(i)+1},\frac{1}{(g(i)+1)^2},h(i),u_1^{(i)},u_2^{(i)},u_3^{(i)},u_4^{(i)} & ] 
\end{array}
\end{equation*}

\subsubsection{Layer 3: Distinguishing Node Type}
\label{sec:appndx-new-layer}
This is an additional layer (not present in the work of \cite{Perez19}), where we propagate computations 
in our sparse graph. 
In particular, we will use this layer to ``compute'' or accumulate state in intermediate nodes. 
We make this clear below.
The self-attention and cross-attention are all set to be the identity function, so $\va_i^3=\vp_i^3=\vz_i^2$.
In this layer, we only use the dense attention layers to select the newly computed states or to continue 
with previous states.
Using idea similar to Lemma B.6 of \citep{Perez19}, we can construct a dense network such that
\[
O([\vx,\vy,\vz,b])) = 
\begin{cases} 
[\vzero,\vzero,\vzero,0] & \text{if }b=1, \\
[\vzero,\vz-\vy,-\vz,0] & \text{if }b=0.
\end{cases} 
\]
The negatives are generated to offset results from skip connection. 
We utilize such network to switch Turing machine state and position embedding for intermediate steps to the values received from previous time step and do nothing for compute nodes. 
We use $h(i)$ as the flipping bit $b$.
Thus, at end of layer 3, we obtain
\begin{equation*}
\begin{array}{rcllr}
\vz^3_i \ = \ O_3(\va^3_i) + \va^3_i & = & [ 
& \sz, \\
&&& \oh{\hat{q}^{(i)}},\oh{\hat{v}^{(i)}},\hat{c}^{(i)},\frac{1}{g(i)+1},\frac{1}{(g(i)+1)^2},\frac{c^{g(i)+1}}{g(i)+1},\hat{u}_4^{(i)}, \\
&&& \oh{\hat{\alpha}^{(i)}},\hat{\beta}^{(i)},\vzs, \\
&&& 1,\hat{u}_1^{(i)},\hat{u}_2^{(i)},\hat{u}_3^{(i)},h(i),0,0,0,0 & ] 
\end{array}
\end{equation*}
where we used $h(i)$ for selecting old states. In particular,
\vspace{-3mm}
\begin{itemize}[leftmargin=6mm, itemsep=2mm, partopsep=0mm,parsep=0mm]
    \item We copy the input state and head position as is for intermediate nodes. We do not need to transition to next Turing machine states in these nodes.
\end{itemize}

\begin{table}[h]
    \small
    \vspace{-3mm}
    \hspace*{12mm}
    \begin{tabular}{lclcl}
    $
    \hat{q}^{(i)} = 
    \begin{cases} 
    q^{g(i)+1} & \text{if }h(i)=1 \\
    q^{g(i)} & \text{if }h(i)=0
    \end{cases} 
    $ ,
    &
    $
    \hat{v}^{(i)} = 
    \begin{cases} 
    v^{g(i)} & \text{if }h(i)=1 \\
    w^{(i)} & \text{if }h(i)=0
    \end{cases} 
    $ ,
    &
    $
    \hat{c}^{(i)} = 
    \begin{cases} 
    c^{g(i)+1} & \text{if }h(i)=1 \\
    c^{g(i)} & \text{if }h(i)=0
    \end{cases} 
    $ .
    \end{tabular}
    \vspace{-3mm}
\end{table}

\begin{itemize}[leftmargin=6mm, itemsep=2mm, partopsep=0mm,parsep=0mm]
    \item To preserve the symbol under the head for intermediate nodes, we copy the previous symbol to $\alpha$ location and set $\beta=g(i)+1$, as the symbol at $\alpha$ location will be copied as the symbol under head for next transformer step by the final transformation layer if $\beta=g(i)+1$. Thus, we correctly preserve the previous symbol under head as Turing machine does not transition these nodes. For compute nodes, things happen as usual.
\end{itemize}
\begin{table}[h]
    \small
    \vspace{-3mm}
    \hspace*{12mm}
    \begin{tabular}{lclcl}
    $
    \hat{\alpha}^{(i)} = 
    \begin{cases} 
    \alpha^{g(i)+1} & \text{if }h(i)=1 \\
    s^{g(i)} & \text{if }h(i)=0
    \end{cases} 
    $ ,
    &&
    $
    \hat{\beta}^{(i)} = 
    \begin{cases} 
    \beta^{g(i)+1} & \text{if }h(i)=1 \\
    g(i)+1 & \text{if }h(i)=0
    \end{cases} 
    $ .
    \end{tabular}
    \vspace{-3mm}
\end{table}
\begin{itemize}[leftmargin=6mm, itemsep=3mm, partopsep=0mm,parsep=0mm]
    \item Finally for the intermediate nodes, we copy the position embedding corresponding to current best symbol $w$, which is stored in $u_1,u_2,u_3$. For compute node, we let the position embedding correspond to current Turing machine step.
\end{itemize}
\vspace{-3mm}
\begin{table}[ht]
    \small
    \hspace*{12mm}
    \begin{tabular}{lclcl}
    $
    \hat{u}_1^{(i)} = 
    \begin{cases} 
    g(i)+1 & \text{if }h(i)=1 \\
    u_1^{(i)} & \text{if }h(i)=0
    \end{cases} 
    $ ,
    &&
    $
    \hat{u}_2^{(i)} = 
    \begin{cases} 
    \frac{1}{(g(i)+1)} & \text{if }h(i)=1 \\
    u_2^{(i)} & \text{if }h(i)=0
    \end{cases}  
    $ ,
    \\\\
    $
    \hat{u}_3^{(i)} = 
    \begin{cases} 
    \frac{1}{(g(i)+1)^2} & \text{if }h(i)=1 \\
    u_3^{(i)} & \text{if }h(i)=0
    \end{cases} 
    $ ,
    &&
    $
    \hat{u}_4^{(i)} = 
    \begin{cases} 
    \frac{c^{g(i)}}{g(i)+1} & \text{if }h(i)=1 \\
    u_4^{(i)} & \text{if }h(i)=0
    \end{cases} 
    $ .
    \end{tabular}
\end{table}

For further simplification note that $g(i+1) = g(i) $ if $h(i) = 0$ else $ g(i) + 1$ when $h(i) = 1$. With this fact, we can conclude that $\hat{q}^{(i)}=q^{g(i+1)}$ and $\hat{c}^{(i)}=c^{g(i+1)}$. Thus, we can write,
\begin{equation*}
\begin{array}{rcllr}
\vz^3_i \ & = & [ 
& \sz, \\
&&& \oh{{q}^{g(i+1)}},\oh{\hat{v}^{(i)}},{c}^{g(i+1)},\frac{1}{g(i)+1},\frac{1}{(g(i)+1)^2},\frac{c^{g(i)+1}}{g(i)+1},\hat{u}_4^{(i)}, \\
&&& \oh{\hat{\alpha}^{(i)}},\hat{\beta}^{(i)},\vzs, \\
&&& 1,\hat{u}_1^{(i)},\hat{u}_2^{(i)},\hat{u}_3^{(i)},h(i),0,0,0,0 & ] 
\end{array}
\end{equation*}

\subsubsection{Layer 4: Finding next symbol on tape}
To find the symbol on tape under next head position $c^{g(i)+1}$, we try to find what was written last at the location $c^{g(i)+1}$.
To facilitate this, following \citep{Perez19}, we define $\ell(j)$ to be the last time (previous to $j$) in which $M$ was pointing to position $c^{(j)}$,
or it is $j-1$ if this is the first time that $M$ is pointing to $c^{(j)}$.
Recall $j$ is the Turing machine step counter, which is different from sparse transformer step $i$. \citep{Perez19} could utilize full attention mechanism to find $v^{\ell(j+1)}$ at one go, but we have to do it over multiple steps owing to our sparse attention mechanism.

We use similar query, key, value functions as used for full attention by \citep{Perez19} $\forall i $:
\begin{equation*}
\begin{array}{rcllr}
Q_4(\vz^3_i)  & = & [ 
& \sz \\
&&& \sz, \\
&&& \sz, \\
&&& 0, \frac{c^{g(i)+1}}{g(i)+1}, \frac{1}{g(i)+1}, \frac{1}{3(g(i)+1)^2}, 0,0,0,0,0 & ] \\
\end{array}
\end{equation*}
\begin{equation*}
\begin{array}{rcllr}
K_4(\vz^3_i)  & = & [ 
& \sz \\
&&& \sz, \\
&&& \sz, \\
&&& 0, \hat{u}_2^{(i)}, \hat{u}_4^{(i)}, \hat{u}_3^{(i)},0,0,0,0,0 & ] \\
V_4(\vz^3_i)  & = & [ 
& \sz, \\
&&& \sz, \\
&&& \vzs,0,\oh{\hat{v}^{(i)}}, \\
&&& 0,0,0,0,0, \hat{u}_1^{(i)}, \hat{u}_2^{(i)}, \hat{u}_3^{(i)} , \hat{u}_4^{(i)} & ] 
\end{array}
\end{equation*}
It is clear that the three functions are linear transformations and thus they can be defined by feed-forward networks. Notice that the query vector is always formed using current time step position embedding, whereas key and value vectors are formed using copied over entries for intermediate nodes and using current entries only for compute node.

\citet{Perez19} find the desired $v^{l(j+1)}$ as $v^{m(j)}$ using full attention, where
\begin{equation*}
    m(t) = \argmin_{m\in\{0,...,t\}} \chi_t^j = \argmin_{m\in\{0,...,t\}} | \langle Q_4(\vz_j^3), K_4(\vz_m^3) \rangle |
\end{equation*}
Note the minimization is only over Turing machine steps, i.e. over compute nodes in our case.
We show below that we can estimates $m(j)$ by parts using sparse attention mechanism. The main idea is just to notice that minimization problem $\min_{m\in\{0,...,t\}} \chi_t^j$ can be expressed as $\min\{ \cdots \min\{ \min\{\chi^j_0, \chi^j_1\}, \chi^j_2 \}, ..., \chi^j_t \} $ by the associativity property.

By definition of our graph $D$, at every intermediate node $i$ of the form $j(j+1)/2+k$, i.e. where $k>0$, $g(i)=j$ and $h(i)=0$, we will attend over node $k(k+1)/2$ and best till now copied from $i-1$.
The node $k(k+1)/2$ is never an intermediate node as $h(k(k+1)/2)=1$ for all $k$ and in fact corresponds to Turing machine's step $k$.
This will help us select the key and value corresponding to min between node $k(k+1)/2$ and $i-1$.
In other words, at node $i$ of the form $j(j+1)/2+k$ we would have evaluated $m(k)$ and corresponding value selected:
\begin{equation*}
    w^{(j(j+1)/2+k+1)} = \hat{v}^{m(k-1)}
\end{equation*}
and similarly for $u$'s.
So after going through all the intermediate nodes, finally at the next compute node, i.e. when $k=j+1$, we will obtain the minimum value over all of $0,1,...,j$.
This implies at a compute node will be able to recover $\ell(g(i)+1)$ and its corresponding value as shown in Lemma B.4 of \citep{Perez19}.
Then we have that $\vp^4_i$ is given by
\begin{equation}
\begin{array}{rcllr}
\vp^4_i & = && \Attn_D (\mZ_i^3) + \vz^3_i \\
& = & [ 
& \sz, \\
&&& \oh{q^{g(i+1)}},\oh{\hat{v}^{(i)}},c^{g(i+1)},0,\frac{c^{g(i)+1}}{g(i)+1},\hat{u}_4^{(i)}, \\
&&& \oh{\hat{\alpha}^{(i)}},\hat{\beta}^{(i)},\oh{w^{(i+1)}}, \\
&&& 1,\hat{u}_1^{(i)},\hat{u}_2^{(i)},\hat{u}_3^{(i)},h(i),{u}_1^{(i+1)},{u}_2^{(i+1)},{u}_3^{(i+1)},{u}_4^{(i+1)} & ] 
\end{array}
\end{equation}
The cross-attention and feed-forward network are set to be identity, so $\vz_i^4=\va_i^4=\vp_i^4$.

\subsubsection{Final transformation}
We finish our construction by using the final transformation function $F(\cdot)$ from the corresponding lemma from~\citet{Perez19}, with a slight modification.
\begin{lemma}[Lemma B.5 \citep{Perez19}]
\label{lem:yr+1}
There exists a function $F:\Q^d\to \Q^d$ defined by a feed-forward network such that
\begin{equation*}
\begin{array}{rcllr}
F(\vz^4_r) & = & [ 
& \oh{q^{g(r+1)}},\oh{s^{g(r+1))}},c^{g(r+1)}, \\
&&& \sz, \\
&&& \vzs,0,\oh{w^{(r+1)}}, \\
&&& 0,0,0,0,0,{u}_1^{(r+1)}, {u}_2^{(r+1)}, {u}_3^{(r+1)}, {u}_4^{(r+1)} ] \\ 
& = & & \vy_{r+1}
\end{array}
\end{equation*}
\end{lemma}
The modification is to let $w, u_1, u_2, u_3$ to pass through.
This yields the desired input to transformer at next time step for both intermediate and compute node, thereby concluding our induction.

\newpage
\section{Limitations} 
\label{sec:apndx-limit}

Finally, we show that sparse attention mechanisms can not universally replace dense attention mechanisms, i.e.~there is no free lunch.
We demonstrate a natural task which can be solved by the full attention mechanism in $O(1)$-layers.
However, under standard complexity theoretic assumptions, we show that this problem will require 
$\tilde{\Omega}(n)$-layers for any sparse attention layers with $\tilde{O}(n)$ edges (not just \bigb).
(We use the standard notation $\tilde{\Omega}(n)$  to hide the dependence on poly-logarithmic factors. )

We consider the simple problem of finding the furthest vector for each vector 
in the given sequence of length $n$ and dimension $d \in \Omega(\log^2n)$. The assumption on 
the dimension is mild , as in many situations the dimension $d =768$ is actually comparable to 
the number of $n$. 

\begin{task}
Given $n$ unit vectors $\{u_1,\dots,u_n\}$, each in $\R^d$ where $d=\Theta(\log^2n)$, 
compute $f(u_1,\dots,u_n) \to (u_{1^*}, \dots, u_{n^*})$  where for a fixed $j \in [n]$, we define 
$ j^* = \argmax_{k} \|u_k - u_j\|_2^2$.
\end{task}

Finding vectors that are furthest apart boils down to minimizing inner product search 
in case of unit vectors. For a full-attention mechanism with appropriate query and keys, 
this task is very easy as we can evaluate all pair-wise inner products. 

The impossibility for sparse-attention follows from hardness results stemming from  
Orthogonal Vector Conjecture (OVC) \citep{abboud2015tight,abboud2014consequences,williams2005new,backurs2015edit}, which is a widely used assumption in fine-grained complexity. Informally, it states that 
one cannot determine if the minimum inner product among $n$ Boolean vectors is $0$ in 
subquadratic time.  
\begin{conjecture}[Orthogonal Vectors Conjecture] \label{conj:ovc}
For every $\epsilon>0$, there is a $c \geq 1$ such that given $n$ Boolean vectors in 
$d$ dimension, cannot determine if there is a pair of orthogonal vectors in $O(n^{2-\epsilon})$
time on instances with $d  \geq  c \log n$. 
\end{conjecture}

Using~\cref{conj:ovc}, we show a reduction to show that a 
transformer $g \in \mathcal{T}_D^{H=O(d),m=O(d),q=O(d)}$ for any sparse directed graph $D$ 
which completes Task $1$  must require a superlinear number of layers. 

\begin{proposition}
    There exists a single layer full-attention network $g\in\mathcal{T}^{H=1,m=2d,q=0}$ that can 
    evaluate Task 1,  i.e. $g(u_1,...,u_n) = [u_{1^*},\dots, u_{n^*}]$, but for any sparse-attention network in $\mathcal{T}_D^{H=O(d),m=O(d),q=O(d)}$ with
    graph $D$ having $\tilde{O}(n)$ edges (i.e.~inner product evaluations), would require $\tilde{\Omega}(n^{1-o(1)})$ layers.
\end{proposition}
\begin{proof}
We will break this proof into two parts:
\paragraph{Part 1: The full attention mechanism can solve the problem in $O(1)$ layer}
We begin by providing an explicit construction of a single layer full self-attention that can evaluate Task 1.

\textbf{Step 1} 
We embed each $u_i$ in the input into $\R^{2d}$ as follows:
\begin{equation}
    x_i := E(u_i) = [u_i; 0]
\end{equation}

\textbf{Step 2}
Construct query, key, value functions as follows:
\begin{equation}
\begin{aligned}
    Q([a; b]) &= -a \\
    K([a; b]) &= a \\
    V([a; b]) &= [0; a] \\
\end{aligned}
\end{equation}

Then $\mathrm{Attn}(Q(x_i), K(X), V(X) = [0; u_{\argmax_{j} \langle -u_i, u_j \rangle}] $. Then,
\begin{equation}
    a_i = \mathrm{Attn}(Q(x_i), K(X), V(X)) + x_i = [u_i; u_{\argmax_{j} \langle -u_i, u_j \rangle}] = [u_i; u_{i^*}]
\end{equation}

\textbf{Step 3} 
Let $O(a_i) = 0$, then the output $z_i = [u_i; u_{i^*}]$ as desired.

To complete the argument, observe that it now only takes $O(n)$ inner products to check
if there is a pair of orthogonal vectors as we need only compare $\bkt{u_i}{u_{i^*}}$. 

\paragraph{Part 2: Every Sparse Attention Mechanism will need $\tilde{\Omega}(n^{1-\epsilon})$ layers}
    We prove by contradiction that it is impossible to solve Task 1 by any 
    $g\in\mathcal{T}_D^{H=O(d),m=O(d),q=O(d)}$ sparse-attention graph $D$ with $\tilde{O}(n)$ edges.

    Suppose we can solve Task 1 using a network $g\in\mathcal{T}_D^{H=O(d),m=O(d),q=O(d)}$ that has 
    $l$ layers. Recall that all the computation we do in one layer is:
    \begin{equation}
    \begin{aligned}
        a_i &= \Attn_D(Q(x_i), K(X_{N(i)}), V(X_{N(i)}) + x_i \\
        x_i &= O(a_i) + a_i
    \end{aligned}
    \end{equation}
    where $\mathrm{Attn}_D$ is defined in \cref{AT_app}.

    Thus, total computation per layer is $\tilde{O}(nd^3)$ and consequently $\tilde{O}(nld^3)$ for the 
    whole network consisting of $l$ layers.
    
    We can use the result of Task 1 to solve the orthogonal vector (OV) problem (defined 
    in~\Cref{conj:ovc}) in linear time. So in total, we will be able to solve any instance of OV in $\tilde{O}(nld^3)$ time.
    
    Now if $l=O(n^{1-\epsilon})$ for any $\epsilon > 0$ and $d=\Theta(\log^2 n)$, then it 
    appears that we are able to solve OV in $\tilde{O}(n^{2-\epsilon})$ which contradicts~\Cref{conj:ovc}.
    Therefore, we need at least $\tilde{\Omega}(n^{1-o(1)})$ layers.
\end{proof}


\newpage
\section{Implementation details}
\label{sec:apndx-impl}

We optimize the code for modern hardware.  Hardware accelerators like GPUs and TPUs truly shine on 
coalesced memory operations which load blocks of contiguous bytes at once.  Thus, its not very efficient 
to have  small sporadic  look-ups caused by a sliding window or random element queries. We alleviate this by 
``blockifying'' the  lookups.

\begin{figure}[b]
    \centering
    \begin{subfigure}{.24\textwidth}
    \includegraphics[width=\linewidth]{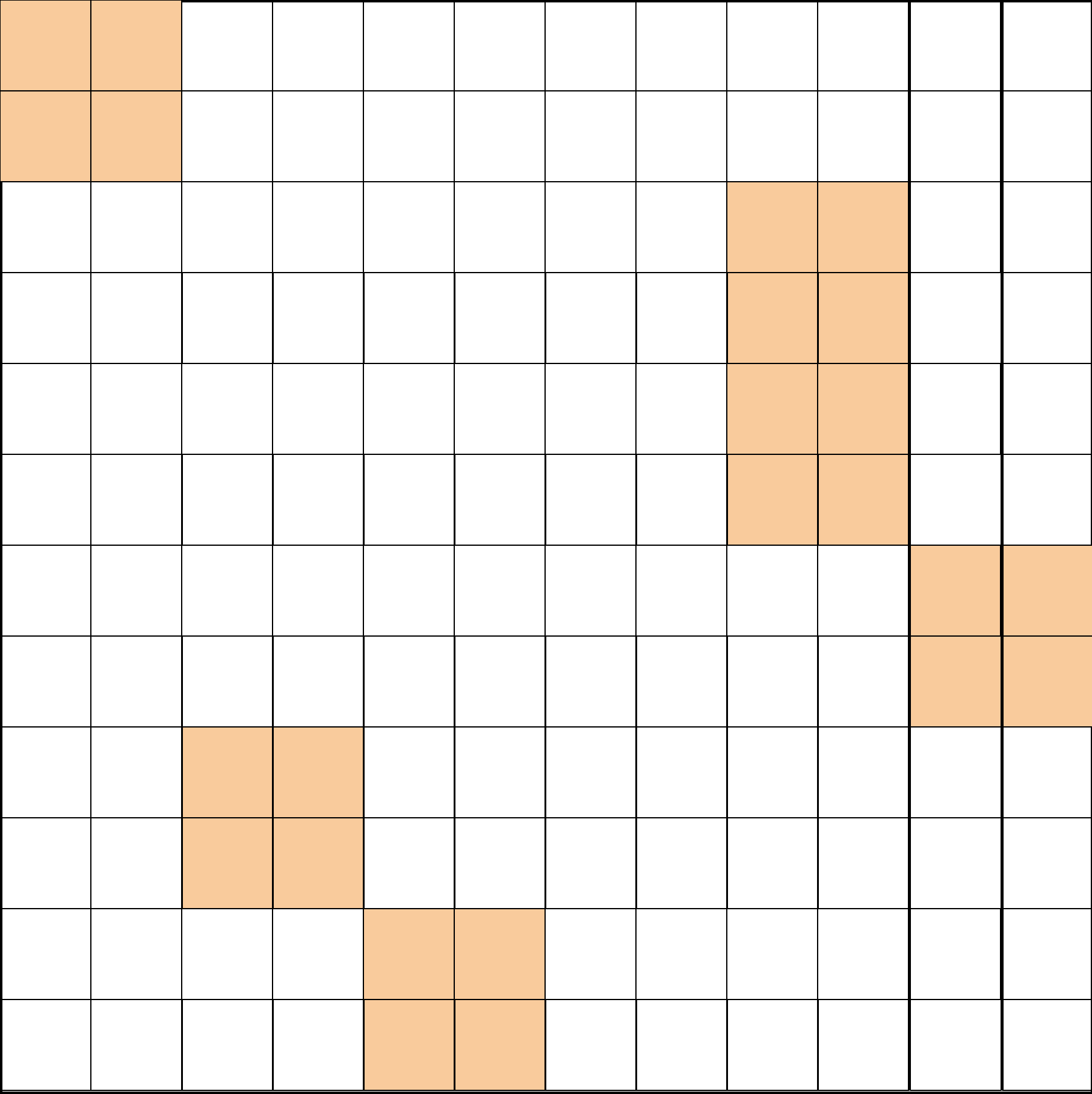}
    \caption{Random Attention \label{fig:apndx_rnd_atn}}
    \end{subfigure}\hfill
    \begin{subfigure}{.24\textwidth}
    \includegraphics[width=\linewidth]{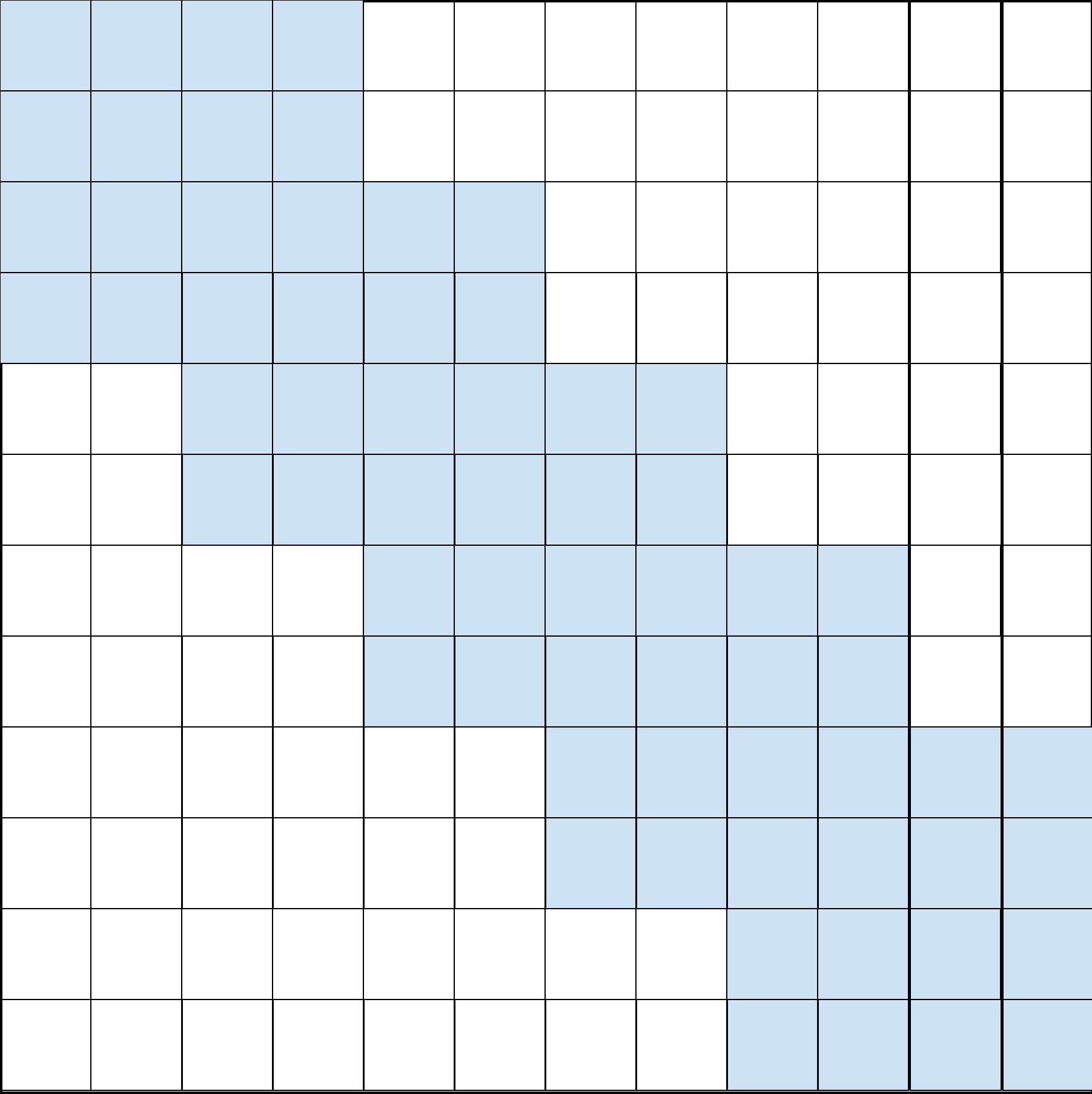}
    \caption{Window Attention \label{fig:apndx_wnd:atn}}
    \end{subfigure}\hfill
    \begin{subfigure}{.24\textwidth}
    \includegraphics[width=\linewidth]{figures/GlobalAttention.pdf}
    \caption{Global Attention \label{fig:apndx_gbl_atn}}
    \end{subfigure}\hfill
    \begin{subfigure}{.24\textwidth}
    \includegraphics[width=\linewidth]{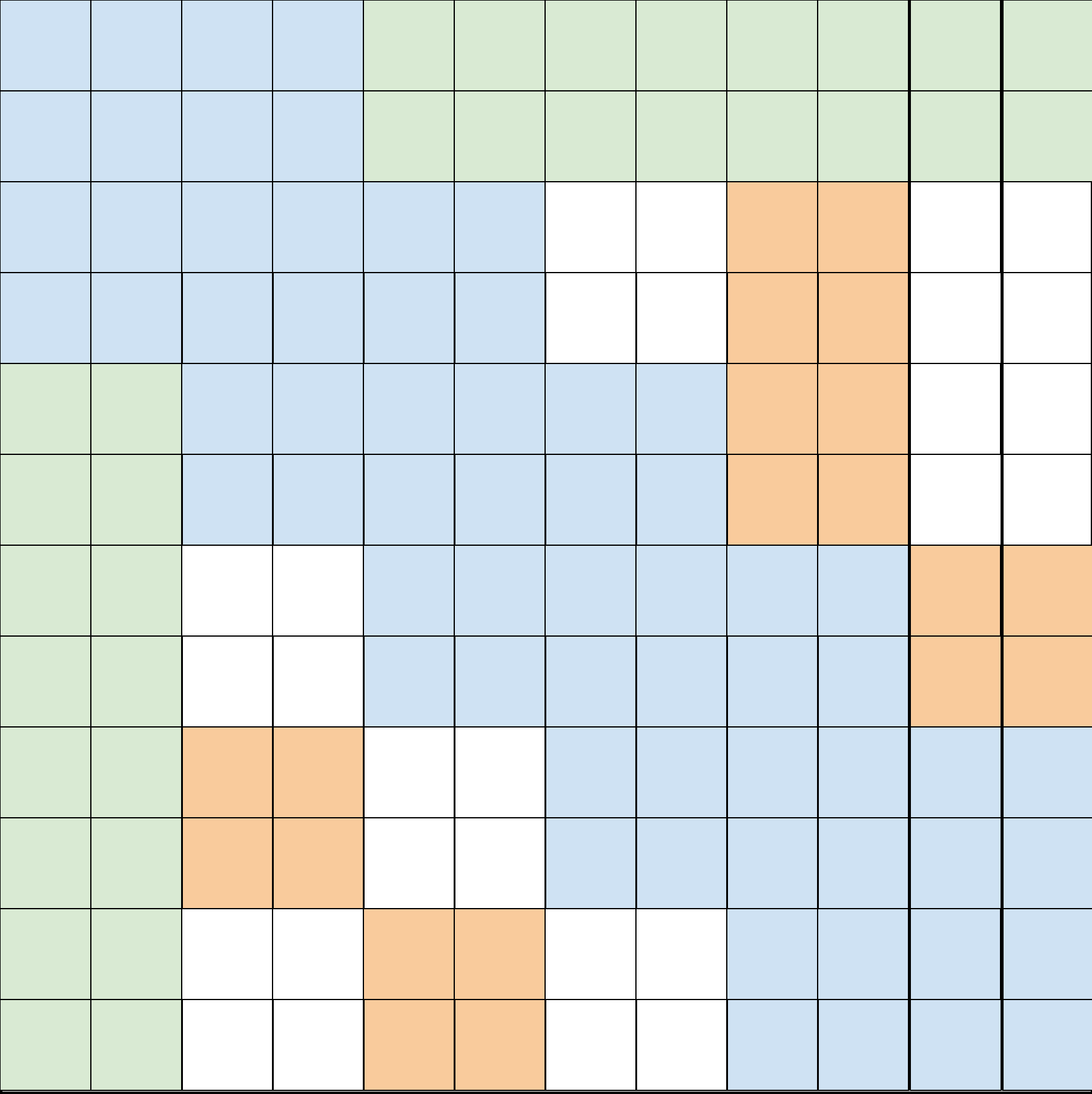}
    \caption{\bigb \label{fig:apndx_bigb_atn}}
    \end{subfigure}
    \hfill
    \caption{Building blocks of the \emph{block-attention} mechanism used in \bigb with block size = $2$. This implies the attention matrix is split into blocks of size $2 \times 2$. All the previous \bigb parameters work on each block as a unit. White color indicates absence of attention. (a) random attention with $r=1$, (b) sliding window attention with $w=3$ (c) global attention with $g=1$. (d) the combined \bigb model.}
    \label{fig:apndx_my_label}
\end{figure}

\paragraph{GPU/TPU and Sparsity}
Ideally, if the adjacency matrix $A$ described in~\Cref{sec:arch} is sparse, one would hope this would be sufficient to speed up the implementation.  
Unfortunately, it is well known~\citep{gray2017gpu,yao2019balanced}, that such sparse multiplications cannot be efficiently implemented in GPUs. GPUs have thousands of cores performing operations in parallel. Thus, we cannot efficiently perform the sparse matrix multiplication mentioned in section~\Cref{sec:arch}.

As a result we propose to first blockify the attention pattern i.e. we pack sets of query and keys together and then define attention on these blocks. 
It is easier to explain this process using the example 
shown in \Cref{fig:apndx_my_label}. Suppose, there are $12$ query and $12$ key vectors to attend to. Using a block size of $2$, we split the query matrix into $12/2 = 6$ blocks and similarly the key matrix into $12/2 = 6$ blocks. Then the three different building components of \bigb are defined on the block matrix. In particular the three different components are:
\begin{enumerate}
    \item Random attention: Each query block attends to $r$ random key blocks. In~\Cref{fig:apndx_rnd_atn}, $r=1$ with block size $2$. This implies that each query block of size $2$ randomly attends to a key block of size $2$. 
    \item Window local attention: While creating the block, we ensure that the number of query blocks and the number of key blocks are the same. This helps us in defining the block window attention. Every query block with index $j$ attends to key block with index $j - (w-1)/2$ to $j + (w-1)/2$, including key block $j$. In~\Cref{fig:apndx_wnd:atn}, $w=3$ with block size $2$. It means that each query block $j$ (size $2$ queries) attends to key block $j-1, j, j+1$.
    \item Global attention: Global attention remains the same as defined in~\Cref{sec:arch}, but we compute it in terms of blocks. In \Cref{fig:apndx_gbl_atn}, $g=1$ with block size $2$. For \bigb-\textsc{itc} this implies that one query and key block, attend to everyone. 
\end{enumerate}
The resulting overall attention matrix is shown in~\Cref{fig:apndx_bigb_atn}.
Unfortunately, simply trying to compute this attention score as multiplying arbitrary pairs of query and key vectors would require use of gather operation, which is inefficient. 
Upon closer examination of window and global attention, we observe that we can compute these attention scores without using a gather operation. 
\begin{figure}
    \vspace{-15mm}
    \centering
    \begin{subfigure}[b]{0.84\textwidth}
         \centering
         \includegraphics[trim=10mm 15mm 10mm 50mm, clip,width=\textwidth,page=1]{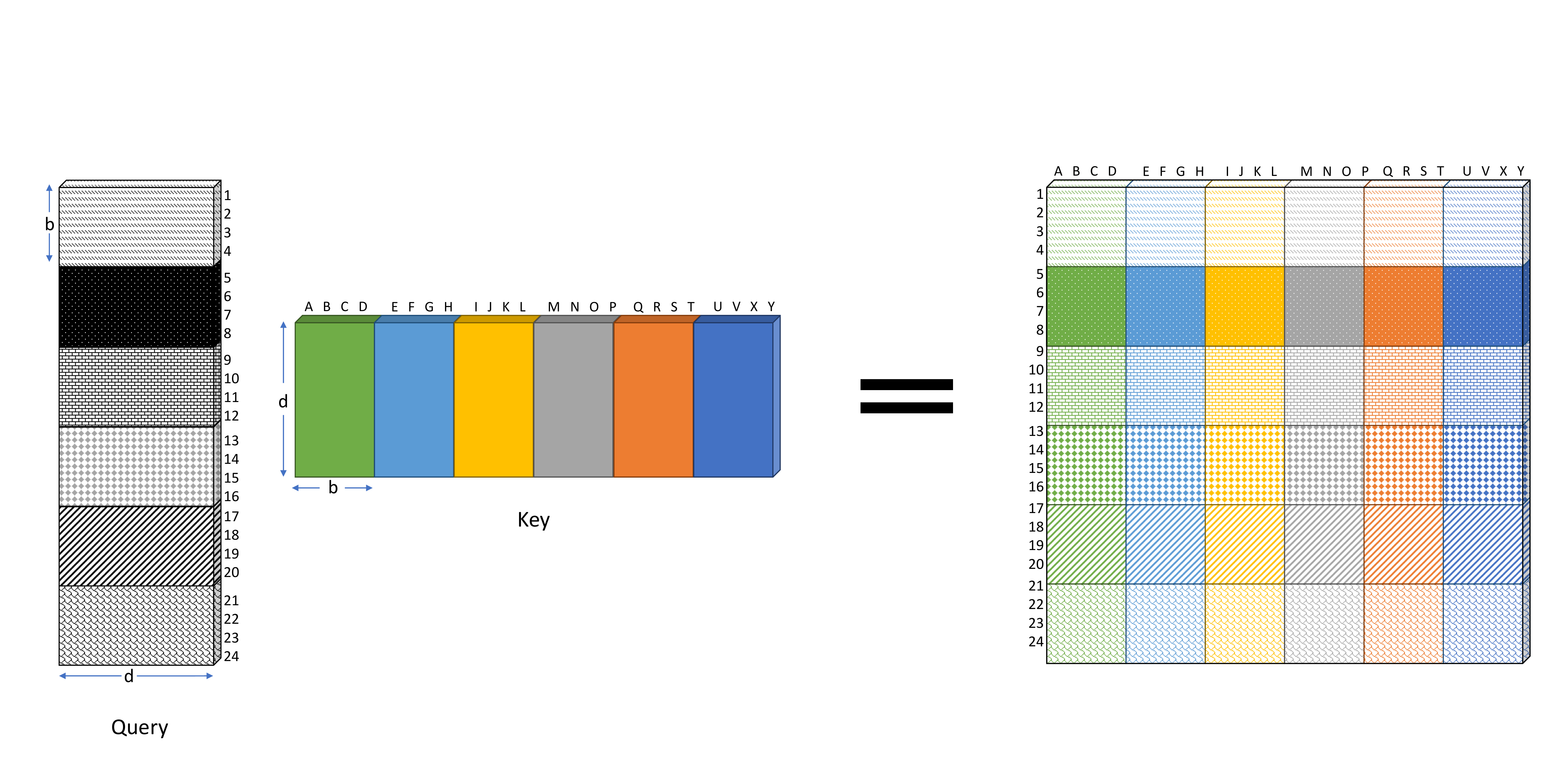}
         \caption{Full all pair attention can be obtained by direct matrix multiplication between the query and key matrix. Groupings just shown for guidance.}
         \label{fig:apndx_full_attn_score}
     \end{subfigure}
     \hfill
     \begin{subfigure}[b]{0.84\textwidth}
         \centering
         \includegraphics[trim=10mm 36mm 10mm 50mm, clip,width=\textwidth,page=2]{figures/BigBird-Figures.pdf}
         \caption{Block diagonal attention can be computed by ``blockifying'' the query and key matrix}
         \label{fig:apndx_block_diag_attn_score}
     \end{subfigure}
     \hfill
     \begin{subfigure}[b]{0.84\textwidth}
         \centering
         \includegraphics[trim=10mm 10mm 10mm 10mm, clip,width=\textwidth,page=3]{figures/BigBird-Figures.pdf}
         \caption{Window local attention obtained by ``blockifying'' the query/key matrix, copying key matrix, and rolling the resulting key tensor (Obtaining rolled key-block tensor is illustrated in detail in~\Cref{fig:apndx_copy-roll}). This ensures that every query attends to at least one block and at most two blocks of keys of size $b$ on each side.}
         \label{fig:apndx_wind_diag_attn_score}
     \end{subfigure}
     \begin{subfigure}[b]{0.84\textwidth}
         \centering
         \includegraphics[trim=10mm 10mm 10mm 10mm, clip,width=\textwidth,page=4]{figures/BigBird-Figures.pdf}
         \caption{Window + Random attention obtained by following the procedure above along with gathering some random key blocks.}
         \label{fig:apndx_rand_att}
     \end{subfigure}
    \caption{Idea behind fast sparse attention computation in \bigb.}
    \label{fig:apndx_bigb_calc_idea}
    \vspace{-2mm}
\end{figure}

Recall, full dense attention scores can be calculated by simple matrix product of query and key matrix with a cost of $O(n^2d)$, as illustrated in~\Cref{fig:apndx_full_attn_score}.
Now note that if we blockify the query and key matrix and multiply, then with only $O(nbd)$ cost we will obtain the block diagonal portion of the attention score, as depicted in~\Cref{fig:apndx_block_diag_attn_score}.
To elaborate this lets assume that $Q, K \in \R^{n \times d}$ are the query and key matrix corresponding to $n$ tokens such that $Q_{i.} = x_i W_Q$ and $K_{i.} = x_i W_K$.
We reshape $n \times d$ query matrix, $Q$, and key matrix, $K$, along the sequence length to obtain $\lceil n/b \rceil \times b \times d$ tensors $Q'$ and $K'$ respectively.
Now we multiply the two tensors as
\begin{equation}
    A_{jst} = \sum_{u} Q'_{jsu} K'_{jtu}, \qquad j=0,1,...,\lceil n/b \rceil 
\end{equation}
The resulting $A$ tensor of size $\lceil n/b \rfloor \times b \times b $ can be reshaped to correspond to the block diagonal portion of the full attention pattern.
Now to extend the attention from block diagonal to a window, i.e. where query block with index $j$ attends to key block with index $j - (w-1)/2$ to $j + (w-1)/2$, we make $w$ copies of the reshaped key tensor $K'$.
We ``roll'' each copy of key-block tensor incrementally along the first axis of length $\lceil n/b \rceil$ as illustrated in~\Cref{fig:apndx_copy-roll}.
Multiplying these $w$ rolled key-block tensors with the query-block tensor would yield the desired window attention scores (\Cref{fig:apndx_wind_diag_attn_score}).
Likewise the global component, we can always include the first $g$ blocks from key tensor corresponding to the global tokens.
Finally, for the random attention, which is very small ($r=3$ for all of our experiments), we resort to using gather ops (\Cref{fig:apndx_rand_att}). 
Also note by design, each query block attends to exactly $r$ random blocks.

Thus, the result of all the three components is basically a compact dense tensor $K''$ of size $\lceil n/b \rceil \times (g+w+r)b \times d$ as shown in~\Cref{fig:apndx_blk_cmp}.
Computing the final attention score then just boils down to a dense tensor multiplication, at which TPU/GPU are very efficient.
Specifically, we need to multiply $Q'$ (size: $\lceil n/b \rceil \times b \times d$) and $K''$ (size: $\lceil n/b \rceil \times (g+w+r)b \times d$) with a cost of $O(n(g+w+r)bd)$ to yield the desired attention score tensor of size $\lceil n/b \rceil \times b \times (g+w+r)b $, which can be reshaped to obtain all the attention scores according to the BigBird pattern.

\begin{figure}
    \vspace{-7mm}
    \centering
    \includegraphics[trim=10mm 15mm 10mm 10mm, clip,width=0.7\textwidth,page=5]{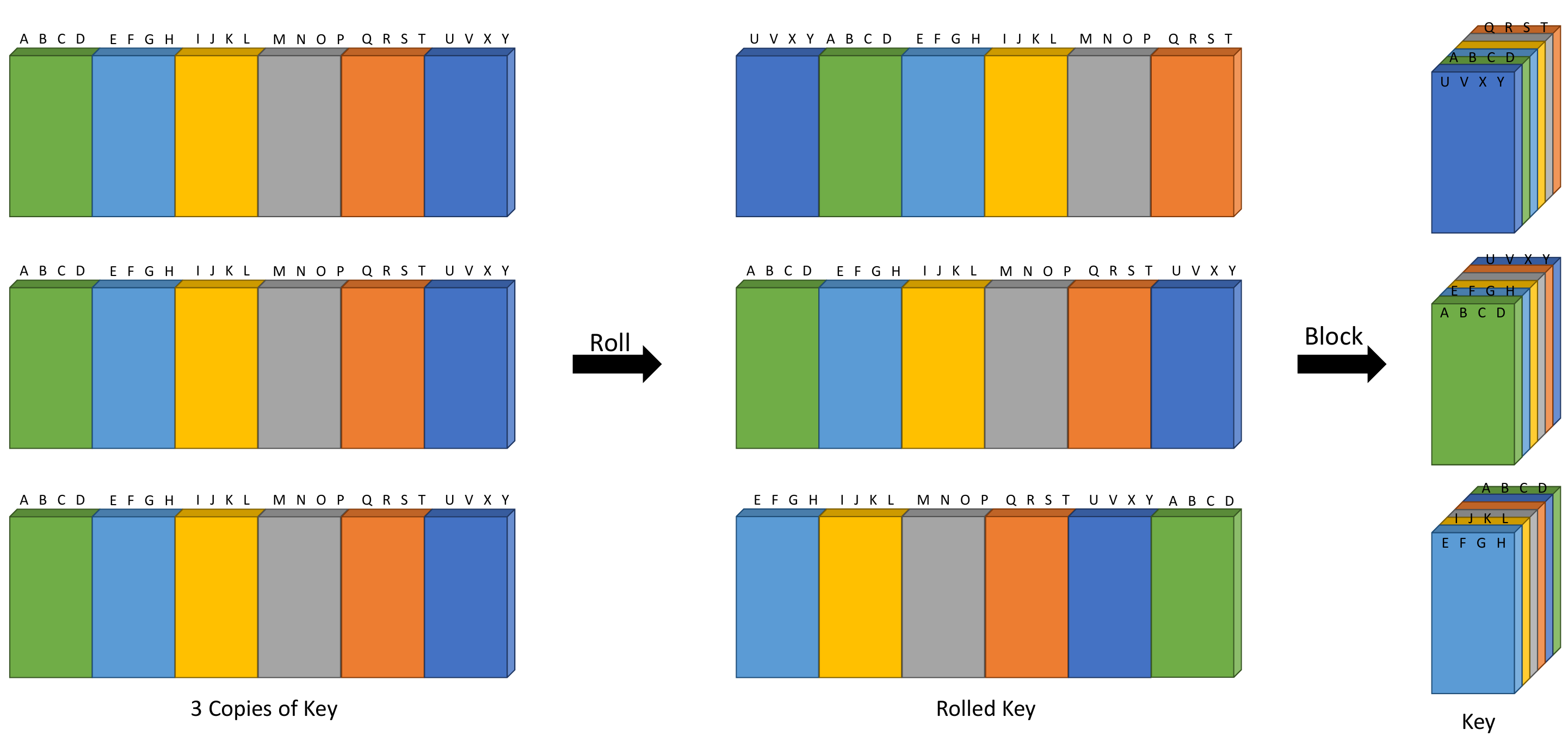}
    \caption{Construction of rolled key-block tensor. Make $w$ copies of the key matrix. Index the copies as  $-(w-1)/2 \leq j \leq (w-1)/2$. Roll $j^\text{th}$ copy by $j$ blocks. Positive roll means circular shift entries left and likewise for negative roll corresponds to right shift. Finally, reshape by grouping the blocks along a new axis to obtain the key-blocked tensor. For illustration purpose $w=3$ is chosen.}
    \label{fig:apndx_copy-roll}
    \vspace{-2mm}
\end{figure}

\begin{figure}[hbt]
    \centering
    \includegraphics[width=0.8\linewidth]{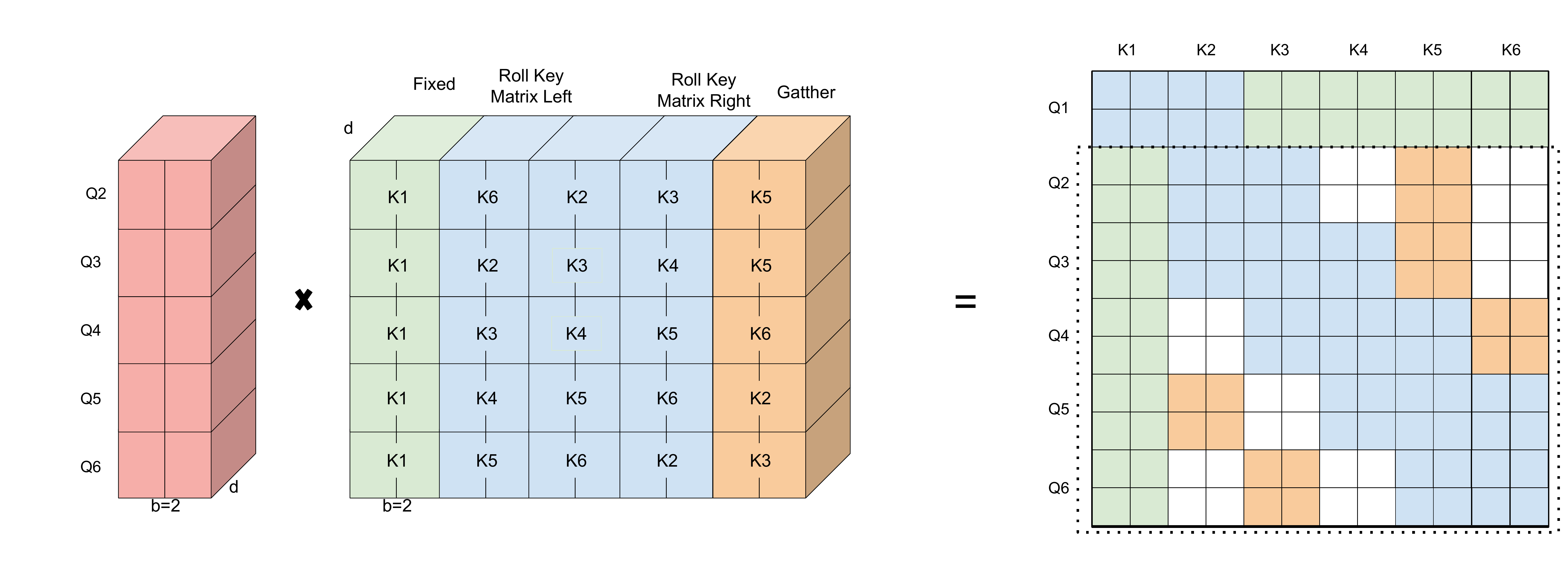}
    \caption{Overview of \bigb attention computation. Structured block sparsity helps in compactly packing our operations of sparse attention, thereby making our method efficient on GPU/TPU.
    On the left, we depict the transformed dense query and key tensors.
    The query tensor is obtained by simply blocking and reshaping while the final key tensor by concatenating three transformations: 
    The first green columns, corresponding to global attention, is fixed. 
    The middle blue columns correspond to window local attention and can be obtained by appropriately rolling as illustrated in~\Cref{fig:apndx_copy-roll}.
    For the final orange columns, corresponding to random attentions, we need to use computationally inefficient gather operation.
    Dense multiplication between the query and key tensors efficiently calculates the sparse attention pattern (except the first row-block, which is computed by direct multiplication), using the ideas illustrated in~\Cref{fig:apndx_bigb_calc_idea}.
    The resultant matrix on the right is same as that shown in~\Cref{fig:apndx_bigb_atn}. 
    }
    \label{fig:apndx_blk_cmp}
    \vspace{-3mm}
\end{figure}

\newpage
\section{NLP experiments details}
\label{sec:apndx-expt-nlp}

\subsection{MLM Pretraining}
\label{sec:app-expt-nlp:mlm}

We use four publicly available datasets Books \citep{zhu2015aligning}, CC-News \citep{guu2020realm}, Stories \citep{trinh2018simple} and Wikipedia to pretrain \bigb. 
We borrow the sentencepiece vocabulary as RoBERTa (which is in turn borrowed from GPT2).
We split any document longer than $4096$ into multiple documents and we join documents that were much smaller than $4096$.
Following the original BERT training, we mask $15\%$ of tokens in these four datasets, and train to predict the mask. We warm start from RoBERTa's checkpoint. 
We train two different models: \bigb-\textsc{itc}-base and \bigb-\textsc{etc}-base. The hyper-parameters for these two models are given in~\Cref{tab:app_mlm_param}. In all experiments we use a learning
rate warmup over the first 10,000 steps, and linear
decay of the learning rate. 

Similar to the norm, we trained a large version of model as well, which has 24 layers with 16 heads and hidden dimension of 1024.
Following the observation from RoBERTa, we pretrain on a larger batch size of 2048 for this size.
For \bigb-\textsc{itc} the block length was kept same as base size, but for \bigb-\textsc{etc} the block length was almost doubled to 169. All the remaining parameters were the same.

\begin{table}[bht]
\small
\centering
\begin{tabular}{@{}l c r c r @{}}
\toprule
Parameter & & \bigb-\textsc{itc} & & \bigb-\textsc{etc} \\
\midrule
 Block length, $b$ & & $64$ & & 84\\
 $\#$ of global token, $g$ & & $2\times b$ & & $256$ \\
 Window length, $w$  & & $3\times b$ & & $3\times b$ \\
 $\#$ of random token, $r$ & & $3\times b$ & & $0$ \\
 Max. sequence length & & $4096$ & & $4096$\\
 $\#$ of heads & & $12$ & & $12$\\
 $\#$ of hidden layers & & $12$ & & $12$ \\
 Hidden layer size & & $768$ & & $768$ \\
 Batch size & & $256$ & & $256$ \\
 Loss & & MLM & & MLM \\
 Activation layer & & gelu & & gelu \\ 
 Dropout prob & & $0.1$ & & $0.1$ \\
 Attention dropout prob & & $0.1$ & & $0.1$ \\
 Optimizer & & Adam & & Adam\\
 Learning rate & & $10^{-4}$  & & $10^{-4}$\\
 Compute resources & & $8 \times 8$ TPUv3 & & $8 \times 8$ TPUv3\\
\bottomrule
\end{tabular}
\vspace{2mm}
\caption{Hyperparameters for the two \bigb base models for MLM.}
\label{tab:app_mlm_param}
\end{table}

\begin{table}[b]
    \centering
    \small
    \parbox{.45\linewidth}{
    \centering
    \begin{tabular}{@{}lrr@{}}
    \toprule
    Dataset & $\#$ tokens & Avg. doc len. \\
    \midrule
        Books \citep{zhu2015aligning} & $1.0$B & $37$K \\
        CC-News \citep{guu2020realm} & $7.4$B & $561$ \\
        Stories \citep{trinh2018simple} & $7.7$B & $8.2$K \\
        Wikipedia &  $3.1$B & $592$ \\
    \bottomrule
    \end{tabular}
    \vspace{2mm}
    \caption{Dataset used for pre training.}
    \label{tab:mlm_data}
    }
    \quad \quad
    \parbox{.45\linewidth}{
    \centering
    \begin{tabular}{@{}lcc@{}}
    \toprule
    Model &  Base & Large \\
    \midrule
    RoBERTa (sqln: 512)  & 1.846 & 1.496 \\
    Longformer (sqln: 4096)  & 1.705 & 1.358 \\
    \bigb-\textsc{itc} (sqln: 4096)   & 1.678 & 1.456 \\
    \bigb-\textsc{etc} (sqln: 4096)   & \textbf{1.611} & \textbf{1.274} \\
     \bottomrule
    \end{tabular}
    \vspace{2mm}
    \caption{MLM performance on held-out set.}
    \label{tab:mlm_bpc}
    }
\end{table}

\begin{table}
    \small
    \centering
    \begin{tabular}{@{}lrrrrrr@{}}
    \toprule
    & & \multicolumn{2}{c}{Instances} & & \multicolumn{2}{c}{Instance Length} \\
    \cmidrule{3-4} \cmidrule{6-7}
    Dataset & & Training & Dev & & Median & Max \\
    \midrule
    HotpotQA-distractor \citep{yang2018hotpotqa} & & $90447$ & $7405$ & & $1227$ & $3560$ \\
        Natural Questions \citep{kwiatkowski2019natural} & & $307373$ & $7830$ & & $3258$ & $77962$ \\
        TriviaQA \citep{JoshiTriviaQA2017} & & $61888$ & $7993$ & & 4900 & 32755 \\
        WikiHop \cite{welbl2018constructing} & & $43738$ & $5129$ & & $1541$ & $20337$ \\
    \bottomrule
    \end{tabular}
    \vspace{2mm}
        \caption{Question Answering Datasets}
    \label{tab:qa_data}
\end{table}

\begin{table}[t]
\small
\centering
\begin{tabular}{@{}l c rr c rr c rr c rr@{}}
\toprule
Parameter & & \multicolumn{2}{c}{HotpotQA} & & \multicolumn{2}{c}{NaturalQ} & & \multicolumn{2}{@{}c@{}}{TriviaQA} & & \multicolumn{2}{@{}c@{}}{WikiHop}\\
\cmidrule{3-4} \cmidrule{6-7} \cmidrule{9-10} \cmidrule{12-13}
Global token location & & \textsc{itc} &\textsc{etc} & & \textsc{itc} & \textsc{etc} & &\textsc{itc} & \textsc{etc}& &\textsc{itc} & \textsc{etc} \\
\midrule
$\#$ of global token, $g$ & & $128$ & $256$ & & $128$  & $230$ & & $128$  & $320$ & & $128$ & $430$   \\
 Window length, $w$  & & $192$ & $252$ & & $192$ & $252$ & &  $192$ & $252$ & & $192$ & $252$  \\
 $\#$ of random token, $r$ & & $192$ &  $0$ & & $192$ &  $0$ & & $192$ &  $0$ & & $192$ &  $0$ \\
 Max. sequence length & & $4096$ & $4096$ & & $4096$ & $4096$ & & $4096$ &  $4096$ & & $4096$ &  $4096$\\
 $\#$ of heads & & $12$ & $12$ & & $12$ & $12$ & & $12$ & $12$ & & $12$ & $12$\\
 $\#$ of hidden layers & & $12$ & $12$ & & $12$ & $12$ & & $12$ & $12$ & & $12$ & $12$ \\
 Hidden layer size & & $768$ & $768$ & & $768$ & $768$ & & $768$ & $768$ & & $768$ &   $768$ \\
 Batch size & & $32$ & $32$ & & $128$ & $128$ & & $32$ &  $32$ & & $64$ &  $64$ \\
 \multirow{2}{*}{Loss} & &   \multicolumn{2}{c}{cross-entropy} & & \multicolumn{2}{c}{cross-entropy} & & \multicolumn{2}{@{}c@{}}{cross-entropy} & & \multicolumn{2}{@{}c@{}}{cross-entropy} \\
 & &   \multicolumn{2}{c}{golden spans} & &   \multicolumn{2}{c}{golden spans} & & \multicolumn{2}{@{}c@{}}{noisy spans~\citep{clark2017simple}} & & \multicolumn{2}{@{}c@{}}{ans choices} \\
 Compute resources & & \multicolumn{2}{c}{$4 \times 2$ TPUv3} & & \multicolumn{2}{c}{$4 \times 8$ TPUv3} & &  \multicolumn{2}{@{}c@{}}{$4 \times 2$ TPUv3} & &  \multicolumn{2}{@{}c@{}}{$4 \times 4$ TPUv3}\\
\bottomrule
\end{tabular}
\vspace{2mm}
\caption{Hyperparameters of base \bigb model used for Question Answering i.e.~the numbers reported in \Cref{tab:QADev}}
\label{tab:app_qa_dev}
\end{table}

\begin{table}[t]
\small
\centering
\begin{tabular}{@{}l c r c r c r c r@{}}
\toprule
Parameter & &  HotpotQA & & NaturalQ & & TriviaQA & & WikiHop \\
\midrule
Global token location & & \textsc{etc} & & \textsc{etc} & & \textsc{etc} & &\textsc{etc} \\
 $\#$ of global token, $g$ & & $256$ & & $230$ & & $320$ & & $430$  \\
 Window length, $w$  & & $507$ & & $507$ & & $507$ & & $507 $ \\
 $\#$ of random token, $r$ & & $0$ & & $0$ & & $0$ & & $0$ \\
 Max. sequence length & & $4096$ & & $4096$ & & $4096$ & & $4096$\\
 $\#$ of heads & & $16$ & & $16$ & & $16$ & & $16$\\
 $\#$ of hidden layers & & $24$ & & $24$ & & $24$ & & $24$ \\
 Hidden layer size & & $1024$ & & $1024$ & & $1024$ & & $1024$ \\
 Batch size & & $32$ & & $64$ & & $32$ & & $64$ \\
 Loss & & cross-entropy & & cross-entropy & & cross-entropy & & cross-entropy \\
 Num epochs & & $\{5, 9\}$ & & $\{3,5\}$ & & $\{3,5\}$ & & $\{5, 10\}$ \\
 Optimizer  & & Adam  & & Adam & & Adam & & LAMB\\
 Learning rate  & & $3\times 10^{-5}$ & & $\{5, 10\}\times 10^{-5}$ & & $\{3, 5\}\times 10^{-5}$ & & $\{2,5 \}\times 10^{-5}$ \\
 Compute resources & & $4 \times 4$ TPUv3 & & $4 \times 8$ TPUv3 & & $4 \times 4$ TPUv3 & & $4 \times 8$ TPUv3\\
\bottomrule
\end{tabular}
\vspace{2mm}
\caption{Hyperparameters of large \bigb model for Question Answering submitted for test i.e.~the numbers reported in~\Cref{tab:QATest}}
\label{tab:app_qa}
\end{table}

\subsection{Question Answering}
\label{sec:app-expt-nlp:qa}

The detailed statistics of the four datasets used are given in~\Cref{tab:qa_data}. 
All the hyperparameters for \bigb, used for creating~\Cref{tab:QADev} are shown in~\Cref{tab:app_qa_dev} and those submitted to get~\Cref{tab:QATest} 
are shown in \Cref{tab:app_qa}. 
We use two types of regularization in training:

\begin{itemize}[leftmargin=12mm, itemsep=0mm, partopsep=0pt,parsep=0pt]
    \item We used a variant of contrastive predictive coding~\citep{oord2018representation} as a dual encoder model.
    \item We use position embedding for \textsc{itc} and relative position encoding~\citep{shaw2018self} for \textsc{etc}.
\end{itemize}

Next, we will mention the dataset/task specific part of the model.

\paragraph{HotpotQA}
The data consists of each question with multiple evidence paragraphs. 
We filtered 16 QA where the answer was not in the given  evidences. 
For \bigb-\textsc{itc}, we use first $128$ global tokens. 
For \bigb-\textsc{etc}, we have one global token for each question token, one for each evidence paragraph, and one for each sentence within the paragraph, for a maximum of $256$ global token.
We use a dense layer on the output corresponding to global token of the evidence paragraph to predict whether its a supporting fact with a threshold over the output logits.
The answer type (yes/no/span) is predicted with a single dense layer from the global CLS token.
For span based answers, the spans are predicted with dense layers on the sequence with the distance between start and end positions to be no more than 30 words. 
The spans are ranked by sum of start and end logits.

\paragraph{Natural Questions}
Here also the data consists of question with supporting evidence, but in form of a single, potentially long, document and not multiple paragraphs. 
We largely follow the setup of \citep{alberti2019bert}.
For documents, that are longer than
4096, a sliding window approach is used
with stride of 2048.
We use CLS token at the beginning, followed by the question followed by a separator token followed by the document as input.
For \bigb-\textsc{itc}, we make the first $128$ tokens as global. For \bigb-\textsc{etc}, we make a global token for CLS, question, and one token for each of the paragraphs.
We train four predictors at the final layer to predict long answer start, long answer end, short answer start and short answer end respectively.
Instead of independently predicting the start and end of answers we first predict the start and then predict the best end location beyond the start.
For short answer, we limit the distance between start and end positions to be no more than 38 words.
The answer type (null, yes, no, short, long) is predicted from CLS token output embedding.
When the logit for a yes/no answer is higher than the logits for short, long or null answer, we replace the short answer with a corresponding yes/no text.

\paragraph{TriviaQA}
The data consists of question-answer pairs with Wikipedia articles as the ``noisy'' supporting evidence.
We call them noisy because the given Wikipedia articles may or may not contain the answer.
Moreover, the answer entities is not annotated to appropriate span in the article, rather all occurrences found using fuzzy string matching are listed.
We use CLS token at the beginning, followed by the question followed by a separator token followed by the document as input.
For \bigb-\textsc{itc}, we make the first $128$ tokens as global. 
For \bigb-\textsc{etc}, we make a global token for CLS, question, and one token for each sentence up to a maximum of 320 global tokens.
Given the noisy nature of answer span, we follow~\citet{clark2017simple} for training.
We use a dense layer on the sequence to predict the answer span for each article independently, with the distance between start and end positions to be no more than 16 words.
For each article the span with maximum start logit + end logit is chosen.
Then we normalize over all the documents associated with that question.

\paragraph{WikiHop}
For each question in WikiHop, we are given upto $79$ candidates, and $63$ supporting paragraphs.
In our \bigb-\textsc{itc} model, following~\citet{beltagy2020longformer}, we concatenate the answer and the question with special tokens, 
\texttt{[q] Question [/q] [ans] Ans1 [/ans] $\ldots$ [ans] AnsN [/ans]} along with the context. 
As the start of the text, always contains questions followed by answers, we make the first $128$ token attend globally. 
In \bigb-\textsc{etc} model, we do not need to insert special \texttt{[ans]}, \texttt{[/ans]} etc.~as we design global tokens appropriately.
Along with global tokens for question, we have one per candidate answer up to a maximum of 430.
Further, we linked answer tokens to their mentions using relative position label. 
Lastly, we use a dense layer that takes in the output vector corresponding to a candidate answer, and predicts a score for the current candidate to be the correct answer.
We apply this dense layer to each candidate independently and 
the candidate with the best score is picked as our final answer.

It is worthwhile to note that explicitly designed attention connection in \textsc{etc} works slightly better, the random connection based \textsc{itc} is pretty competative.

\subsection{Relationship to Contemporary Work} \label{sec:app-related-work}

\paragraph{Longformer}
\citet{child2019generating} introduced localized sliding window to reduce computation.
A more recent version, which includes localized sliding windows and global tokens was introduced independently by
Longofrmer\citep{beltagy2020longformer}. Although \bigb contains additional random tokens, there are also differences in the way global and local tokens are realized. In particular even when there is no random token, as used to get SoTA in question answering, there are two key differences between Longformer and \bigb-etc (see \citep{ainslie2020etc}):
\begin{enumerate}
    \item We use global-local attention with relative
position encodings enables it to better handle structured inputs
    \item Unlike Longformer, we train the global tokens using CPC loss and learn their use during finetuning.
\end{enumerate}

\subsection{Classification}
\label{sec:app-expt-nlp:cls}
We try two types of classification task.

\begin{table}[bht]
\small
\centering
\begin{tabular}{@{}l c r r r r r @{}}
\toprule
Parameter & & IMDb & Arxiv & Patents & Hyperpartisan & Yelp-5  \\
\midrule
 Batch size & & 64 & 64 & 64 & 32 & 32 \\
 Learning rate & & $1\times 10^{-5}$ & $3\times 10^{-5}$ & $5\times 10^{-5}$ & $5\times 10^{-6}$ & $2\times 10^{-5}$\\
 Num epochs & & 40 & 10 & 3 & 15 & 2 \\
 TPUv3 slice  & & $4 \times 4$ & $4 \times 4$ & $4 \times 4$ & $4 \times 2$ & $4 \times 8$ \\
 $\#$ of heads & & \multicolumn{4}{c}{12} & 16\\
 $\#$ of hidden layers & & \multicolumn{4}{c}{12} & 24 \\
 Hidden layer size & & \multicolumn{4}{c}{768} & $1024$ \\
 Block length, $b$ & & \multicolumn{5}{c}{64} \\
 Global token location & & \multicolumn{5}{c}{\textsc{itc}} \\
 $\#$ of global token, $g$ & & \multicolumn{5}{c}{$2\times b$} \\
 Window length, $w$  & & \multicolumn{5}{c}{$3\times b$ } \\
 $\#$ of random token, $r$ & & \multicolumn{5}{c}{$3\times b$} \\
 Max. sequence length & & \multicolumn{5}{c}{4096}\\
 Vocab size & & \multicolumn{5}{c}{$50358$}\\
 Activation layer & & \multicolumn{5}{c}{gelu} \\ 
 Dropout prob & & \multicolumn{5}{c}{0.1} \\
 Attention dropout prob & & \multicolumn{5}{c}{0.1} \\
 Loss & & \multicolumn{5}{c}{cross-entropy} \\
 Optimizer & & \multicolumn{5}{c}{Adam}\\
\bottomrule
\end{tabular}
\vspace{2mm}
\caption{Hyperparameters for document classification.}
\label{tab:app_dc}
\end{table}

\begin{table}
    \centering
    \small
    \begin{tabular}{@{}lrrrrr@{}}
    \toprule
    Model & IMDb~\citep{maas2011learning} & Yelp-5~\citep{zhang2015character} & Arxiv~\citep{he2019long}  & Patents~\citep{lee2020patent} & Hyperpartisan~\citep{kiesel2019semeval}\\
    \midrule
    \# Examples & 25000 & 650000 & 30043 & 1890093 & 645 \\
    \# Classes & 2 & 5 & 11 & 663 & 2 \\
    Excess fraction & 0.14 & 0.04 & 1.00 & 0.90 & 0.53 \\
    \midrule
    SoTA & \citep{thongtan2019sentiment} 97.4 & \citep{abreu2019hierarchical} 73.28 & \citep{olson2019adapting} 87.96 &  \citep{olson2019adapting} 69.01 &
    \citep{jiang2019team} 90.6 \\
    RoBERTa & $95.0 \pm 0.2$ & 71.75 & 87.42 &  67.07 & $87.8 \pm 0.8$ \\
    \bigb  & $95.2\pm0.2$ & 72.16 & \textbf{92.31} & 69.30 & $\mathbf{92.2 \pm 1.7}$ \\
     \bottomrule
    \end{tabular}
    \vspace{2mm}
        \caption{Classification results. We report the F1 micro-averaged score for all datasets. Experiments on smaller IMDb and Hyperpartisan datasets are repeated 5 times and the average performance is presented along with standard deviation.}
    \label{tab:cls}
\end{table}

\begin{table}[bht]
\small
\centering
 \begin{tabular}{@{}lcccccccc@{}}
    \toprule
System             &  MNLI-(m/mm) & QQP  & QNLI & SST-2 & CoLA & STS-B & MRPC & RTE  \\
                   & 392k         & 363k & 108k & 67k   & 8.5k & 5.7k  & 3.5k & 2.5k   \\ 
\midrule
BERT               & 84.6/83.4    & 71.2 & 90.5 & 93.5  & 52.1 & 85.8  & 88.9 & 66.4 \\
XLNet              & 86.8/-       & 91.4 & 91.7 & 94.7  & 60.2 & 89.5  & 88.2 & 74.0 \\
RoBERTa            & 87.6/-       & 91.9 & 92.8 & 94.8  & 63.6 & 91.2  & 90.2 & 78.7 \\
\bigb     & 87.5/87.3    & 88.6 & 92.2 & 94.6  & 58.5 & 87.8  & 91.5 & 75.0 \\
    \bottomrule
   \end{tabular}
   \vspace{2mm}
      \caption{GLUE Dev results on base sized models. 
   Number of training examples is reported below each task.
   MCC score is reported for CoLA, F1 score is reported for MRPC, Spearman correlation is reported for STS-B, and accuracy scores are reported for the other tasks.}
   \label{tab:glue_dev}
\end{table}

\paragraph{Document classification}
We experiment on datasets of different lengths and contents,
as listed in ~\Cref{tab:cls}.
In particular, we look at sentiment analysis (IMDb~\citep{maas2011learning} and Yelp-5~\citep{zhang2015character}) task and topic assignment (Arxiv~\citep{he2019long}, Patents~\citep{lee2020patent}, and Hyperpartisan~\citep{kiesel2019semeval}) task.
Following BERT, we used one layer with cross entropy loss on top of the first [CLS] token from the \bigb encoder consuming 4096 tokens.
We report the results of document classification experiments in~\Cref{tab:cls}.
We compare against state-of-the-art (SoTA) methods for each dataset and plain RoBERTa model with 512 tokens truncation.
In all experiments we use a learning rate warmup over the first 10\% steps, and linear decay of the learning rate and detail list of remaining hyperparameters are provided in
\Cref{tab:app_dc}.
For better quantitative evaluation, we compute the fraction of the dataset that exceeds 512 tokens, i.e. the length at which the document are often truncated.
We see that gains of using \bigb are more significant when we have longer documents and fewer training examples.
For instance, using base sized model,
\bigb improves state-of-the-art for Arxiv dataset by about $\bm{5\%}$ \textbf{points}.
On Patents dataset, there is improvement over using simple BERT/RoBERTa, but given the large size of training data the improvement over SoTA (which is not BERT based) is not significant.
Note that this performance gain is not seen for much
smaller IMDb dataset.
Along with experimental setup detail, we present detailed results in~\Cref{sec:app-expt-nlp:cls} which show competitive performance.

\paragraph{GLUE}
The General Language Understanding Evaluation (GLUE) benchmark \citep{wang2018glue}, test language models on 8 different natural language understanding tasks.
We used the same training parameters as mentioned in \url{https://github.com/pytorch/fairseq/blob/master/examples/roberta/README.glue.md}. Our model parameters are $b=64, g=2\times b, w = 3 \times b, r = 3 \times b$ ( we used the \bigb-\textsc{itc} base model pretrained on MLM task). 
We compare the performance of \bigb to BERT, XLNet \citep{yang2019xlnet} and RoBERTa in \Cref{tab:glue_dev}. We find that even on task that have a much smaller context, our performance is competitive to full attention models.

\subsection{Summarization\label{sec:appn_summarization}}
As discussed in~\Cref{sec:seq2seq},
given the small length of output sequence, we used sparse \bigb attention only for encoder, while keeping the full attention for decoder.
The number of hidden layers, number of heads, and hidden dimension is same for encoder and decoder.
The hyperparameters are detailed in \Cref{tab:app_sum}.
We summarize our result in~\Cref{tab:app_sum_num}. In all experiments, we use a learning
rate warmup over the first 10,000 steps, and square root
decay of the learning rate.

\begin{table}[bht]
\small
\centering

\begin{tabular}{@{}l r r c r @{}}

\toprule
Parameter & \multicolumn{2}{r}{Base: \bigb-RoBERTa} & & Large: \bigb-Pegasus \\
\midrule

 Block length, $b$ & & $64$ & & $64$\\
 Global token location & & \textsc{itc} & & \textsc{itc} \\
 $\#$ of global token, $g$ & & $2\times b$ & & $2\times b$ \\
 Window length, $w$  & & $3\times b$ & & $3\times b$ \\
 $\#$ of random token, $r$ & & $3\times b$ & & $3\times b$ \\
 \multirow{3}{*}{Max. encoder sequence length} & & BBC-XSUM: \qquad \phantom{.}1024 & & 1024 \\
 & & CNN/DM: \qquad \phantom{.}2048  & & 2048 \\
 & & Others: \qquad \phantom{.}3072 & & 3072 \\
 \multirow{3}{*}{Max. decoder sequence length} & & BBC-XSUM: \qquad \phantom{10.}64 & & 64 \\
 & & CNN/DM: \qquad \phantom{1.}128  & & 128 \\
 & & Others: \qquad \phantom{1.}256  & & 256 \\
 Beam size & & 5 & & 5 \\
 \multirow{2}{*}{Length penalty} & & BBC-XSUM: \qquad \phantom{10}0.7 & & 0.7 \\
 & & Others: \qquad \phantom{10}0.8  & & 0.8 \\
 $\#$ of heads & & $12$ & & $16$\\
 $\#$ of hidden layers & & $12$ & & $16$ \\
 Hidden layer size & & $768$ & & $1024$ \\
 Batch size & & $128$ & & $128$ \\
 \multirow{2}{*}{Loss} & & teacher forced & & teacher forced \\
 & & cross-entropy & & cross-entropy \\
 Activation layer & & gelu & & gelu \\ 
 Dropout prob & & $0.1$ & & $0.1$ \\
 Attention dropout prob & & $0.1$ & & $0.1$ \\
 Optimizer & & Adam & & Adafactor\\
 Learning rate & & $1\times10^{-5}$  & & $1\times10^{-4}$\\
  Compute resources & & $4 \times 4$ TPUv3 & & $4 \times 8$ TPUv3\\
\bottomrule
\end{tabular}
\vspace{2mm}
\caption{Encoder hyperparameters for Summarization. We use full attention in decoder}
\label{tab:app_sum}
\end{table}

\begin{table}[bht]
    \centering
    \small
    \begin{tabular}{@{}lrrrrrccrcc@{}}
    \toprule
    & & \multicolumn{3}{c}{Instances} & & \multicolumn{2}{c}{Input Length}& & \multicolumn{2}{c}{Output Length} \\
    \cmidrule{3-5} \cmidrule{7-8} \cmidrule{10-11}
    Dataset & & Training & Dev & Test & & Median & 90\%-ile & &  Median & 90\%-ile \\
    \midrule
       Arxiv \citep{cohan2018discourse} & & 203037 & 6436 & 6440 & & 6151 & 14405 & & 171 & 352 \\
       PubMed \citep{cohan2018discourse} & & 119924 & 6633 & 6658 & & 2715 & 6101 & & 212 & 318 \\
       BigPatent \citep{sharma2019bigpatent} & &  1207222 & 67068 & 67072 & & 3082 & 7693  & & 123 & 197  \\
    \bottomrule
    \end{tabular}
    \vspace{2mm}
     \caption{Statistics of datasets used for summarization.}
    \label{tab:long_sum_data}
\end{table}

Following success of several recent works~\citep{rothe2019leveraging,liu2019roberta}, we warm start our encoder-decoder \bigb transformer model with pretrained weights and the weights between encoder and decoder are shared.
In particular, the query/key/value matrix of self-attention and all the feedforward layers are shared between encoder and decoder.
The only variable that is initialized randomly is the encoder-decoder attention.
For base sized model, we utilize our MLM pretrained model on 4096 sequence length from~\Cref{sec:app-expt-nlp:mlm}, which is in turn initialized using the public RoBERTa checkpoint.
For the large size model, we lift weight from the state-of-the-art Pegasus model~\citep{zhang2019pegasus}, which is pretrained using an objective designed for summarization task.

To check if sparse attention causes significant degradation as compared to full attention, we further experiment on two shorter but popular datasets, where full attention can be used without significantly truncating the document. 
The statistics of these two datasets are in~\Cref{tab:sum_data}.
We see that our performance is competitive, which shows that sparse attention can achieve similar performance to a full attention models. 

\begin{table}
    \small
    \centering
    \begin{tabular}{@{}lrrrrrrrrrr@{}}
    \toprule
    & & \multicolumn{3}{c}{Instances} & & \multicolumn{2}{c}{Input Length}& & \multicolumn{2}{c}{Output  Length} \\
    \cmidrule{3-5} \cmidrule{7-8} \cmidrule{10-11}
    Dataset & & Training & Dev & Test & & Median & 90\%-ile & &  Median & 90\%-ile \\
    \midrule
       BBC XSum \citep{narayan2018don}  & &     204044 & 11332 & 11334 & &  359 & 920 & & 25 & 32 \\
       CNN/DailyMail  \citep{hermann2015teaching} & &  287113 & 13368 & 11490 & &   777 & 1439 & & 59 & 93 \\
    \bottomrule
    \end{tabular}
    \vspace{2mm}
    \caption{Shorter summarization dataset statistics.}
    \label{tab:sum_data}
\end{table}

\begin{table}
    \centering
    \small
    \begin{tabular}{@{}llrrrrrrrr@{}}
    \toprule
     \multicolumn{2}{l}{\multirow[b]{2}{*}{\hspace{-2mm}\normalsize{Model}}} & & 
     \multicolumn{3}{c}{BBC XSum} & & \multicolumn{3}{c}{CNN/DailyMail}\\
    \cmidrule{4-6} \cmidrule{8-10}
    & & & R-1 & R-2 & R-L & & R1 & R2 & R-L \\
    \midrule
    \multirow{9}{*}{\rotatebox[origin=c]{90}{Prior Art}}
    & Lead & & $16.30$ & $1.61$ & $11.95$ & & $39.60$ & $ 17.70$ & $ 36.20$ \\
    & PtGen~\citep{see2017get} & & $29.70$ & $ 9.21$ & $23.24$ & & $39.53$ & $17.28$ & $36.38$ \\
    & ConvS2S~\citep{gehring2017convolutional} & & $31.89$ & $11.54$ & $25.75$ & & $-$ & $-$ & $-$ \\
    & MMN~\citep{kim2018abstractive} & & $32.00$ & $12.10$ & $26.00$ & & $-$ & $-$ & $-$ \\
    & Bottom-Up~\citep{gehrmann2018bottom} & & $-$ & $-$ & $-$  & & $41.22$ & $18.68$ & $38.34$ \\
    & TransLM~\citep{khandelwal2019sample}  & & $-$ & $-$ & $-$  & & $39.65$ & $17.74$ & $ 36.85$ \\
    & UniLM~\citep{dong2019unified}  & & $-$ & $-$ & $-$  & & $43.47$ & $20.30$ & $40.63$ \\
    & Extr-Abst-BERT~\citep{liu2019text} & & 38.81  & 16.50  & 31.27 & &  42.13     & 19.60 & 39.18               \\
    & BART~\citep{lewis2019bart}  & & 45.14  & 22.27 & 37.25 & & 44.16 & 21.28 & 40.90 \\
    \midrule
    \multirow{4}{*}{\rotatebox[origin=c]{90}{Base}}
    & Transformer~\citep{vaswani2017attention}                & & 29.61 &  9.47 & 23.17 & & 34.89 & 13.13 & 32.12 \\
    & \; + RoBERTa~\citep{rothe2019leveraging} & & \underline{39.92} & \underline{17.33} & \underline{32.63} & & 39.44 & 18.69 & 36.80 \\
    & \; + Pegasus~\citep{zhang2019pegasus} & & 39.79 & 16.58 & 31.70 & & \underline{41.79} & \underline{18.81} & \underline{38.93} \\  
    & \bigb-RoBERTa & & 39.52 & 17.22 & 32.30 & & 39.25 & 18.46 & 36.61   \\
    \midrule
    \multirow{3}{*}{\rotatebox[origin=c]{90}{Large}}
    & Pegasus (Reported)~\citep{zhang2019pegasus} & & 47.60 & 24.83 & 39.64 & & 44.16 & 21.56 & 41.30 \\
    & Pegasus (Re-eval) & & \textbf{47.37} & \textbf{24.31} & \textbf{39.23} & & \textbf{44.15} & \textbf{21.56} & \textbf{41.05} \\
    & \bigb-Pegasus  & & 47.12 & 24.05 & 38.80 & & 43.84 & 21.11 & 40.74  \\
     \bottomrule
    \end{tabular}
    \vspace{2mm}
    \caption{Summarization ROUGE score for shorter documents.}
    \label{tab:app_sum_num}
\end{table}

\newpage
\section{Genomics experiments details}
\label{sec:apndx-expt-bio}

In this section we provide details of the experimental setup for \bigb on genomics data.

\subsection{Pretraining}
\label{sec:apndx-expt-bio:mlm}
We try to keep the experimental setup as close to a typical NLP pipeline.
In this regard, we take human reference GRCh37\footnote{\url{https://www.ncbi.nlm.nih.gov/assembly/GCF_000001405.39}} and convert it into documents $\mathcal{D}$. Each document $d\in \mathcal{D}$ is a sequence of sentences, where each sentence is a sequence of fragments of DNA. We construct the documents as follows:
\begin{enumerate}[leftmargin=6mm, itemsep=2mm, partopsep=0pt,parsep=0pt]
    \item Start with empty document set $D = \emptyset$.
    \item For each chromosome $C$, repeat the following procedure 10 times.
    \begin{enumerate}[leftmargin=6mm, itemsep=2mm, partopsep=0pt,parsep=0pt]
         \item Pick uniformly at random a starting point $q$ between base pairs 0 and 5000 from the 5' end.
    \item Repeat until $q > |C|$
    \begin{enumerate}[leftmargin=6mm, itemsep=2mm, partopsep=0pt,parsep=0pt]
    \vspace{1mm}
        \item Pick uniformly at random $s$ a number between 50 and 100 to denote number of sentences per document.
         \item Constructs a document $d$ containing $s$ sentences using consecutive base pairs (bps). The length of each sentence is chosen uniformly at random between 500-1000. Thus the resulting document has $25,000$ - $100,000$ bps. 
        \item $D = D \bigcup d$
        \item $q = q + |d|$
    \end{enumerate}
\end{enumerate}
\end{enumerate}
By this procedure we end-up with approximately $450K$ documents.

\begin{figure}
    \centering
    \includegraphics[width=\linewidth]{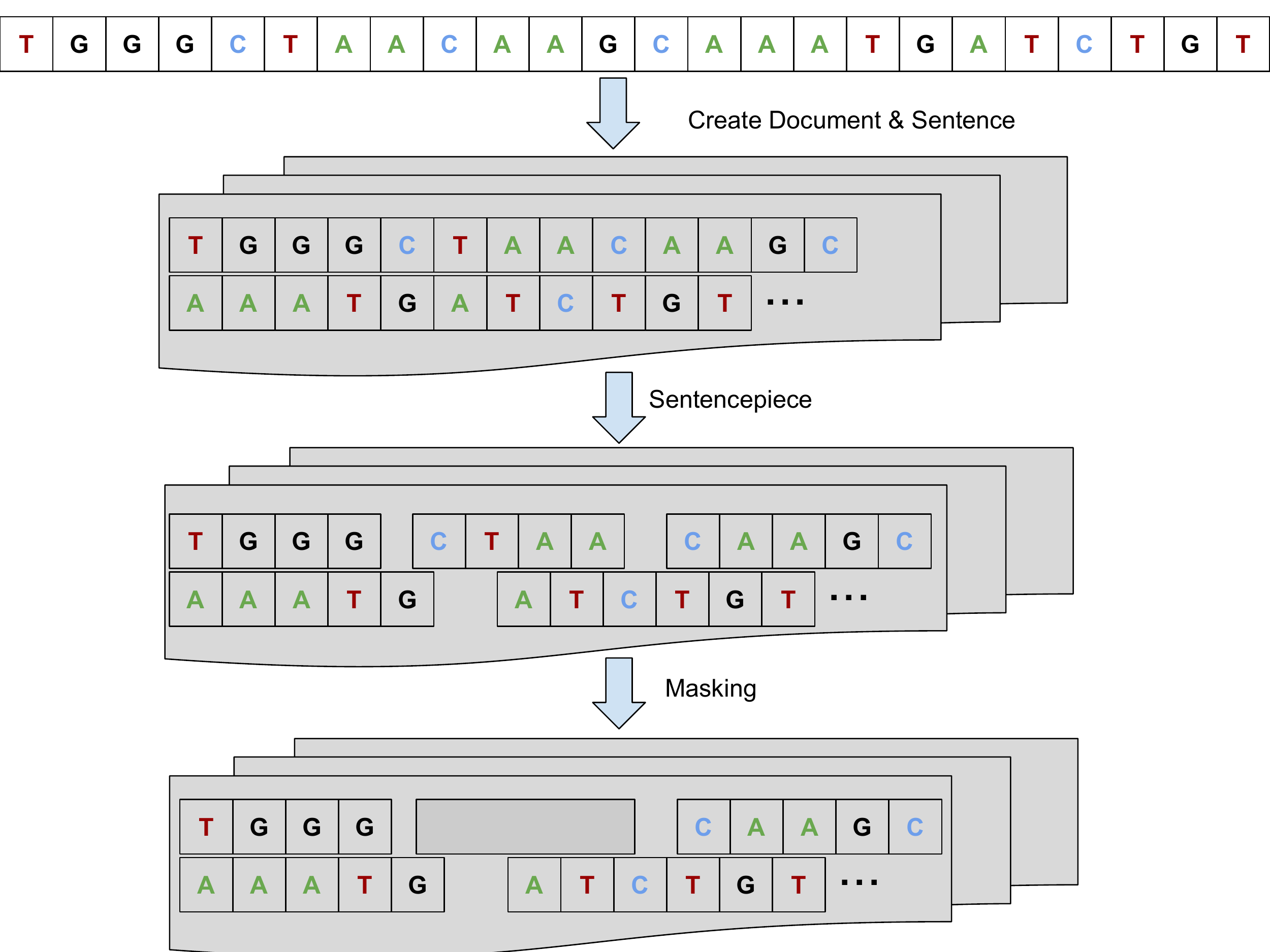}
    \caption{Visual description of how the masked language modeling data was generated from raw DNA dataset. The raw DNA sequences of GRCh37, where split at random positions to create documents with 50-100 sentences where each sentence was 500-1000 base pairs (bps). Thus each document had a continuous strand of 25000-100,000 bps of DNA. This process was repeated 10 times to create 10 sets of document for each chromosome of GRCH37. The resulting set of documents was then passed through Sentencepiece that created tokens of average 8bp. For pretraining we used masked language model and masked $10\%$ of the tokens and trained on predicting the masked tokens.}
    \label{fig:apndx_mlm_data}
\end{figure}

Next we run sentencepiece~\citep{kudo2018sentencepiece} tokenization on the resulting documents. In particular, using 5 characters as the building blocks (four for bases - A, T, C, G and one for missing symbol N), we construct a byte pair encoding table of size 32k, with each token representing 8.78 base pairs on average.

Using the above constructed documents, we construct a dataset for two pretraining tasks following \citet{devlin2018bert}:
\begin{itemize}[leftmargin=6mm, itemsep=2mm, partopsep=0pt,parsep=0pt]
    \item \textbf{Masked Language Model (MLM):} 
    In order to train a deep bidirectional representation, BERT training introduces the MLM task, where we simply mask out 15\% of the input tokens at random, and then predict those masked tokens.
    We can simply replace such masked out of the tokens with a [MASK] placeholder, but it leads to a distribution mis-match for downstream tasks which will not have such placeholders. To mitigate with this issue, out of the 15\% of the tokens selected for masking:
    \begin{itemize}
        \item 80\% of the tokens are actually replaced with the token [MASK].
        \item 10\% of the time tokens are replaced with a random token.
        \item 10\% of the time tokens are left unchanged, but are still predicted at output.
    \end{itemize}
    We run this entire sequence through the \bigb transformer encoder and then predict corresponding to the masked positions, based on the context provided by the other non-masked tokens in the sequence.

    \item \textbf{Next Sentence Prediction (NSP):} 
    In order to understand relationship between two sequences, BERT training introduces the NSP task, where we predict if a given pair of sequences are contiguous or not.
    During training the model gets as input pairs of sequences separated by [SEP] token along with a [CLS] token at the start. Overall the input pattern is: [CLS] sequence A [SEP] sequence B [SEP]. For 50\% of the time the second sequence comes from true sequence after the first one. Remaining 50\% of the time it is a a random sequence from the full dataset. The model is then required to predict this relationship using the output corresponding to the [CLS] token, which is fed into a simple binary classification layer.
\end{itemize}

\begin{wrapfigure}{r}{0.32\textwidth}
    \vspace{-2mm}
    \centering
    \includegraphics[width=0.35\textwidth]{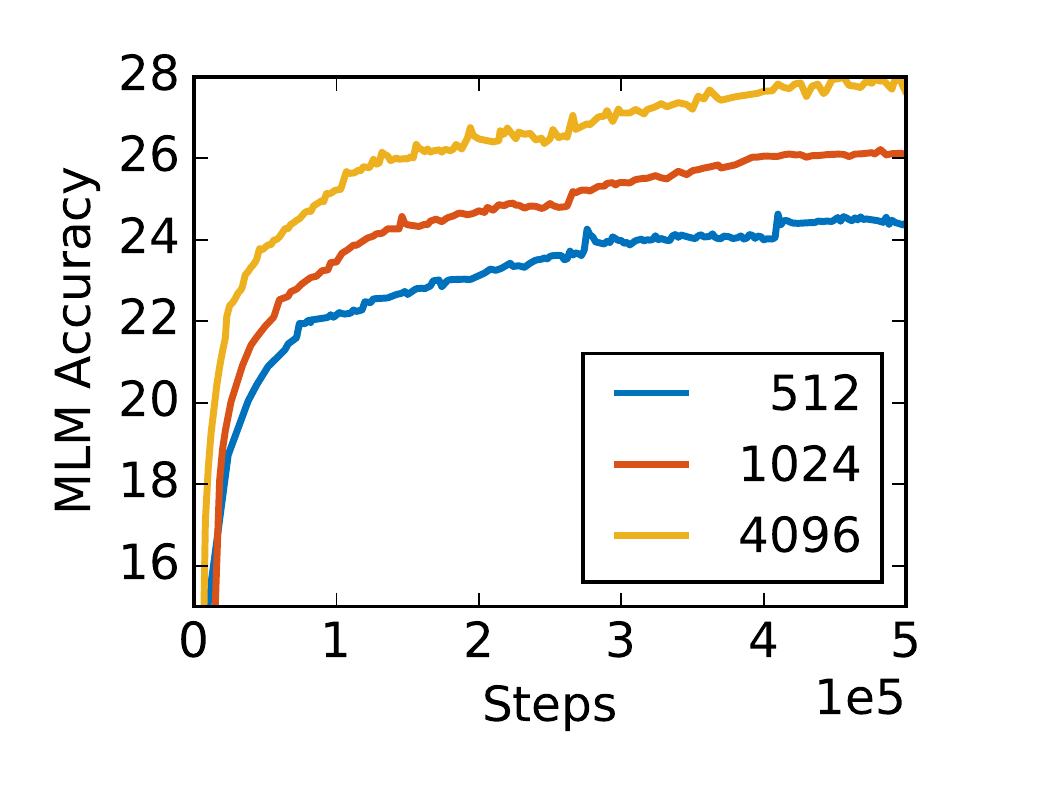}
    \vspace{-5mm}
    \caption{\bigb accuracy with context length.}
    \label{fig:apndx_dna_mlm}
\end{wrapfigure}
The sequence of steps is visually elaborated in \Cref{fig:apndx_fep_data}.
The model is trained with both MLM and NSP together. Training hyperparameter is provided in second columns of \Cref{tab:app_bio}. In all experiments we use a learning
rate warmup over the first 10,000 steps, and linear
decay of the learning rate.

We additionally performed a simple ablation study to validate the hypothesis, that similar to NLP, having a larger context improves performance. 
We use MLM task described above to test how \bigb performed with sequences of different length. 
Accuracy on MLM task with increasing sequence length is shown in \Cref{fig:apndx_dna_mlm}. 
Not only longer context improves final accuracy, it also leads to faster learning, as we have now more opportunities for masking.

\begin{figure}
\vspace{-5mm}
    \centering
    \includegraphics[width=\linewidth]{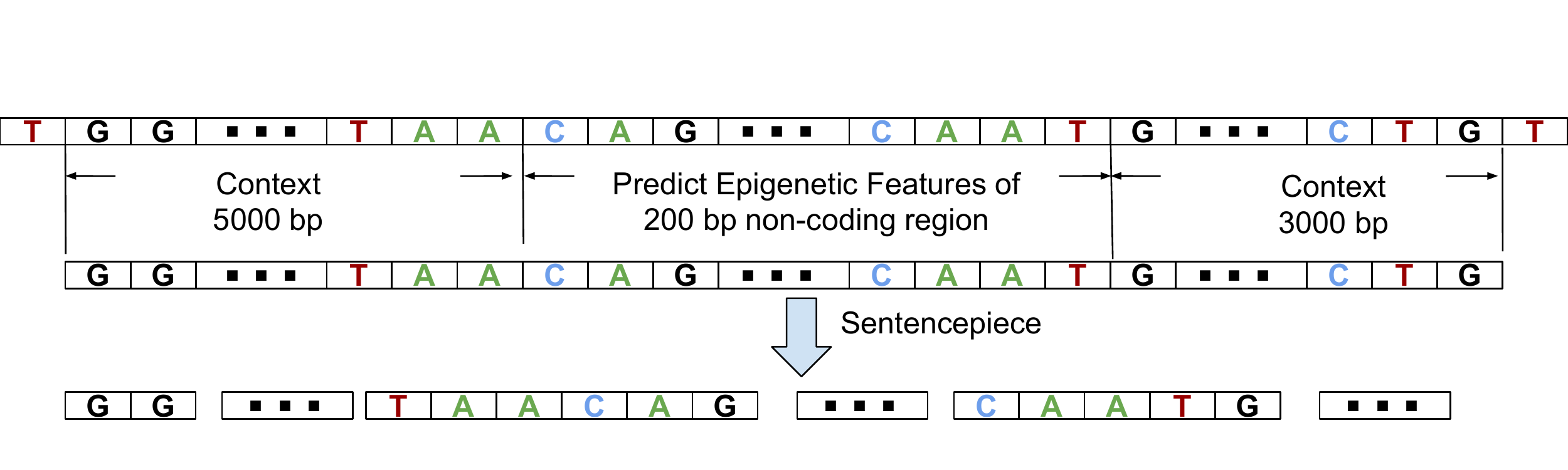}
    \caption{Visual description of the DNA segment from which we predict the chromatin profile for a given non-coding region of the raw DNA sequences of GRCh37. We take 8000 bps of DNA before and after the given non-coding region as context. The complete fragment of DNA including the context on both side, is then tokenized to form our input sequence of tokens. The task is to predict  919  chromatin profile including $690$ transcription factors (TF) binding profiles for $160$ different TFs, $125$ DNase I sensitivity (DHS) profiles and $104$ histone-mark (HM) profiles}
    \label{fig:apndx_fep_data}
\end{figure}

\subsection{Promoter Region Prediction}
The promoter region plays an important role in transcription initiation and thus its recognition is an important area of interest in the field of bioinformatics.
Following \citet{oubounyt2019deepromoter}, we use datasets from Eukaryotic Promoter Database (EPDnew) \citep{dreos2013epd}, which contains
29,597 promoter region in the human genome.
Around the transcription start site (TSS), we extract a sequence of 8000 bp (-5000~+3000 bp) from the human reference genome GRCh37.
Since EPDnew uses newer GRCh38, we convert to GRCh37 coordinates using LiftOver~\citep{kent2002human}.

Following \citet{oubounyt2019deepromoter} for each promoter region example, a negative example (non-promoter sequences) with the same size of the positive one is constructed as follow:
The positive sequence is divided into 20 subsequences. Then, 12 subsequences are picked randomly and substituted randomly. The remaining 8 subsequences are conserved. This process is illustrated in Figure 1 of \citep{oubounyt2019deepromoter}. Applying this process to the positive set results in new non-promoter sequences with conserved parts from promoter sequences (the unchanged subsequences, 8 subsequences out of 20). These parameters enable generating a negative set that has 32 and 40\% of its sequences containing conserved portions of promoter sequences.

We prefix and append each example with [CLS] and [SEP] token respectively.
The output corresponding to the [CLS] token from \bigb transformer encoder is fed to a simple binary classification layer.
We fine-tune the pretrained \bigb from~\Cref{sec:apndx-expt-bio:mlm} using hyper-parameters described in~\Cref{tab:app_bio}.
We note that high performance is not surprising due to the overlap in the nature of negative example generation and MLM pretraining.

\subsection{Chromatin-Profile Prediction}
The first step of sequence-based algorithmic framework for predicting non-coding effects is to build a model to predict, large scale chromatic profile \citep{zhou2015predicting}. 
\begin{table}[b]
\small
\centering
\begin{tabular}{@{} l r r c r c r @{}}
\toprule
Parameter & & Pretraining & & Promoter Region & & Chromatin-Profile \\
\midrule
 Block length, $b$ & & $64$ & & $64$& & $64$\\
Global token location & & \textsc{itc} & & \textsc{itc}& & \textsc{itc} \\
 $\#$ of global token, $g$ & & $2\times b$ & & $2\times b$& & $2\times b$ \\
 Window length, $w$  & & $3\times b$ & & $3\times b$& & $3\times b$ \\
 $\#$ of random token, $r$ & & $3\times b$ & & $3\times b$& & $3\times b$ \\
 Max. Sequence Length & & $4096$ & & $4096$& & $4096$\\
 $\#$ of heads & & $12$ & & $12$& & $12$\\
 $\#$ of hidden layers & & $12$ & & $12$& & $12$ \\
 Hidden layer size & & $768$ & & $768$ & & $768$ \\
 Batch Size & & $256$ & & $256$& & $256$ \\
 Vocab Size & & $32000$ & & $32000$& & $32000$\\
 \multirow{2}{*}{Loss} & & \multirow{2}{*}{MLM+NSP} & & \multirow{2}{*}{BCE} & & 919 x +ve upweighted \\
  & &  & &  & & BCE \\
 Dropout prob & & $0.1$ & & $0.1$ & & $0.1$ \\
 Optimizer & & Adam & & Adam & & Adam\\
 Learning rate & & $0.0001$  & & $0.0001$ & & $0.0001$\\
 $\#$ of steps & & $1000000$ & & 711 && 500000 \\
  Compute Resources & & $8 \times 8$ TPUv3 & & $8 \times 8$ TPUv3 & & $8 \times 8$ TPUv3\\
\bottomrule
\end{tabular}
\vspace{2mm}
\caption{Table of hyperparameters for Computational biology.}
\label{tab:app_bio}
\end{table}
In this paper, we use the dataset provided in \citet{zhou2015predicting}\footnote{ \url{http://deepsea.princeton.edu/media/code/deepsea_train_bundle.v0.9.tar.gz}}, to train \bigb to predict the chromatic profile.

Each training sample consists of a 8,000-bp sequence from the human GRCh37 reference genome centered on each 200-bp bin and is paired with a label vector for 919 chromatin features.
As before, we prefix and append each example with [CLS] and [SEP] token respectively.
The output corresponding to the [CLS] token from \bigb transformer encoder is fed to a linear layer with 919 heads. Thus we jointly predict the 919 independent binary classification problems.
We fine-tune the pretrained \bigb from~\Cref{sec:apndx-expt-bio:mlm} using hyper-parameters described in~\Cref{tab:app_bio}.
As the data is highly imbalanced data (way more negative examples than positive examples), we upweighted loss function for positive examples by factor of 8.

We used training and testing split provided by \citet{zhou2015predicting} using chromosomes and strictly non-overlapping. Chromosome 8 and 9 were excluded from training to test chromatin feature prediction performances, and the rest of the autosomes were used for training and validation. 4,000 samples on chromosome 7 spanning the genomic coordinates 30,508,751–35,296,850 were used as the validation set. 

As the predicted probability for each sequence in DeepSea~\citet{zhou2015predicting} was computed as the ensemble average of the probability predictions for the forward and complementary sequence pairs, we also predict using an ensemble of two \bigb model trained independently.

\end{document}